%% file: example_paper.tex
\documentclass[nohyperref]{article}

\usepackage{microtype}
\usepackage{graphicx}
\usepackage{subcaption}
\usepackage{booktabs} 

\usepackage{hyperref}

\usepackage[accepted]{icml2022}
\input{math_commands.tex}

\usepackage{amsmath}
\usepackage{amssymb}
\usepackage{mathtools}
\usepackage{bbm}
\usepackage{amsthm}
\usepackage{titlesec}
\usepackage{xcolor}
\usepackage{enumitem}
\setlist{leftmargin=4mm}
\usepackage{multicol}
\usepackage{longtable}
\usepackage{float}
\usepackage{caption}

\definecolor{forestgreen}{RGB}{34,139,34}
\usepackage{pifont}

\floatstyle{plaintop}
\restylefloat{table}
\usepackage[capitalize,noabbrev]{cleveref}

\theoremstyle{plain}
\newtheorem{theorem}{Theorem}[section]

\newtheorem{lemma}[theorem]{Lemma}

\theoremstyle{definition}
\newtheorem{definition}[theorem]{Definition}

\theoremstyle{remark}

\usepackage[textsize=tiny]{todonotes}

\icmltitlerunning{Deep Probability Estimation}

\begin{document}

\twocolumn[
\icmltitle{Deep Probability Estimation}



\icmlsetsymbol{equal}{*}
\begin{icmlauthorlist}
\icmlauthor{Sheng Liu}{equal,cds}
\icmlauthor{Aakash Kaku}{equal,cds}
\icmlauthor{Weicheng Zhu}{equal,cds}
\icmlauthor{Matan Leibovich}{equal,courant}
\icmlauthor{Sreyas Mohan}{equal,cds}
\icmlauthor{Boyang Yu}{cds}
\icmlauthor{Haoxiang Huang}{courant}
\icmlauthor{Laure Zanna}{cds,courant}
\icmlauthor{Narges Razavian}{langone}
\icmlauthor{Jonathan Niles-Weed}{cds,courant}
\icmlauthor{Carlos Fernandez-Granda}{cds,courant}

\end{icmlauthorlist}

\icmlaffiliation{cds}{Center for Data Science, New York University, New York, USA}
\icmlaffiliation{courant}{Courant Institute of Mathematical Sciences, New York University, New York, USA}
\icmlaffiliation{langone}{Department of Population Health \& Department of Radiology, NYU School of Medicine, New York, USA}

\icmlcorrespondingauthor{SL, AK, WZ, ML, SM}{\{shengliu, ark576, jackzhu, ml7557, sm7582\}@nyu.edu}

\icmlkeywords{Machine Learning, ICML}

\vskip 0.3in
]



\printAffiliationsAndNotice{\icmlEqualContribution} 

\begin{abstract}
Reliable probability estimation is of crucial importance in many real-world applications where there is inherent (aleatoric) uncertainty. Probability-estimation models are trained on observed outcomes (e.g. whether it has rained or not, or whether a patient has died or not), because the ground-truth probabilities of the events of interest are typically unknown. The problem is therefore analogous to binary classification, with the difference that the objective is to estimate probabilities rather than predicting the specific outcome. This work investigates probability estimation from high-dimensional data using deep neural networks. There exist several methods to improve the probabilities generated by these models but they mostly focus on model (epistemic) uncertainty. For problems with inherent uncertainty, it is challenging to evaluate performance without access to ground-truth probabilities. To address this, we build a synthetic dataset to study and compare different computable metrics. We evaluate existing methods on the synthetic data as well as on three real-world probability estimation tasks, all of which involve inherent uncertainty: precipitation forecasting from radar images, predicting cancer patient survival from histopathology images, and predicting car crashes from dashcam videos. \textcolor{black}{We also give a theoretical analysis of a model for high-dimensional probability estimation which reproduces several of the phenomena evinced in our experiments.} Finally, we propose a new method for probability estimation using neural networks, which modifies the training process to promote output probabilities that are consistent with empirical probabilities computed from the data. The method outperforms existing approaches on most metrics on the simulated as well as real-world data.\footnote{Code available at \url{https://jackzhu727.github.io/deep-probability-estimation/}.} 
\end{abstract}

\section{Introduction}








We consider the problem of building models that answer questions such as: \emph{Will it rain?} \emph{Will a patient survive?} \emph{Will a car collide with another vehicle?} Due to the inherently-uncertain nature of these real-world phenomena, this requires performing \emph{probability estimation}, i.e. evaluating the likelihood of each possible outcome for the phenomenon of interest. 
Models for probability prediction 
must be trained on observed outcomes (e.g. whether it rained, a patient died, or a collision occurred), because the ground-truth probabilities are unknown. The problem is therefore analogous to binary classification, with the important difference that the objective is to estimate probabilities rather than predicting specific outcomes. In probability estimation, two identical inputs (e.g. histopathology images from cancer patients) can potentially result in two different outcomes (death vs. survival). \textcolor{black}{In contrast, in classification the class label is usually completely determined by the data (a picture either shows a cat or it does not).}

The goal of this work is to investigate probability estimation from high-dimensional data using deep neural networks. Probability estimation is a fundamental problem in machine learning \citep{murphy2013machine}. 
Deep networks trained for classification often generate probabilities, which quantify the uncertainty of the estimate (i.e. how likely the network is to classify correctly). This quantification has been observed to be inaccurate, and several methods have been developed to improve it~\citep{Platt1999,guo2017calibration,Szegedy2016RethinkingTI,Zhang2020MixnMatchEA,thulasidasan2020mixup,Mukhoti2020CalibratingDN,thagaard2020can}, 
\textcolor{black}{including Bayesian neural networks~\citep{gal2016dropout, DBLP:journals/corr/WangSY16, shekhovtsov2018feedforward, DBLP:journals/corr/abs-1908-00598}.}  
However, these works restrict their attention almost exclusively to classification in datasets \textcolor{black}{(e.g.  CIFAR-10/100~\cite{Krizhevsky09learningmultiple}, or ImageNet~\citep{imagenet}) where the label itself is \emph{not uncertain}:
it quantifies the confidence of the model in its own prediction, \emph{not the probability of an event of interest}.} In the literature, this is known as epistemic (model) uncertainty \citep{hullermeier2021aleatoric,tagasovska2019single}. We focus on aleatoric uncertainty, stemming from inherent uncertainty in the problem under study. 
\textcolor{black}{To formalize this distinction, we propose and rigorously analyze a simple high-dimensional model with \emph{uncertain} labels, and show that even in this simple model, classification-based methods fail to yield good probability estimates.}

\textcolor{black}{Probability estimation from high-dimensional data is a problem of critical importance in medical prognostics~\citep{Wulczyn2020DeepLS}, weather prediction~\citep{agrawal2019machine}, and autonomous driving~\citep{kim2019crash}.} In order to advance deep-learning methodology for probability estimation it is crucial to build appropriate benchmark datasets. Here we build a synthetic dataset and gather three real-world datasets, which we use to systematically evaluate existing methodology. In addition, we propose a novel approach for probability estimation, which outperforms current state-of-the-art methods. Our contributions are the following:  

\begin{figure}[t]
    \centering
    \includegraphics[width=\columnwidth]{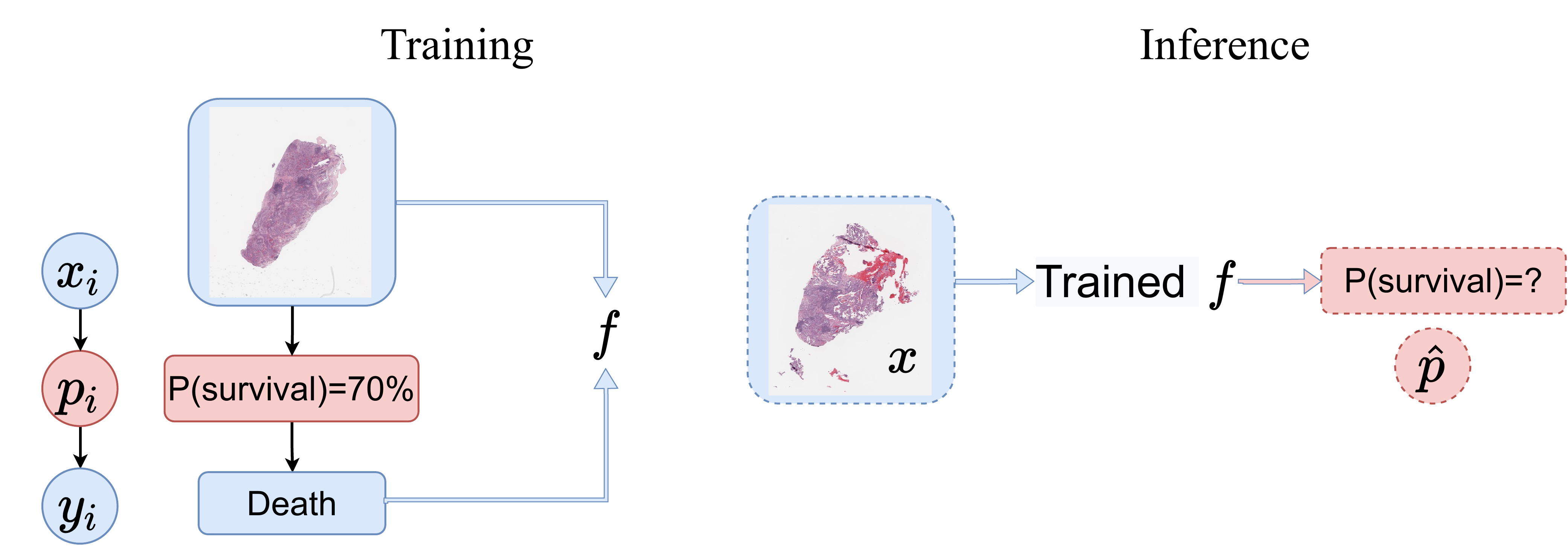}
    \caption{\textbf{The probability-estimation problem.} In probability estimation, we assume that each observed outcome $y_i$ (e.g. death or survival in cancer patients) in the training set is randomly generated from a latent unobserved probability $p_i$ associated to the corresponding data $\vx_i$ (e.g. histopathology images). 
    \textbf{Training} (left): Only $\vx_i$ and $y_i$ can be used for training, because $p_i$ is not observed. 
    \textbf{Inference} (right): Given new data $\vx$, the trained network $f$ produces a probability estimate $\hat{p}\in [0,1]$.}
    \label{fig:illustration}
    \vspace{-1mm}
\end{figure}
%
\setlist{nolistsep}

\begin{itemize}[leftmargin=*]
\item \textcolor{black}{We give a theoretical analysis of the probability estimation problem for a high-dimensional logistic model, and establish that predictors trained by minimizing the cross entropy loss overfit the observed outcomes and fail to yield calibrated outcomes. However, we show that the predictions \emph{are} well calibrated during the initial stages of training.}
\item We introduce a new synthetic dataset for probability estimation where a population of people may have a certain disease associated with age. The task is to estimate the probability that they contract the disease from an image of their face. The data are generated using the UTKFaces dataset~\citep{zhifei2017cvpr}, which contains age information. The dataset contains multiple versions of the synthetic labels, which are generated according to different distributions designed to mimic real-world probability-prediction datasets. 
The dataset serves two objectives. First, 
it allows us to evaluate existing methodology. 
Second, it enables us to evaluate different metrics in a controlled scenario where we have access to ground-truth probabilities. 


\item We have used publicly available data to build probability-estimation benchmark datasets for three real-world applications: 
(1) precipitation forecasting from radar images, (2) prediction of cancer-patient survival from histopathology images, and (3) prediction of vehicle collisions from dashcam videos. We use these datasets to systematically evaluate existing approaches, which have been previously tested mainly on classification datasets.

\item We propose Calibrated Probability Estimation (CaPE), a novel technique which modifies the training process so that output probabilities are consistent with empirical probabilities computed from the data. CaPE outperforms existing approaches on most metrics on  synthetic and real-world data.
\end{itemize}

\begin{figure}[t]
    \centering
    \begin{subfigure}[b]{\columnwidth}
         \centering
         \includegraphics[width=0.99\textwidth]{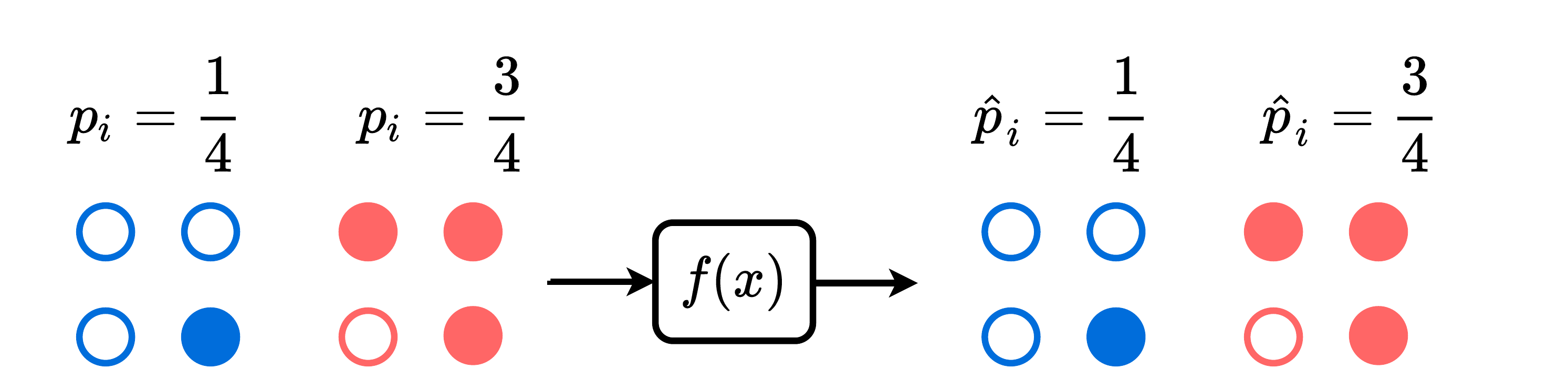}
         \caption{\textbf{\color{forestgreen}\checkmark} Accurate probability estimation\\ \hspace{-7em}\textcolor{forestgreen}{\checkmark} Calibration}
         \label{fig:prob_pred}
     \end{subfigure}\\\hspace{-3em}
     \begin{subfigure}[b]{0.9\columnwidth}
         \centering
         \includegraphics[width=0.94\textwidth]{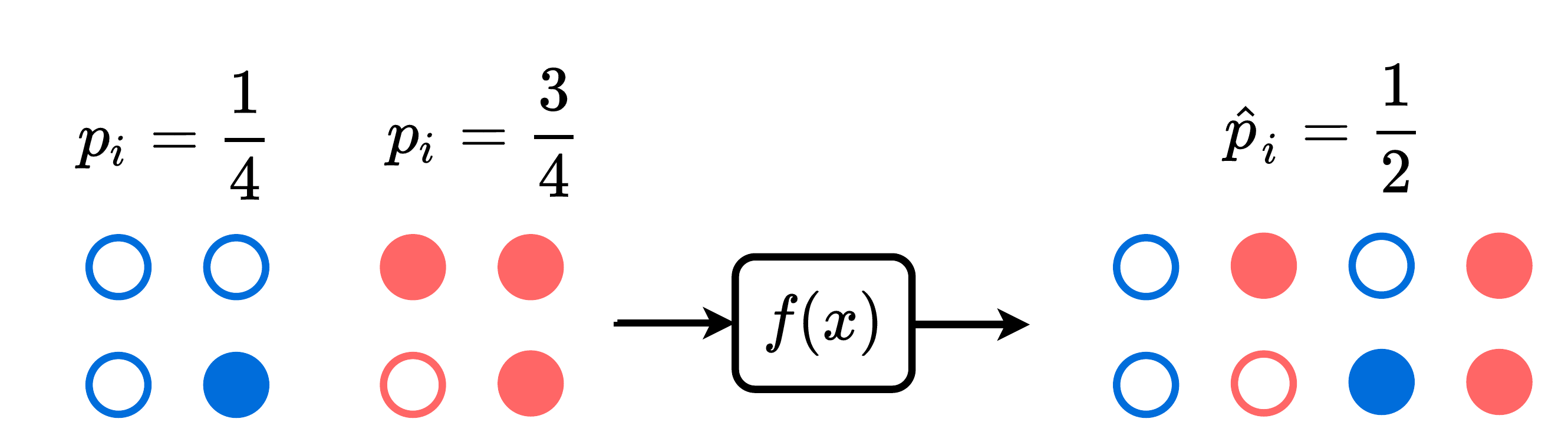}
    \caption*{\hspace{3em} (b)\hspace{.2em}\textcolor{black}{\ding{55}}  Inaccurate probability estimation \\\hspace{-4.25em} \textcolor{forestgreen}{\checkmark} Calibration}
    \label{fig:pred_diag_2}
     \end{subfigure}
     \caption{\textbf{Calibration is not enough}. Uncolored/colored markers denote $y=0/1$ outcomes, respectively. Blue/red stand for two classes with different associated ground-truth probabilities (1/4 and 3/4 respectively).  $(a)$ The model $f$ retrieves the true probabilities, which requires discriminating between inputs with low and high probability. $(b)$ The model $f$ has no discriminative power, it just assigns the same probability to all outputs. However, the model is \emph{perfectly calibrated} because out of all outcomes assigned 0.5 by the model, the fraction that are equal to 1 is 50\%.}
     \label{fig:pred_diag}
\end{figure}
\section{Problem Formulation}
The goal of probability estimation is to evaluate the likelihood of a certain event of interest, based on observed data. The available training data consist of $n$ examples $\vx_i$, $1 \leq i \leq n$, each associated with a corresponding outcome  $y_i$. In our applications of interest, the input data are high dimensional: each $\vx_i$ corresponds to an image or a video. The corresponding label $y_i$ is either 0 or 1 depending on whether or not the event in question occurred. For example, in the cancer-survival application $\vx_i$ is a histopathology image of a patient, and $y_i$ equals 1 if the patient survived for 5 years after $\vx_i$ was collected. 
The data have inherent uncertainty: $y_i$, the patient's survival, does not depend deterministically on the histopathology image \textcolor{black}{(due e.g. to comorbidities and other health factors)}. Instead, we assume that $y_i$ equals 1 with a certain probability $p_i$ associated with $\vx_i$, as illustrated in Figure~\ref{fig:illustration},  \textcolor{black}{because the input data provides key information about the patient's survival chances}. 

At inference, a probability-estimation model aims to generate an estimate $\hat{p}$ of the underlying probability $p$, associated with a new input data point $\vx$ (e.g. the probability of survival for over 5 years for new patients based on their histopathology data). \textcolor{black}{To summarize, this is not just a classification problem, because it involves aleatoric uncertainty. Instead, the goal is to predict the probability of the outcome, which is critical in choosing a course of treatment for the patient.} 

\section{Evaluation Metrics}
\label{sec:metrics}
Probability estimation shares similar target labels and network outputs with binary classification. 
However, classification accuracy is \emph{not} an appropriate metric for evaluating probability-estimation models due to the inherent uncertainty of the outcomes. This is illustrated by the example in Figure~\ref{fig:prob_pred} where a perfect probability estimate would result in a classification accuracy of just 75\%.\footnote{\textcolor{black}{A perfect model (in terms of probability estimation), assigns 0.25 to the blue class and 0.75 to the red class. To maximize classification accuracy, we predict 1 when the model outputs 0.75 (red examples) and 0 when it outputs 0.25 (blue examples). However, 25\% of red examples have an outcome of 0, and 25\% of blue examples have an outcome of 1. As a result, the model would only have 75\% accuracy.}} 

\textbf{Metrics when ground-truth probabilities are available.}
 For synthetic datasets, we have access to the ground truth probability labels and can use them to evaluate performance. 
 %
%
%
Two reasonable metrics are the mean squared error or $\ell_2$ distance ${\text{MSE}_p}$, and the Kullback–Leibler divergence ${\text{KL}_p}$ between the estimated and ground-truth probabilities:
\begin{align}
    \text{MSE}_{p} & =\frac{1}{N}\sum\limits_{i=1}^N (\hat{p}_i - p_i)^2, \notag\\
     \text{KL}_{p} & = \frac{1}{N}\sum\limits_{i=1}^N \left(\hat{p}_i \log\left(\frac{\hat{p}_i}{p_i}\right) + (1-\hat{p}_i) \log\left(\frac{1-\hat{p}_i}{1-p_i}\right)\right).\notag
\end{align}
$N$ is the number of data points, and $p_i,\hat p_i$ are the ground-truth and predicted probabilities, respectively.

\begin{figure*}
    \centering
    \includegraphics[width=0.9\textwidth]{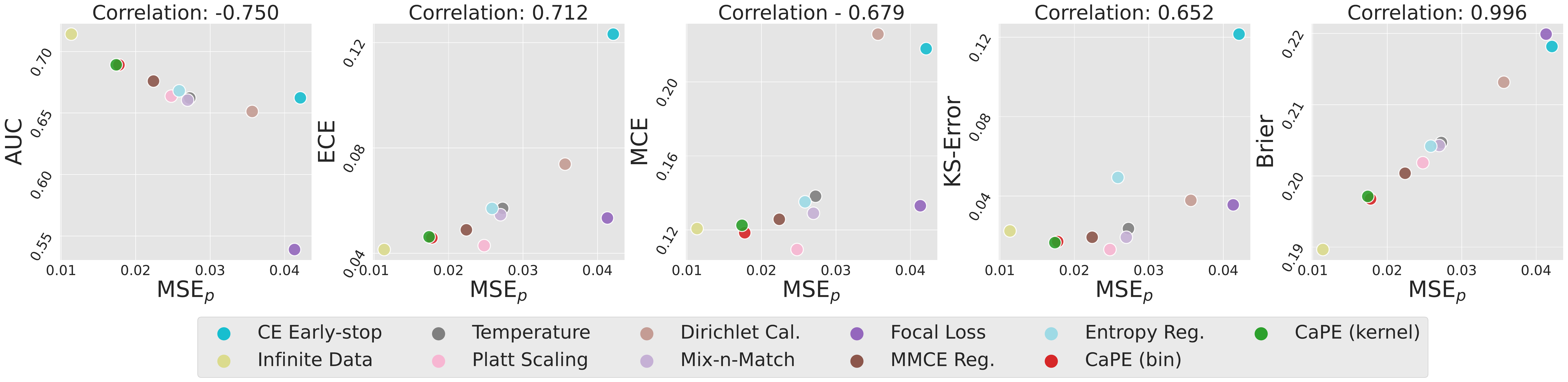}
    \caption{\textbf{Evaluating evaluation metrics.} We use synthetic data to compare different metrics to the \emph{gold-standard} $\text{MSE}_p$ that uses ground-truth probabilities. Brier score is highly correlated with $\text{MSE}_p$, in contrast to the classification metric AUC and the calibration metrics ECE, MCE and KS-Error. The graphs show the results of the proposed method CaPE, as well as the baselines described in Section~\ref{sec:baselines} on the \textit{Linear} scenario (see Section~\ref{sec:synthetic_data}). Results on other scenarios and a similar comparison with $\text{KL}_p$ are reported in Appendix~\ref{app:metric_comparsion_appendix}.}
    \label{fig:metric_comparison}
\end{figure*}
\textbf{Calibration metrics.} 
Ground-truth probabilities are not available for real-world data. In order to evaluate the probabilities estimated by a model, we need to compare them to the observed probabilities. To this end, we aggregate the examples for which the model output equals a certain value (e.g. 0.5), and verify what fraction of them have outcomes equal to 1. If the fraction is close to the model output, then the model is said to be well calibrated.  
%
\begin{definition}
\label{def:calibration}
A model $f$ is well \emph{calibrated} if
\begin{equation}
\mathbb{P}\left(y=1\mid  f(\vx) \in I(q)\right)=q, \quad \forall\hspace{0.2em} 0\le q\le 1,
\label{eq:cal_def}
\end{equation}
where $y$ is the observed outcome, $f(\vx)$ is the probability predicted by model $f$ for input $\vx$, and $I(q)$ is a small interval around $q$. 
%
\end{definition}
Model calibration can be evaluated using the expected calibration error ({ECE}) \citep{guo2017calibration} \textcolor{black}{(note however that the definition in \cite{guo2017calibration} is specific to classification)}. 
Given a probability-estimation model $f$ and a dataset of input data $\vx_i$ and associated outcomes $y_i$, $1\leq i \leq N$, we partition the examples into $B$ bins, $I_1,
I_2, \cdots, I_B$, according to the probabilities assigned to the examples by the model. Let $Q_1$,\ldots, $Q_{B-1}$ the $B$-quantiles of the set $\{f(\vx_1),\ldots,f(\vx_N)\}$, we have $I_b:=[Q_{b-1},Q_b] \cap \{f(\vx_i)\}_{i=1}^N$ (setting $Q_0=0$). For each bin, we compute the mean empirical and predicted probabilities,
\begin{align}
    p_\text{emp}^{(b)} & =\mathbb{E}\left(y \mid f(\vx) \in I_b\right) = \frac{1}{|I_b|}\sum\limits_{i\in \text{Index}(I_b)}y_i,
    \label{eq:p_emp}\\
    q^{(b)}& =\frac{1}{|I_b|}\sum\limits_{i\in \text{Index}(I_b)}f(\vx_i),
    \label{eq:q}
\end{align}
where $\text{Index}(I_b) = \{i \mid f(\vx_i) \in I_b \}$.

The pairs $(q^{(b)},p_\text{emp}^{(b)})$ can be plotted as a reliability diagram, shown in the second row of Figure~\ref{fig:overfit} and in Figure~\ref{fig:real_world_reliablility_curve}. ECE is then defined as
\begin{equation}
    \text{ECE}=\frac{1}{B}\sum\limits_{b=1}^B\left| p_\text{emp}^{(b)}-q^{(b)}\right|.
\end{equation}
Other metrics for calibration include the maximum calibration error ({MCE}) defined as $$\text{MCE}=\max\limits_{b=1,\dots,B}\left| p_\text{emp}^{(b)}-q^{(b)}\right|,$$ and the Kolmogorov-Smirnov error ({KS-error})~\citep{spline}, a metric based on the cumulative distribution function, which is described in more detail in Appendix~\ref{app:KS_error}.

\textbf{Brier score. }
Crucially, a model without any discriminative power can be perfectly calibrated (see Figure~\ref{fig:pred_diag}).
The Brier score is a metric designed to evaluate both calibration and discriminative power. It is the mean squared error between the predicted probability and the observed outcomes:
    \begin{equation}
        \text{Brier}=\frac{1}{N}\sum\limits_{i=1}^N (\hat{p}_i-y_i)^2.
    \end{equation}
This score can be decomposed into two terms associated with calibration and discrimination, as shown in Appendix~\ref{app:brier_decomp}. 
Using the synthetic data in Section~\ref{sec:synthetic_data}, where the ground-truth probabilities are known,
we show that Brier score is indeed a reliable proxy for the  \emph{gold-standard} MSE metric based on ground-truth probabilities $\text{MSE}_p$, in contrast to calibration metrics such as ECE, MCE or KS-error, and to classification metrics such as AUC (see Figure \ref{fig:metric_comparison} and Appendix~\ref{app:metric_comparsion_appendix}). 
    
%
\section{Early Learning and Memorization}
\label{sec:early_learning_model}
Prediction models based on deep learning are typically trained by minimizing the cross entropy between the model output and the training labels~\citep{goodfellow2016deep}. This cost function is a \emph{proper scoring rule}, meaning that it evaluates probability estimates in a consistent manner and is therefore guaranteed to be well calibrated in an infinite-data regime~\citep{Buja05lossfunctions}, as illustrated by Figure~\ref{fig:overfit} (first column). 

Unfortunately, in practice, prediction models are trained on finite data. This is crucial in the case of deep neural networks, which are highly overparametrized and therefore prone to overfitting~\citep{goodfellow2016deep}. For classification, deep neural networks have been shown to be capable of fitting arbitrary random labels~\citep{Zhang2017UnderstandingDL}.  
In probability estimation, we observe that neural networks indeed eventually overfit and \emph{memorize} the observed outcomes completely. Moreover, the estimated probabilities collapse to 0 or 1 (Figure~\ref{fig:overfit}, second column), a phenomenon that has also been reported in classification~\citep{Mukhoti2020CalibratingDN}. 
However, calibration is preserved during the first stages of training (Figure~\ref{fig:overfit}, third column). This is reminiscent of the \emph{early-learning} phenomenon observed for classification from partially corrupted labels~\citep{yao2020searching,xia2020robust}, where neural networks learn from the correct labels before eventually overfitting the false ones~\citep{liu2020early}. 

\textcolor{black}{Though early learning and memorization are typically observed when training prediction models based on deep neural networks, we argue that these observations represent a much more general phenomenon, intrinsic to the problem of probability estimation with finite data when the dimension is large. To substantiate this claim, we propose a simple analytical model, where data samples $\vx_i\in\mathbb{R}^d$ are drawn from a high dimensional normal distribution $\vx_i\sim\mathcal{N}(0,I_d)$. The probability of each data point is determined by a generalized linear model $p_i({\vtheta})=\left(1+e^{-\langle\vtheta,\vx_i\rangle}\right)^{-1}$, with true parameter $\vtheta^*$.
For $k \geq 1$
we denote by $\hat p^k$ the predictor obtained by running $k$ iterations of gradient descent on the cross-entropy loss with step size $\eta$.
We prove the following.
\begin{theorem}[Informal]\label{thm}
There exists $\kappa^* \in (0, + \infty)$ such that the following holds:
if $p$ and $n$ are sufficiently large, the mean squared error of $\hat p^k$ decreases monotonically during the first $k = O(1/\eta)$ iterations of gradient descent, but if $\frac pn > \kappa^*$, then as $k \to \infty$, $\hat p^k$ collapses to a predictor that only predicts probabilities $0$ and $1$.
\end{theorem}
A precise statement and proof can be found in Appendix~\ref{app:analytical_model}.
Theorem \ref{thm} identifies a sharp threshold ($\kappa^*$) at which the memorization phenomenon occurs and separates early learning stage and memorization. It indicates that even simple generalized linear models exhibit the early learning and memorization phenomena: in high dimensions, predictors obtained by cross-entropy minimization eventually memorize the data. This is likely to plague any overparamterized model, including neural networks. However, this phenomenon does \emph{not} occur if gradient descent is stopped early. This is illustrated in Figures~\ref{fig:leraning_curve_analytic} and \ref{fig:scatter_plot_analytic} in Appendix~\ref{app:linear_model}, which demonstrate that
empirically the linear model has this qualitative behavior. 
These observations motivate our proposed methodology.}

\section{Calibrated Probability Estimation ($\text{CaPE}$)}
\label{sec:our_method}

%
\begin{figure*}[t]
    \centering
    \includegraphics[width=0.9\textwidth]{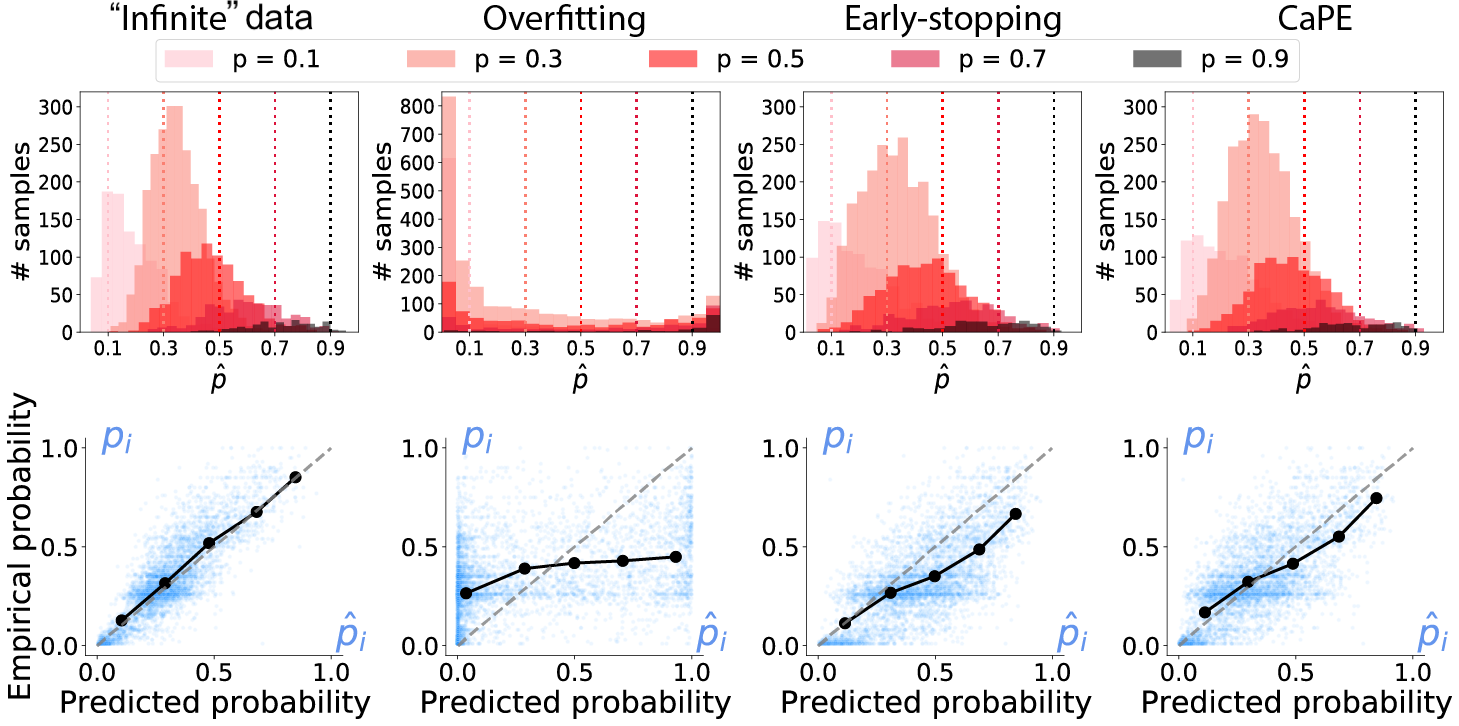}
    \caption{\textbf{Miscalibration due to overfitting and how to avoid it.} 
    The top row shows the histogram of predicted probabilities for the synthetic \emph{Discrete} scenario (see Section~\ref{sec:synthetic_data}). Ideally each histogram should be concentrated around the corresponding value of $p$. The bottom row shows results for the \emph{Linear} scenario. The horizontal and vertical coordinate of each blue point represent the predicted ($\hat{p}_i$) and true ($p_i$) probabilities of a test example respectively. We also show a reliability diagram of binned mean predicted and empirical probabilities in black (see Section~\ref{sec:metrics}). The dashed line indicates perfect calibration. When trained on \emph{infinite} data  \textcolor{black}{obtained by resampling outcome labels at each epoch according to ground-truth probabilities},  models minimizing cross-entropy are well calibrated (first column).  However, when trained on fixed observed outcomes, the model eventually overfits and the probabilities collapse to either 0 or 1 (second column). This is mitigated via early stopping \textcolor{black}{(i.e. selecting the model based on validation cross-entropy loss)}, which yields relatively good calibration (third column). The proposed Calibration Probability Estimation (CaPE) method exploits this to further improve the model while ensuring that the output remains well calibrated. Appendix~\ref{app:app_reliability} shows plots for all synthetic data scenarios.}
    \label{fig:overfit}
\end{figure*}

We propose to exploit the training dynamics of cross-entropy minimization through a method that we name \emph{Calibrated Probability Estimation} (CaPE). Our starting point is a model obtained via early stopping using validation data on the cross-entropy loss. CaPE is designed to further improve the discrimination ability of the model, while ensuring that it remains well calibrated. 
This is achieved by alternatively minimizing the following two loss functions: 

\noindent \textbf{Discrimination loss}: Cross entropy between the model output and the observed binary outcomes,
\begin{equation*}
\mathcal{L}_{\text{D}} = -\sum_{i=1}^N \left[ y_i\log(f(\vx_i)) + (1-y_i) \log(1-f(\vx_i))\right]
\end{equation*}
%
\noindent
\textbf{Calibration loss}: Cross entropy between the output probability of the model and the empirical probability of the outcomes conditioned on the model output:
\begin{equation*}
\mathcal{L}_{\text{C}} = -\sum_{i=1}^N\left[ p^i_{\text{emp}}\log(f(\vx_i)) + (1-p^i_{\text{emp}})\log(1-f(\vx_i))\right],
\end{equation*}
where $p_{\text{emp}}^i$ is an estimate of the conditional probability $\mathbbm{P}[y = 1| f(\vx) \in I( f(\vx_i) )]$ and $I( f(\vx_i) )$ is a small interval centered at $f(\vx_i)$. 
As explained in Section~\ref{sec:metrics} if $f(\vx_i)$ is close to this value, then the model is well calibrated. We consider two approaches for estimating $p_{\text{emp}}^i$. 
(1) CaPE (bin) where we divide the training set into bins, select the bin $b_i$ containing $f(x_i)$ and set $p_{\text{emp}}^i=p_{\text{emp}}^{(b_i)}$ in~\eqref{eq:p_emp}. 
(2) CaPE (kernel) where $p_{\text{emp}}^i$ is estimated through a moving average with a kernel function (see Appendix \ref{app:emp_estimation} for more details). Both methods are efficiently computed by sorting the predictions $\hat p_i$. \textcolor{black}{The calibration loss requires a reasonable estimation of the empirical probabilities  $p_{\text{emp}}^{(i)}$, which can be obtained from the model after early learning. Therefore using the calibration loss from the beginning is counterproductive, as demonstrated in Section~\ref{sec:beginning}.} 
We note that a variant of CaPE can be implemented using a weighted sum of the calibration and discrimination loss. 

CaPE is summarized in Algorithm~\ref{alg:algo_cali}.
Figures~\ref{fig:overfit} \textcolor{black}{and~\ref{fig:training_curve}} show that incorporating the calibration-loss minimization step indeed preserves calibration as training proceeds (this is not necessarily expected because CaPE minimizes a calibration loss \emph{on the training data}), and prevents the model from overfitting the observed outputs. This is beneficial also for the discriminative ability of the model, because it enables it to further reduce the cross-entropy loss without overfitting, \textcolor{black}{as shown in Figure~\ref{fig:training_curve}}. The experiments with synthetic and real-world data reported in Section~\ref{sec:experiments} suggest that this approach results in accurate probability estimates across a variety of realistic scenarios. 

\begin{figure*}[t]
    \centering
    \includegraphics[width=1.00\textwidth]{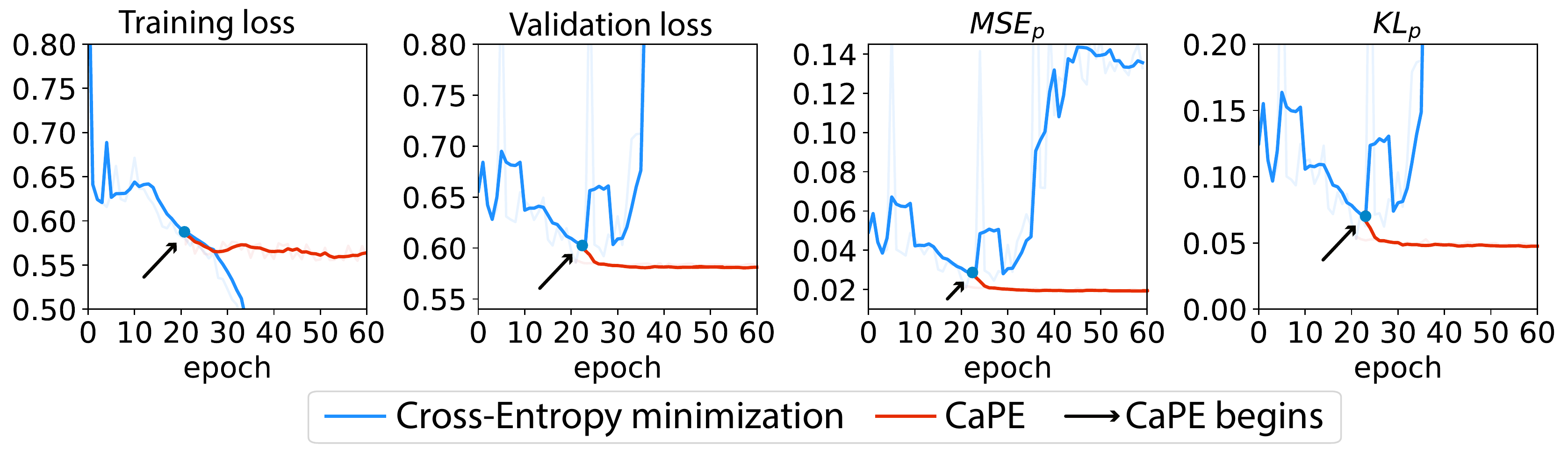}
    \caption{\textbf{Calibrated Probability Estimation prevents overfitting.}
    Comparison between the learning curves of cross-entropy (CE) minimization and the proposed calibrated probability estimation (CaPE), smoothed with a 5-epoch moving average. After an \emph{early-learning} stage where both training and validation losses decrease, CE minimization overfits (first and second graph), with disastrous consequences in terms of probability estimation (third and fourth graph).
    In contrast, CaPE prevents overfitting, continuing to improve the model while maintaining calibration (see Figure~\ref{fig:overfit}).}
    \label{fig:training_curve}
\end{figure*}

\input{algorithm}

\section{Related Work}
\label{sec:related_work}
 Neural networks trained for classification often generate a probability associated with their prediction which quantifies its uncertainty. These estimates have been found to be inaccurate in certain situations~\citep{Mukhoti2020CalibratingDN,guo2017calibration,zhao2021right} (although a recent study suggests that  transformer-based models tend to be well calibrated in vision-based classification tasks~\citep{minderer2021revisiting}). 
Calibration methods to mitigate this issue broadly fall into three categories depending on whether they: (1) postprocess the outputs of a trained model, (2) combine multiple model outputs, or (3) modify the training process. 

\textbf{Post-processing methods} transform the output probabilities in order to improve calibration on held-out data~\citep{Zadrozny2001ObtainingCP, spline,betacalibration,dirchlet}. For example, Platt scaling~\citep{Platt1999} fits a logistic function that minimizes the negative log-likelihood loss. Temperature scaling~\citep{guo2017calibration} does the same with a temperature parameter augmenting the softmax function. Another approach trains a recalibration model on the outputs of an uncalibrated model~\citep{kuleshov2018accurate}.
In contrast to these methods, CaPE enforces calibration \emph{during training}, which has the advantage of enabling further improvements in the discriminative abilities of the model.

\textbf{Ensembling methods} combine multiple models to improve generalization. Mix-n-Match~\citep{Zhang2020MixnMatchEA} uses a single model, and ensembles predictions using multiple temperature scaling transformations. Other methods \citep{deepensemble, swag} ensemble multiple models obtained using different initializations. These approaches are compatible with the proposed method CaPE; how to combine them effectively is an interesting future research direction.

\textbf{Modified training methods} can be divided into two groups. The first group smooths the target 0/1 labels in order to prevent output estimates from collapsing to 0/1~\citep{Mukhoti2020CalibratingDN, Szegedy2016RethinkingTI, Zhang2018mixupBE, thulasidasan2020mixup}. The second group, attaches additional calibration penalties to a cross entropy loss~\citep{MMCE, entropy, liang2020imporved}. CaPE is most similar in spirit to the latter methods, although its data-driven calibration loss is different to the penalties used in these techniques. 

\textbf{Datasets for evaluation} The methods discussed in this section were developed for calibration in classification, and tested on datasets such as CIFAR-10/100~\citep{Krizhevsky09learningmultiple}, SVHN~\citep{netzer2011reading}, and ImageNet~\citep{imagenet} where the relationship between labels and input data is \emph{completely deterministic}. Here, we evaluate these methods on synthetic and real-world probability-estimation problems with inherent uncertainty. 

\begin{table*}[t]
    \centering
    \scriptsize{
   \begin{tabular}{lc@{\hspace{0.2cm}}c|c@{\hspace{0.2cm}}c|c@{\hspace{0.2cm}}c|c@{\hspace{0.2cm}}c|c@{\hspace{0.2cm}}c}
\toprule
   \small{Methods}  &      \multicolumn{2}{c|}{\small{\textit{Linear}}} &    \multicolumn{2}{c|}{\small{\textit{Sigmoid}}}         &       \multicolumn{2}{c|}{\small{\textit{Centered}}}        &         \multicolumn{2}{c|}{\small{\textit{Skewed}}}  &\multicolumn{2}{c}{\small{\textit{Discrete}}}           \\
 \multicolumn{1}{r}{$\small{(\times 10^{-2})}$} &  $\small{\text{MSE}_p}$ &  $\small{\text{KL}_p}$ &  $\small{\text{MSE}_p}$ &  $\small{\text{KL}_p}$  &  $\small{\text{MSE}_p}$ &  $\small{\text{KL}_p}$ &  $\small{\text{MSE}_p}$ &  $\small{\text{KL}_p}$ &  $\small{\text{MSE}_p}$ &  $\small{\text{KL}_p}$ \\
\midrule
CE + resampled labels* & $1.14_{\pm.04}$ & $2.81_{\pm.11}$ & $5.34_{\pm.20}$ & $14.82_{\pm.51}$ & $4.21_{\pm.15}$ & $4.21_{\pm.15}$ & $4.21_{\pm.15}$ & $4.21_{\pm.15}$ & $4.21_{\pm.15}$ & $4.21_{\pm.15}$\\
\midrule
CE early-stop   &   $4.21_{\pm.15}$ &  $10.94_{\pm.36}$ & $6.16_{\pm.17}$ &  $17.16_{\pm.48}$ &        $0.48_{\pm.01}$ &        $0.98_{\pm.03}$ &        $0.40_{\pm.01}$ &        $1.79_{\pm.06}$ &        $2.24_{\pm.08}$ &        $5.27_{\pm.17}$ \\
Temperature      &  $2.73_{\pm.11}$ &  $6.75_{\pm.25}$ &  $6.16_{\pm.17}$ &   $17.09_{\pm.43}$ & $0.48_{\pm.01}$ &        $0.98_{\pm.03}$ &        $0.40_{\pm.02}$ &        $1.76_{\pm.06}$ &        $2.21_{\pm.08}$ &        $5.15_{\pm.18}$ \\
Platt Scaling & $2.48_{\pm.09}$ &  $6.07_{\pm.22}$ & $5.78_{\pm.19}$ & $16.15_{\pm.47}$ &        $0.41_{\pm.01}$ &        $0.83_{\pm.03}$ &        $\mathbf{0.39}_{\pm.01}$ &        $1.72_{\pm.06}$ &        $2.06_{\pm.08}$ &        $4.83_{\pm.17}$ \\
Dirichlet Cal.      &  $3.56_{\pm.13}$ & $9.08_{\pm.29}$ &  $8.64_{\pm.26}$ & $25.18_{\pm.58}$ &        $0.46_{\pm.01}$ &        $0.94_{\pm.03}$ &        $0.47_{\pm.02}$ &        $2.31_{\pm.07}$ &        $2.74_{\pm.10}$ &        $6.53_{\pm.22}$ \\
Focal Loss        & $4.13_{\pm.11}$ & $10.52_{\pm.28}$ &  $6.86_{\pm.21}$ &       $19.46_{\pm.50}$ &        $0.48_{\pm.01}$ &        $0.97_{\pm.03}$ &        $1.28_{\pm.03}$ &       $1.63_{\pm.66}$ &        $2.92_{\pm.08}$ &        $6.77_{\pm.21}$ \\
Mix-n-match      &  $2.70_{\pm.11}$ &   $6.72_{\pm.24}$ &  $6.12_{\pm.17}$ &       $17.08_{\pm.46}$ &        $0.48_{\pm.01}$ &        $0.98_{\pm.03}$ &        ${0.40}_{\pm.01}$ &        $1.75_{\pm.05}$ &        $2.21_{\pm.08}$ &        $5.14_{\pm.18}$ \\
Entropy Reg.         &  $2.58_{\pm.09}$ &  $6.65_{\pm.21}$ & $7.02_{\pm.17}$ &       $21.16_{\pm.42}$ &        $0.45_{\pm.01}$ &        $0.92_{\pm.03}$ &        $1.18_{\pm.03}$ &       $10.74_{\pm.65}$ &        $2.84_{\pm.08}$ &        $6.62_{\pm.19}$ \\
MMCE  Reg.   &  $2.24_{\pm.08}$ &  $5.68_{\pm.20}$ &  $5.35_{\pm.18}$ & $15.06_{\pm.49}$ &        $0.44_{\pm.01}$ &        $0.90_{\pm.03}$ &        $0.54_{\pm.02}$ &        $2.44_{\pm.08}$ &        $2.09_{\pm.08}$ &        $4.92_{\pm.18}$ \\
Deep Ensemble    & $1.90_{\pm.07}$ &  $4.55_{\pm.18}$ & $5.86_{\pm.22}$ &   $16.46_{\pm.60}$ &        $0.44_{\pm.01}$ &        $0.89_{\pm.03}$ &        $0.55_{\pm.02}$ &        $2.58_{\pm.07}$ &        $1.97_{\pm.08}$ &        $4.61_{\pm.17}$ \\
\midrule
CaPE (bin)       &  $1.78_{\pm.07}$ &  $4.35_{\pm.16}$ &   $5.17_{\pm.20}$ &       $\mathbf{14.27}_{\pm.49}$ &        $\mathbf{0.38}_{\pm.01}$ &        $\mathbf{0.78}_{\pm.03}$ &        $0.40_{\pm.02}$ &        $1.73_{\pm.06}$ &        $\mathbf{1.81}_{\pm.08}$ &        $\mathbf{4.28}_{\pm.18}$ \\
CaPE (kernel)        &  {$\mathbf{1.74}_{\pm.07}$} &  {${4.30_{\pm.17}}$} &        $\mathbf{5.16}_{\pm.20}$ & $14.34_{\pm.49}$  &        $0.40_{\pm.01}$ &        $0.81_{\pm.03}$  &        $\mathbf{0.39}_{\pm.01}$ &        $\mathbf{1.69}_{\pm.06}$ &        $1.84_{\pm.08}$ &        $4.35_{\pm.17}$ \\

\bottomrule

\end{tabular}
}
    \caption{Results on synthetic data. All numbers are downscaled by $10^{-2}$. Appendix~\ref{app:face_results} shows a table with confidence intervals obtained via bootstrapping. * is a model trained via cross-entropy minimization from data obtained by continuous label resampling. No baseline outperforms the proposed method CaPE in any of the scenarios, and CaPE outperforms all the individual baseline models in most scenarios, in both metrics, except for the skewed case, where the difference is statistically insignificant.
    }
    \label{tab:syntheic_data_metrics}
\end{table*}

\section{Experiments}
\label{sec:experiments}
\subsection{Synthetic Dataset: Face-based Risk Prediction}
\label{sec:synthetic_data}
To benchmark probability-estimation methods, we build a synthetic dataset based on UTKFace~\citep{zhifei2017cvpr}, containing face images and associated ages. 
We use the age of the $i$th person $z_i$ to assign them a risk of contracting a disease $p_i= \psi(z_i)$ for a fixed function $\psi: \mathbb{Z}_{\ge 0} \rightarrow [0,1]$. Then we simulate whether the person actually contracts the illness (label $y_i=1$) or not ($y_i=0$) with probability $p_i$. The probability-estimation task is to estimate the ground-truth probability $p_i$ from the face image $x_i$ using a model that only has access to the images and the binary observations during training. This requires learning to discriminate age and map it to the corresponding risk. We design $\psi$ to create five scenarios, inspired by real-world data \textcolor{black}{(see  Appendix~\ref{app:syntheic})}:



\begin{itemize}[leftmargin=*]
    \item \textbf{Linear}: Equally-spaced, inspired by weather forecasting: \\
    $\psi (z) = z / 100$
    \item \textbf{Sigmoid}: Concentrated near two extremes:\\
    $\psi (z) =  \sigmoid(25(z/100 - 0.29))$
    \item \textbf{Skewed}: Clustered close to zero, inspired by vehicle-collision detection: 
    $\psi (z) = z / 250$
    \item \textbf{Centered}: Clustered in the center, inspired by cancer-survival prediction:
    $\psi (z) = z / 300 + 0.35$
    \item \textbf{Discrete}: Discretized:
    $\psi (z) = 0.2 [ \mathbbm{1}_{\{z > 20\}} + \mathbbm{1}_{\{z > 40\}} + \mathbbm{1}_{\{z > 60\}} + \mathbbm{1}_{\{z > 80\}} ]+0.1$
\end{itemize}

In addition, we report an experiment with simulated probabilistic labels on CIFAR-10 in Appendix~\ref{app:cifar10}. 

\begin{figure*}[t]
    \centering
    \includegraphics[width=0.25\textwidth]{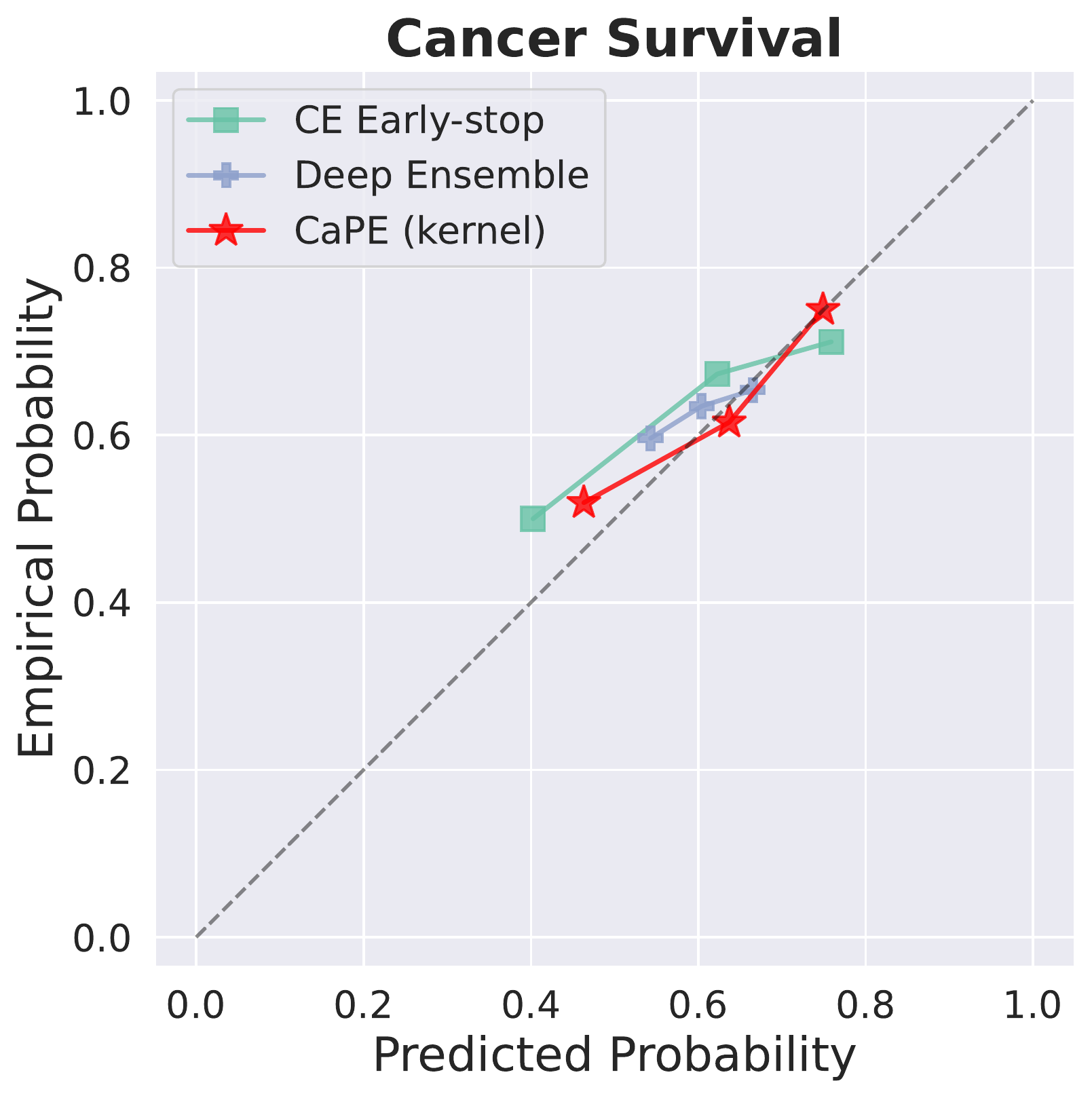}\hspace{2em}
    \includegraphics[width=0.25\textwidth]{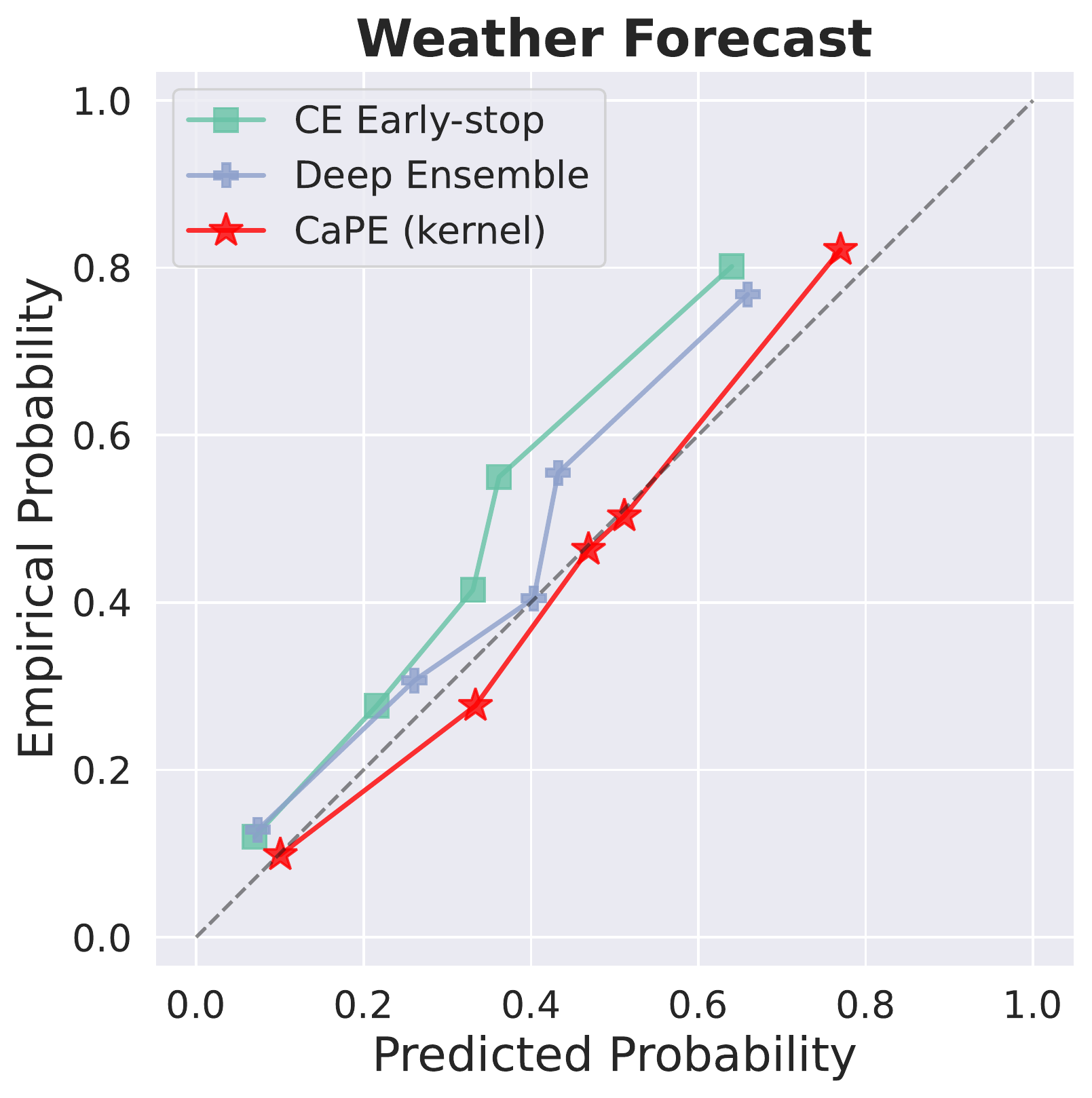}\hspace{2em}
    \includegraphics[width=0.25\textwidth]{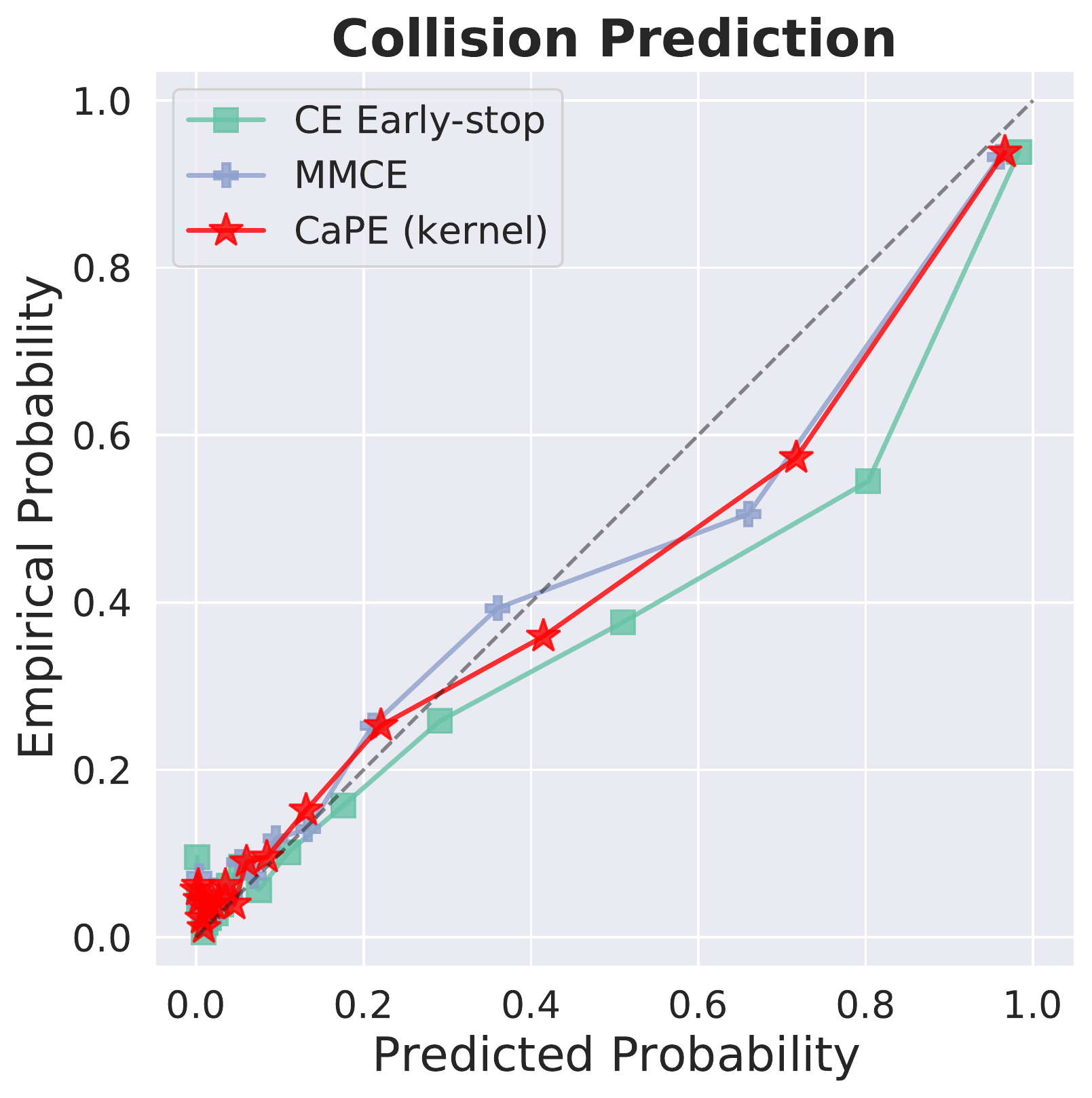}
    \caption{\textbf{Reliability diagrams for real-world data.} Reliability diagrams of binned mean predicted and empirical probabilities  (see Section~\ref{sec:metrics}). The dashed line indicates perfect calibration. The results are computed on test data for cross-entropy minimization with early stopping, the proposed method (CaPE) and the best baseline for each dataset. 
    CaPE produces better calibrated outputs. Appendix~\ref{app:app_reliability} shows additional reliability diagrams.}
    \label{fig:real_world_reliablility_curve}
\end{figure*}
\subsection{Real-World Datasets}
We use three open-source, real-world datasets to benchmark probability-estimation approaches (see Appendix~\ref{app:real-world} for further details on the datasets and experiments).

\textbf{Survival of Cancer Patients.} 
Histopathology aims to identify tumor cells, cancer subtypes, and the stage and level of differentiation of cancer. Hematoxylin and Eosin (H\&E)-stained slides are the most common type of histopathology data used for clinical decision making. In particular, they can be used for survival prediction~\citep{Wulczyn2020DeepLS}, 
which is critical in evaluating the prognosis of patients. 
Treatments assigned to patients after diagnosis are not personalized and their impact on cancer trajectory is complex, so the survival status of a patient is not deterministic. In this work, we use the H\&E slides of non-small cell lung cancers from The Cancer Genome Atlas Program (TCGA)\footnote{\url{https://www.cancer.gov/tcga}} to estimate the the 5-year survival probability of cancer patients. 
\textcolor{black}{The outcome distribution is similar to the \textit{Centered} scenario in Section~\ref{sec:synthetic_data}.}

\textbf{Weather Forecasting.}
The atmosphere is governed by nonlinear dynamics, hence weather forecast models possess inherent uncertainties
~\citep{richardson2007weather}. Nowcasting, weather prediction in the near future, is of great operational significance, especially with increasing number of extreme inclement weather conditions~\citep{agrawal2019machine, ravuri2021skillful}. We use the German Weather service dataset\footnote{\url{https://opendata.dwd.de/weather/radar/}}, which contains quality-controlled rainfall-depth composites from 17 operational Doppler radars. We use 30 minutes of precipitation data to predict if the mean precipitation over the area covered will increase or decrease one hour after the most recent measurement. Three precipitation maps from the past 30 minutes serve as an input. 
\textcolor{black}{The outcome distribution is similar to the \textit{Linear} scenario in Section~\ref{sec:synthetic_data}.}

\begin{table*}[t]
\centering
{
\scriptsize
    \begin{tabular}{lrrr|rrr|rrr}
\toprule
Method & \multicolumn{3}{c}{Cancer Survival} & \multicolumn{3}{c}{Weather forecasting} & \multicolumn{3}{c}{Collision Prediction}\\
\multicolumn{1}{r}{$(\times 10^{-2})$} &    AUC &    ECE   &  Brier  &  AUC &   ECE   &  Brier &  AUC &    ECE   &  Brier\\
\midrule
CE early-stop               &  58.88 &  12.25   &  23.96  & 77.64& 10.91 & 20.57 & 85.68 & 4.36 & 8.59\\
Temperature      &  58.88 &  12.07  &  23.73 &77.64 & 8.66 & 20.21 & 85.68 & 4.56 & 8.51\\
Platt Scaling &  58.91 &  10.28   &  23.33 & 77.65& 6.97 & 19.53 & 85.76 & 3.04 & 8.23\\
Dirichlet Cal. &  49.89 &  13.83  &  24.08 & 77.51& 14.29 & 21.89 & 83.36 & 5.78 & 8.78\\
Mix-n-match &  58.88 &  12.16  &  23.67 &77.64 & 8.65 & 20.21  & 85.68 & 4.40 & 8.52\\
Focal Loss         &  55.02 &  12.15 &  23.31 &76.18 & 8.32 & 20.27 & 82.21 & 9.07 & 9.82\\
Entropy Reg.       &  56.29 &  11.73  &  23.62  &79.01 & 10.53 & 19.77 & 83.15 & 14.54 & 11.10\\
MMCE Reg.            & 48.45 & 11.84 & 23.73 &76.69 & 8.46 & 20.12 & 85.18 & \textbf{2.94} & 8.48\\
Deep Ensemble    &  52.46 &	9.99 &	23.47 &\textbf{79.86} & 7.41 & 18.82 & 85.27 & 3.15 & 8.55\\
\midrule
CaPE (bin)       & \textbf{ 61.44 }&  12.31 &  23.20 & 78.99& 5.16 & \textbf{18.37}& 85.70 & 3.16 & 8.18\\
CaPE (kernel)        &  61.22 &  \textbf{9.48} &  \textbf{23.18} & 79.00& \textbf{5.08} & 18.39 & \textbf{85.95} & 3.22 & \textbf{8.13}\\
\bottomrule
\end{tabular}
}
\caption{Results on cancer-survival prediction, weather forecasting, and collision prediction. All numbers are downscaled by $10^{-2}$. Tables with all the metrics described in Section~\ref{sec:metrics} are provided in Appendix~\ref{app:real_world_results}. The proposed method CaPE outperforms existing techniques in terms of Brier score, the metric that best captures probability-estimation accuracy.}
    \label{tab:real_world_metrics}
\end{table*}

\textbf{Collision Prediction.}
Vehicle collision is one of the leading causes of death in the world. Reliable collision prediction systems are therefore instrumental in saving human lives. These systems predict potential collisions from dashcam cameras. Collisions are influenced by many unknown factors, and hence are not deterministic. Following ~\cite{kim2019crash},  we use $0.3$ seconds of real dashcam videos from \texttt{YouTubeCrash} dataset as input, and predict the probability of a collision in the next $2$ seconds. \textcolor{black}{The data are very imbalanced as the number of collisions is very low, so the outcome distribution is similar to the \textit{Skewed} scenario in Section~\ref{sec:synthetic_data}.}

\subsection{Baselines}
\label{sec:baselines}

We apply existing calibration methods developed for classification to probability estimation (as well as cross-entropy minimization with early-stopping): (1) \textbf{Three post-processing methods}: Temperature Scaling~\citep{guo2017calibration}, Platt Scaling~\citep{Platt1999}, and Dirichlet Calibration~\citep{dirchlet} applied to the best CE model, (2) \textbf{Two Ensemble Methods}:  Mix-n-Match~\citep{Zhang2020MixnMatchEA} applied to best CE model, and Deep Ensemble~\citep{deepensemble} \textcolor{black}{with 5 networks}, and (3) \textbf{Three Modified Training methods}: Focal loss~\citep{Mukhoti2020CalibratingDN}, entropy-maximizing loss~\citep{entropy}, and MMCE regularization~\citep{MMCE}. Appendix \ref{app:baselines} provides a detailed description. For our experiments on synthetic data, we also compare against a model trained on a \textbf{large amount of data} by repeatedly sampling new outcomes from the ground-truth probabilities at each epoch. This provides a best-case reference for each scenario. We refer to Appendix~\ref{app:baselines}, for a detailed discussion of the baselines and hyperparameter optimization.
 
\section{Results and Discussion}
\label{sec:results}
%


Table \ref{tab:syntheic_data_metrics} shows that calibration methods developed for classification can be effective for probability estimation. 
However, the performance of some methods is not consistent across all scenarios. For instance, regularization with negative entropy, which penalizes very high/low confidence, performs worse than CE when the ground-truth probability is close to 0 or 1. In contrast, methods that do not make strong assumptions tend to generalize better to multiple scenarios (e.g. Platt scaling consistently beats CE).
The proposed method
CaPE outperforms other techniques in most scenarios, and even matches the performance of the best-case baseline with resampled labels for the \textit{Sigmoid} scenario.
 Finally, we observe that the \textit{Skewed} scenario is very challenging: most methods barely improve the CE baseline.


Table~\ref{tab:real_world_metrics} compares the baseline methods and CaPE on the three real-world datasets. We present AUC, ECE for 15 equally-sized bins, and Brier score, as complementary metrics since the underlying ground-truth probabilities are unobserved. As discussed in Section~\ref{sec:metrics}, 
Brier score is the metric that best captures the quality of  
probability estimates. CaPE has the lowest Brier score in all three datasets, while also achieving lower ECE values and higher AUC values than most other methods.  
This demonstrates that enforcing better probability estimation during training also yields a more discriminative model. The reliability diagrams in Figure~\ref{fig:real_world_reliablility_curve} depict the probability estimates produced by CE, CaPE and the best baseline method on the three datasets, demonstrating that CaPE produces outputs that are better calibrated on real data.

Figure~\ref{fig:real_world_reliablility_curve} also shows that each real-world dataset closely aligns with a particular synthetic scenario: cancer survival with
\textit{Centered}; weather forecasting with \textit{Linear}; collision prediction with \textit{Skewed}. This supports the significance of our synthetic benchmark dataset, and provides insights in the differences among baseline models. For example, model averaging with deep ensemble performs well on weather forecasting but has higher Brier scores than Platt scaling on the other two datasets (see Appendix~\ref{app:case_study} for further analysis based on pathological stages). Accordingly, deep ensemble also underperforms in the synthetic scenarios where ground-truth probabilities are clustered closely (\emph{Sigmoid}, \emph{Centered}), but is  effective for \emph{Linear}. 
Finally, as in the synthetic \textit{Skewed} scenario, all methods had similar performance on the collision prediction task. This highlights the importance of considering different scenarios when evaluating methodology for probability estimation.

\section{Conclusion}

In this work we evaluate existing approaches to improve the output probabilities of neural networks on probability-estimation problems. To this end, we introduce a new synthetic benchmark dataset designed to reproduce several realistic scenarios, and also gather three real-world datasets relevant to medicine, climatology, and self-driving cars. In addition, we provide theoretical analysis showing that early learning and memorization are fundamental phenomena in high-dimensional probability estimation. Motivated by this, we
propose a novel approach that outperforms existing approaches on our simulated and real-world benchmarks. \textcolor{black}{An important application for probability estimation is in the context of survival analysis, which can be recast as estimation of conditional probabilities~ \citep{deephit, DBLP:journals/corr/abs-2008-01774, goldstein2021xcal}. 
Another interesting research direction is to consider problems with several possible uncertain outcomes (analogous to multiclass classification).}

\section*{Acknowledgements}
The authors gratefully acknowledge support from Alzheimer’s Association grant AARG-NTF-21-
848627 (S.L.), NIH grant R01 LM01331 (A.K., B.Y.), NSF grants HDR-1940097 (M.L.), DMS2009752 (C.F.G., S.L.) and NRT-1922658 (A.K., S.L., S.M., B.Y., W.Z.), and the NYU Langone
Health Predictive Analytics Unit (N.R., W.Z.). Funding for the Multiscale Machine Learning In coupled Earth System Modeling (M2LInES) project was provided to Laure Zanna by the generosity of Eric and Wendy Schmidt by recommendation of the Schmidt Futures program. We would also like to thank Juan Argote for useful
discussions, and the reviewers for their comments. 

\bibliography{example_paper}
\bibliographystyle{icml2022}

\newpage
\appendix
\onecolumn
\textcolor{black}{
\section{Analytical Model}
\label{app:analytical_model}}
In this appendix, we propose a simple analytical model that helps to demonstrate the early learning and memorization phenomena in probability estimation. Consider the following standard logistic regression problem: we are given i.i.d. data $\{(y_i,\bm{x_i})\}_{i\leq n}$ where $y_i\in\{+1,-1\}$ are labels\footnote{For mathematical convenience, in this appendix we work with labels in $\{ \pm 1\}$ rather than $\{0, 1\}$.} and $\bm{x_i}\sim \mathcal{N}(0,I_p)$ are standard Gaussian covariate vectors. Given $\bm{x_i}=\bm{x}$, the label $y_i$ is distributed according to
\begin{align}
    \mathbb{P}(y_i=1|\bm{x_i}=\bm{x})=\sigma(\langle\vtheta^*_p,\vx\rangle),\label{true_prob}
\end{align}
where $\sigma(x)=(1+e^{-x})^{-1}$ is the sigmoid function and $\vtheta^*_p\in\R^p$ is the ground-truth parameter with Euclidean norm $\lVert\vtheta^*_p\rVert=\gamma>0$. Throughout, we assume that we are working in the high dimensional regime: the ratio $\frac{p}{n}\to\kappa\in(0,+\infty)$ as $n\to\infty$. For convenience, we denote $N_0$ an integer such that $\frac{p}{n}<2\kappa$ for $\forall n>N_0$. \textcolor{black}{In Section~\ref{app:linear_model} of the appendix, we provide a numerical example to illustrate this theoretical model.}

Given a fresh sample $\vz\in\R^p$, we predict the probability of the associated label being one via $\sigma(\langle \hat{\vtheta},\vz\rangle)$ for some $\hat{\vtheta}\in\R^p$, and we train the estimator $\hat{\vtheta}$ through gradient descent to minimize the cross-entropy:
\begin{align}
    L_n(\vtheta) & =- \frac{1}{n}\sum_{i=1}^n\left\{\left(\frac{1+ y_i}{2}\right) \log \sigma(\langle \vtheta, \bm{x_i} \rangle) + \left(\frac{1- y_i}{2}\right) \log (1-\sigma(\langle \vtheta, \bm{x_i} \rangle)) \right\}\\
    & =\frac{1}{n}\sum_{i=1}^n\left\{-\frac{1}{2}y_i\langle\vtheta,\bm{x_i}\rangle+\log\left(e^{\frac{1}{2}\langle\vtheta,\bm{x_i}\rangle}+e^{-\frac{1}{2}\langle\vtheta,\bm{x_i}\rangle}\right) \right\}\\
    \hat{\vtheta}^{k+1}_{\eta,n} & =\hat{\vtheta}^{k}_{\eta,n}-\eta \nabla L_n(\hat{\vtheta}^{k}_{\eta,n}),\quad k=0,1,... \label{gradient_descent}
\end{align}
where we choose constant step size $\eta<\delta$, for some prefixed $\delta>0$, to be determined later. For the initialization, we draw $\hat{\vtheta}^{0}_{\eta,n}$ uniformly from the sphere of radius $\gamma_0$.

Define the mean squared error 
\begin{equation}
\text{MSE}_{\eta,n}(k):=\E[(\sigma(\langle \hat{\vtheta}^{k}_{\eta,n},\vz\rangle)-\sigma(\langle\vtheta^*_p,\vz\rangle))^2],\end{equation}
where $z\sim\mathcal{N}(0,I_p)$ is a fresh Gaussian vector, independent of the data. Note that the expectation is taken with respect to all sources of randomness (including the randomness of data and initialization). 
 
 We present two main theorems that summarize our results. The first theorem states that for sufficiently large $n$ and $p$, the mean squared error decreases during the initial iterates of gradient descent.
 This phenomenon justifies the name ``early learning": during early stages of training, the predictions obtained by the iterates of gradient descent improve.
 
 \begin{theorem}\label{thmA1} For any $n>N^*$, the function $k\mapsto\text{MSE}_{\eta,n}(k)$ decreases for $k\in [0,\frac{T^*}{\eta})$, where the constants $N^*=N^*(\eta,\gamma,\gamma_0)$ depends on $\eta,\gamma,\gamma_0$ and $T^*=T^*(\kappa, \gamma,\gamma_0,\delta)>0$  depends on $\kappa, \gamma,\gamma_0$, and $\delta$.
 
 \end{theorem}
 The second theorem, however, states that when the dimension is sufficiently large (with respect to the number of samples), the predictor eventually overfits, and converges to a predictor that outputs only the probabilities $0$ and $1$.
 
 \begin{theorem}\label{thmA2} Let $p(\vz)=\underset{k\to\infty}{\lim}\sigma(\langle\hat{\vtheta}^{k}_{\eta,n},\vz \rangle)$. Then there exists a $\kappa^*=\kappa^*(\gamma)$, such that when $n\to+\infty$ and $\frac{p}{n}\to\kappa>\kappa^*$, we have
 \begin{align*}
     \mathbb{P}(\{p(z)\in\{0,1\} \text{ for a.e. }\vz\in\R^p\})\to1.
 \end{align*}
 \end{theorem}
 
 We begin with the proof of Theorem \ref{thmA2}.
 \begin{proof}[Proof of Theorem \ref{thmA2}] Define the event
 \begin{align*}
     S_n=\{\exists \vw\in\R^p\text{ such that for }\forall i\leq n: y_i\langle\vw,\vx_i\rangle>0 \}.
 \end{align*}
  For some large $C_G>0$ to be determined later, define the event
  \begin{align*}
      G_n=\{\lrv{X}\leq C_G(\sqrt{n}+\sqrt{p})\},
  \end{align*}
 where $X=(\vx_1,...,\vx_n)\in\R^{p\times n}$ is the covariate matrix and $\lrv{X}$ denotes the maximum singular value of $X$. Pick $\delta<\frac{2}{C_G^2(\sqrt{\kappa}+1)^2}$.
 
 We claim that $S_n\cap G_n\subset E_n=\{p(\vz)\in\{0,1\} \text{ for a.e. }\vz\in\R^p\}$. Indeed, assume that $S_n$ and $G_n$ happen. Then the step size $\eta$ satisfies
 \begin{align*}
     \eta<\delta<\frac{2}{C_G^2(\sqrt{\kappa}+1)^2}<\frac{2n}{\lrv{X}^2}.
 \end{align*}
 On the event $S_n$, the data points are separable, and \citet[Theorem 3]{soudry2018implicit} implies that as $k \to \infty$,
 \begin{gather*}
     \lrv{\hat{\vtheta}^{k}_{\eta,n}}\to+\infty,\\
     \frac{\hat{\vtheta}^{k}_{\eta,n}}{\lrv{\hat{\vtheta}^{k}_{\eta,n}}}\to \hat{\vtheta}^{MM},
 \end{gather*}
 where $\hat{\vtheta}^{MM}\in\R^p$ denotes the max-margin separator. Thus for any $\vz\in\R^p$ with $\langle\vz,\hat{\vtheta}^{MM}\rangle>0$  (resp. $<0$), we have $\langle z, \hat{\vtheta}^{k}_{\eta,n}\rangle\to+\infty$ (resp. $-\infty$), and thus $p(\vz)=\underset{k\to\infty}{\lim}\sigma(\langle\hat{\vtheta}^{k}_{\eta,n},\vz \rangle)=1$ (resp. $=0$). Since $\{\vz\in\R^p: \langle\vz,\hat{\vtheta}^{MM}\rangle=0\}$ has Lebesgue measure $0$, we have shown that $S_n\cap G_n\subset E_n$.
 
 It remains to show that $S_n$ and $G_n$ hold with probability tending to one.
 \citet[Theorem 1]{candes2018phase} show that there exists a finite threshold $\kappa^*$, such that when $\frac{p}{n}\to\kappa>\kappa^*$, it holds that $\mathbb{P}(S_n)\to 1$. Moreover, by standard results from random matrix theory [see e.g. \cite{roman2018high}, Theorem 4.4.5], there exists some constant $C_G$ such that $\mathbb{P}(G_n)\geq 1-2e^{-n}\to 1$. Therefore, we can conclude that
 \begin{align*}
     \mathbb{P}(E_n)\geq\mathbb{P}(S_n\cap G_n)\to 1.
 \end{align*}
 \end{proof}
 
 In the remainder of the appendix, we prove Theorem \ref{thmA1}. Consider the following gradient flow, which is the analogue of gradient descent in the $\eta \to 0$ limit.
\begin{align}
    \frac{d\hat{\vtheta}^t_n}{dt}=-\nabla L_n(\hat{\vtheta}^t_n),\quad \forall t\geq 0 \label{gradient_flow}.
\end{align}
Set $\hat{\vtheta}^0_n=\hat{\vtheta}^{0}_{\eta,n}$ and define the corresponding mean square error 
\begin{equation}
\text{MSE}_{n}(t):=\E[(\sigma(\langle \hat{\vtheta}^t_n,\vz\rangle)-\sigma(\langle\vtheta^*_p,\vz\rangle))^2].
\end{equation}

We will first show that the mean squared error decreases along the trajectory of gradient flow, and then establish that this result extends to the iterates of gradient descent.
We denote $\text{MSE}'_n(t)$ as the derivative of $\text{MSE}_n$ at time $t\geq 0$. 
\begin{lemma} \label{lemA3}

(a) $\underset{n\to\infty}{\lim}\text{MSE}'_n(0)<0$.\\
(b) There exists some positive constants $N_1(\gamma,\gamma_0)$, $T_1(\gamma,\gamma_0)$ and $C_1(\gamma,\gamma_0)$, such that for any $n>N_1$ and $t<T_1$, it holds that $$\text{MSE}'_n(t)\leq-C_1.$$
\end{lemma}
\begin{proof} (a) For convenience, we write equation (\ref{gradient_flow}) in matrix form
\begin{align*}
    \frac{d\hat{\vtheta}^t_n}{dt}=\frac{1}{n}X(\vy'-\sigma(X^T\hat{\vtheta}^t_n)),
\end{align*}
where $X=(\vx_1,...,\vx_n)\in\R^{p\times n}$, $\vy'=(\frac{y_1+1}{2},...,\frac{y_n+1}{2})^T\in\R^n$ and $\sigma(X^T\hat{\vtheta}^t_n)=(\sigma(\langle \vx_1,\hat{\vtheta}^t_n\rangle),...,\sigma(\langle \vx_n,\hat{\vtheta}^t_n\rangle))^T\in\R^n$.
We first calculate the derivative of the mean square prediction error:
\begin{align*}
   \text{MSE}'_n(t)&=\E\left[\left(\sigma(\langle\hat{\vtheta}^t_n,\vz\rangle)-\sigma(\langle\vtheta^*_p,\vz\rangle)\right)\sigma'(\langle\hat{\vtheta}^t_n,\vz\rangle)\langle \frac{d\hat{\vtheta}^t_n}{dt},\vz\rangle\right]\\
    &=\E\left[\left(\sigma(\langle\hat{\vtheta}^t_n,\vz\rangle)-\sigma(\langle\vtheta^*_p,\vz\rangle)\right)\sigma'(\langle\hat{\vtheta}^t_n,\vz\rangle)\frac{1}{n}(\vy'-\sigma(X^T\hat{\vtheta}^t_n))^TX^Tz\right]\\
    &=\E\left[\left(\sigma(\langle\hat{\vtheta}^t_n,\vz\rangle)-\sigma(\langle\vtheta^*_p,\vz\rangle)\right)\sigma'(\langle\hat{\vtheta}^t_n,\vz\rangle)\frac{1}{n}(\sigma(X^T\vtheta^*_p)-\sigma(X^T\hat{\vtheta}^t_n))^TX^Tz\right],
\end{align*}
where in the first line we used dominated convergence theorem, and in the last line we used $\E[\vy'|X]=\sigma(X^T\vtheta^*_p)$. 

Thus,
\begin{align*}
    \text{MSE}'_n(0)&=\frac{1}{n}\sum_{i=1}^n\E\left[\left(\sigma(\langle\hat{\vtheta}^0_n,\vz\rangle)-\sigma(\langle\vtheta^*_p,\vz\rangle)\right)\sigma'(\langle\hat{\vtheta}^0_n,\vz\rangle)\left(\sigma(\langle \vtheta^*_p,\vx_i\rangle)-\sigma(\langle \hat{\vtheta}^0_n,\vx_i\rangle)\right)\langle\vx_i,\vz\rangle\right]\\
    &=\E\left[\left(\sigma(\langle\hat{\vtheta}^0_n,\vz\rangle)-\sigma(\langle\vtheta^*_p,\vz\rangle)\right)\sigma'(\langle\hat{\vtheta}^0_n,\vz\rangle)\left(\sigma(\langle \vtheta^*_p,\vx_1\rangle)-\sigma(\langle \hat{\vtheta}^0_n,\vx_1\rangle)\right)\langle\vx_1,\vz\rangle\right],
\end{align*}
where $\hat{\vtheta}^0_n$, $\vz$, and $\vx_1$ are mutually independent. Let $\mu(b):=\E[\sigma'(Z_b)]$, where the scalar $Z_b\sim\mathcal{N}(0,b^2)$. Conditioned on $(\vz,\hat{\vtheta}^0_n)$, the pair $(\langle\vx_1,\vz\rangle,\langle \hat{\vtheta}^0_n,\vx_1\rangle)$ is multivariate Gaussian with covariance $\langle \hat{\vtheta}^0_n,\vz\rangle$. By the Gaussian integration by parts formula,
\begin{align*}
    \E[\sigma(\langle \hat{\vtheta}^0_n,\vx_1\rangle)\langle\vx_1,\vz\rangle|\vz,\hat{\vtheta}^0_n]=\E[\sigma'(\langle \hat{\vtheta}^0_n,\vx_1\rangle)|\vz,\hat{\vtheta}^0_n]\langle\hat{\vtheta}^0_n,\vz\rangle=\mu(\gamma_0)\langle\hat{\vtheta}^0_n,\vz\rangle.
\end{align*}
Similarly,
\begin{align*}
    \E[\sigma(\langle \vtheta^*_p,\vx_1\rangle)\langle\vx_1,\vz\rangle|\vz,\hat{\vtheta}^0_n]=\mu(\gamma)\langle\vtheta^*_p,\vz\rangle.
\end{align*}
So
\begin{align*}
    \text{MSE}'_n(0)=\E\left[\left(\sigma(\langle\hat{\vtheta}^0_n,\vz\rangle)-\sigma(\langle\vtheta^*_p,\vz\rangle)\right)\sigma'(\langle\hat{\vtheta}^0_n,\vz\rangle)\left(\mu(\gamma)\langle\vtheta^*_p,\vz\rangle-\mu(\gamma_0)\langle\hat{\vtheta}^0_n,\vz\rangle\right)\right],
\end{align*}
Since $\hat{\vtheta}^0_n\sim\text{Unif}(\mathbb{S}^{p-1}(\gamma_0))$ , we have $\langle\hat{\vtheta}^0_n,\vtheta^*_p\rangle\to 0$ in probability as $p\to\infty$. Thus, the pair $(\langle\hat{\vtheta}^0_n,\vz\rangle,\langle\vtheta^*_p,\vz\rangle)\overset{d}{\to}(Z_0,Z)$, where $Z_0$ and $Z$ are independent centered Gaussian random variables with variance $\gamma_0$ and $\gamma$, respectively. 

We hence compute
\begin{align*}
    \underset{p\to\infty}{\lim}\text{MSE}'_n(0)&=\E[(\sigma(Z_0)-\sigma(Z))\sigma'(Z_0)(\mu(\gamma)Z-\mu(\gamma_0)Z_0)]\\
    &=-\gamma^2\mu(\gamma)^2\mu(\gamma_0)-\mu(\gamma_0)\E[(\sigma(Z_0)-\frac{1}{2})\sigma'(Z_0)Z_0],
\end{align*}
where in the last inequality we used Gaussian integration by parts and $\E[\sigma(Z)]=\frac{1}{2}$. Since $\mu(b)>0$ for all $b\geq 0$ and $(\sigma(x)-\frac{1}{2})\sigma'(x)x\geq 0$ for all $x\in\R$, we conclude that $\underset{n\to\infty}{\lim}\text{MSE}'_n(0)<0$.

(b) We first show that the second derivative of $\text{MSE}_n$ is uniformly bounded. Write $f(x)=(\sigma(x)-\sigma(\langle\vtheta^*_p,\vz\rangle))^2$. Since $\sigma,\sigma'$ and $\sigma''$ are all bounded, we have
\begin{align*}
    |\text{MSE}''_n(t)|&=\left|\E\left[f''(\langle\hat{\vtheta}^t_n,\vz\rangle)\langle \frac{d\hat{\vtheta}^t_n}{dt},\vz\rangle^2\right]+\E\left[f'(\langle\hat{\vtheta}^t_n,\vz\rangle)\langle \frac{d^2\hat{\vtheta}^t_n}{dt^2},\vz\rangle\right]\right|\\
    &\leq C_1\left(\E\left[\langle \frac{d\hat{\vtheta}^t_n}{dt},\vz\rangle^2\right]+\E\left[\left|\langle \frac{d^2\hat{\vtheta}^t_n}{dt^2},\vz\rangle\right|\right]\right)\\
    &\leq C_2\left(\E\left[\left\lVert\frac{d\hat{\vtheta}^t_n}{dt}\right\rVert_2^2\right]+\E\left[\left\lVert \frac{d^2\hat{\vtheta}^t_n}{dt^2}\right\rVert_2\right]\right),
\end{align*}
where in the first equality, we use dominated convergence to interchange differentiation and expectation, and in the second inequality we used the fact that $\vz\sim\mathcal{N}(0,I_p)$ is independent of the data. 

By a simple calculation,
\begin{gather*}
    \left\lVert\frac{d\hat{\vtheta}^t_n}{dt}\right\rVert_2\leq \frac{1}{n}\left\lVert X\right\rVert_2\left\lVert \vy'-\sigma(X^T\hat{\vtheta}^t_n)\right\rVert_2\leq \frac{2\sqrt{p}\left\lVert X\right\rVert_2}{n},\\
    \left\lVert\frac{d^2\hat{\vtheta}^t_n}{dt^2}\right\rVert_2=\left\lVert\frac{1}{n^2}XD_tX^TX(\vy'-\sigma(X^T\hat{\vtheta}^t_n))\right\rVert_2\leq\frac{2\sqrt{p}\left\lVert X\right\rVert_2^3}{n^2},
\end{gather*} 
where $D_t=\text{Diag}(\sigma'(\langle\hat{\vtheta}^t_n,\vx_i\rangle))_{i=1}^n$ is a diagonal matrix with $\lrv{D_t}\leq 1$. By standard random matrix results [see e.g. \cite{roman2018high}, Theorem 4.4.5], 
\begin{align*}
    \E[\left\lVert X\right\rVert_2^2]\leq C(\sqrt{p}+\sqrt{n})^2,\quad \E[\left\lVert X\right\rVert_2^3]\leq C(\sqrt{p}+\sqrt{n})^3.
\end{align*}
Whence, for any $t\geq 0$ and $n>N_0$,
\begin{align*}
    |\text{MSE}''_n(t)|\leq C_3\left(\frac{p(\sqrt{p}+\sqrt{n})^2}{n^2}+\frac{\sqrt{p}(\sqrt{p}+\sqrt{n})^3}{n^2}\right):=L,
\end{align*}
where $L=C_4(k^2+1)$ is a constant that only depends on $\kappa$.

Let $\tau=\tau(\gamma,\gamma_0)=\underset{n\to\infty}{\lim}\text{MSE}'_n(0)$, and let $N_1=N_1(\gamma,\gamma_0)>N_0$ satisfy $\text{MSE}'_n(0)\leq \frac{\tau}{2}$ for any $n>N_1$.
By Claim (a), $\tau < 0$.
Then for any $t\leq T_1(\gamma,\gamma_0)=\frac{-\tau(\gamma,\gamma_0)}{4L}$, 
\begin{align*}
    \text{MSE}'_n(t)=\text{MSE}'_n(0)+(\text{MSE}'_n(t)-\text{MSE}'_n(0))\leq \frac{\tau}{2}+Lt\leq \frac{\tau}{4}.
\end{align*}
\end{proof}
Lemma A.1 essentially shows that within a small constant time window, $\text{MSE}'_n$ is negative, and thus the mean square error of gradient flow decreases. The next lemma shows that the mean square error of gradient flow is a good approximation of the mean square error of gradient descent. Before stating the next lemma, we extend gradient descent (\ref{gradient_descent}) to a piecewise linear function, i.e. 
\begin{align}
    \frac{d\Tilde{\vtheta}^t_{\eta,n}}{dt}=-\nabla L_n(\Tilde{\vtheta}^{\floor t}_{\eta,n}),\quad \forall t\geq 0,
\end{align}
where $\floor t:=\max\{k\eta:k\eta\leq t,k\in\mathbb{N}\}$. In particular, $\Tilde{\vtheta}^{k\eta}_{\eta,n}=\hat{\vtheta}^{k}_{\eta,n}$.

\begin{lemma}\label{lemA4}
(a) For any $t>0$ and $n>N_0$, 
$$\underset{s\leq t}{\sup}\lVert\Tilde{\vtheta}^{s}_{\eta,n}-\hat{\vtheta}^s_n\rVert\leq C_2(\kappa,t,\delta,\gamma_0)\eta t.$$
with probability $1-2e^{-n}$, where $C_2(\kappa,t,\delta,\gamma_0)>0$ is a constant that depends on $\kappa,t,\delta$ and $\gamma_0$.\\
(b) For any $t>0$ and $n>N_0$, 
$$\underset{k\leq t/\eta}{\sup}\left|\text{MSE}_n(k\eta)-\text{MSE}_{\eta,n}(k)\right|\leq 4C_2(\kappa,t,\delta,\gamma_0)\eta t+8e^{-n}.$$

\end{lemma}

\begin{proof} (a)  We first bound $\underset{s\leq t}{\sup}\lVert\Tilde{\vtheta}^{s}_{\eta,n}-\hat{\vtheta}^s_n\rVert$. 

Compute
\begin{align*}
    \frac{d}{dt}\lrv{\hat{\vtheta}^t_n}&\leq\lrv{\frac{d\hat{\vtheta}^t_n}{dt}}\\
    &=\lrv{\frac{1}{n}X(\vy'-\sigma(\vzero))+\frac{1}{n}X(\sigma(\vzero)-\sigma(X^T\hat{\vtheta}^t_n))}\\
    &\leq\frac{2\sqrt{p}\lrv{X}}{n}+\frac{\lrv{X}^2\lVert\hat{\vtheta}_n^t\rVert_2}{n},
\end{align*}
where we used $\lrv{\sigma(\vx)-\sigma(\vy)}\leq\lrv{\vx-\vy}$ for any $\vx,\vy\in\R^p$. By Gr\"{o}nwall's inequality,
\begin{align}
    \lrv{\hat{\vtheta}_n^t}\leq e^{\frac{\lrv{X}^2}{n}t}\left(\lrv{\hat{\vtheta}_n^0}+\frac{2\sqrt{p}\lrv{X}}{n}t\right).\label{Gronwall1}
\end{align}
Next, compute
\begin{align*}
    \frac{d}{dt}\lrv{\Tilde{\vtheta}^{s}_{\eta,n}-\hat{\vtheta}^s_n}&\leq\lrv{\frac{d}{dt}\left(\Tilde{\vtheta}^{s}_{\eta,n}-\hat{\vtheta}^s_n\right)}\\
    &=\lrv{\frac{1}{n}X(\sigma(X^T\Tilde{\vtheta}^{\floor s}_{\eta,n})-\sigma(X^T\hat{\vtheta}^s_n)}\\
    &\leq\frac{\lrv{X}^2}{n}\lrv{\Tilde{\vtheta}^{\floor s}_{\eta,n}-\hat{\vtheta}^s_n}\\
    &\leq\frac{\lrv{X}^2}{n}\lrv{\Tilde{\vtheta}^{s}_{\eta,n}-\hat{\vtheta}^s_n}+\frac{\lrv{X}^2}{n}\lrv{\Tilde{\vtheta}^{\floor s}_{\eta,n}-\Tilde{\vtheta}^{ s}_{\eta,n}}\\
    &\leq\frac{\lrv{X}^2}{n}\lrv{\Tilde{\vtheta}^{s}_{\eta,n}-\hat{\vtheta}^s_n}+\frac{\lrv{X}^2}{n}\lrv{\Tilde{\vtheta}^{\floor s}_{\eta,n}-\Tilde{\vtheta}^{ s}_{\eta,n}}.
\end{align*}
Since $\Tilde{\vtheta}^{s}_{\eta,n}-\Tilde{\vtheta}^{\floor s}_{\eta,n}=(s-\floor s)\frac{1}{n}X(\vy'-\sigma(X^T\Tilde{\vtheta}^{\floor s}_{\eta,n}))$, we can bound
\begin{align*}
    \lrv{\Tilde{\vtheta}^{s}_{\eta,n}-\Tilde{\vtheta}^{\floor s}_{\eta,n}}&\leq\frac{2\eta\sqrt{p}\lrv{X}}{n}+\frac{\eta\lrv{X}^2\lVert\Tilde{\vtheta}^{\floor s}_{\eta,n}\rVert_2}{n}\\
    &\leq \frac{2\eta\sqrt{p}\lrv{X}}{n}+\frac{\eta\lrv{X}^2\lVert\hat{\vtheta}^{\floor s}_n\rVert_2}{n}+\frac{\eta\lrv{X}^2}{n}\lrv{\hat{\vtheta}^{\floor s}_n-\Tilde{\vtheta}^{\floor s}_{\eta,n}}.
\end{align*}
Combining the above inequalities, we get
\begin{align*}
    \frac{d}{dt}\underset{s\leq t}{\sup}\lrv{\Tilde{\vtheta}^{s}_{\eta,n}-\hat{\vtheta}^s_n}&\leq \frac{d}{dt}\lrv{\Tilde{\vtheta}^{t}_{\eta,n}-\hat{\vtheta}^t_n}\\
    &\leq \frac{\lrv{X}^2}{n}\lrv{\Tilde{\vtheta}^{t}_{\eta,n}-\hat{\vtheta}^t_n}+\frac{\lrv{X}^2}{n}\left(\frac{2\eta\sqrt{p}\lrv{X}}{n}+\frac{\eta\lrv{X}^2\lVert\hat{\vtheta}^{\floor t}_n\rVert_2}{n}+\frac{\eta\lrv{X}^2}{n}\lrv{\hat{\vtheta}^{\floor t}_n-\Tilde{\vtheta}^{\floor t}_{\eta,n}}\right)\\
    &\leq \left(\frac{\lrv{X}^2}{n}+\frac{\delta\lrv{X}^4}{n^2}\right)\underset{s\leq t}{\sup}\lrv{\Tilde{\vtheta}^{s}_{\eta,n}-\hat{\vtheta}^s_n}+\eta\left(\frac{2\sqrt{p}\lrv{X}^3}{n^2}+\frac{\lrv{X}^4\lVert\hat{\vtheta}_n^{\floor t}\rVert}{n^2}\right).
\end{align*}
Together with inequality (\ref{Gronwall1}), by Gr\"{o}nwall's inequality we eventually have
\begin{align}
    \underset{s\leq t}{\sup}\lrv{\Tilde{\vtheta}^{s}_{\eta,n}-\hat{\vtheta}^s_n}\leq \eta t\left(\frac{2\sqrt{p}\lrv{X}^3}{n^2}+\frac{\gamma_0\lrv{X}^4e^{\frac{\lrv{X}^2}{n}t}}{n^2}+\frac{2t\sqrt{p}\lrv{X}^5e^{\frac{\lrv{X}^2}{n}t}}{n^3}\right)e^{(\frac{\lrv{X}^2}{n}+\frac{\delta\lrv{X}^4}{n^2})t}.\label{sup_bound}
\end{align}
Again, by a standard random matrix result [see e.g. \cite{roman2018high}, Theorem 4.4.5], it holds that $\lrv{X}\leq C(\sqrt{p}+\sqrt{n})$, with probability $1-2e^{-n}$. Plug this bound into inequality (\ref{sup_bound}), we get for any $n>N_0$,
\begin{align}
    \underset{s\leq t}{\sup}\lrv{\Tilde{\vtheta}^{s}_{\eta,n}-\hat{\vtheta}^s_n}\leq\eta tC_2(\kappa,t,\delta,\gamma_0),
\end{align}
with probability $1-2e^{-n}$, where $C_2(\kappa,t,\delta,\gamma_0)=C(t+\gamma_0+1)(\kappa^6+1)e^{C(\delta+1)(\kappa^2+1)t}$, for some $C>0$.

(b) Define the event $A_n=\left\{ \underset{s\leq t}{\sup}\lrv{\Tilde{\vtheta}^{s}_{\eta,n}-\hat{\vtheta}^s_n}\leq\eta tC_2(\kappa,t,\delta,\gamma_0)\right\}$. Since $|\sigma|,|\sigma'|<1$, we have for any $n>N_0$,
\begin{align*}
    \underset{k\leq t/\eta}{\sup}\left|\text{MSE}_n(k\eta)-\text{MSE}_{\eta,n}(k)\right|&\leq \underset{k\leq t/\eta}{\sup}4\E[|\sigma(\langle\hat{\vtheta}^{k\eta}_n,\vz\rangle)-\sigma(\langle\hat{\vtheta}^{k}_{\eta,n},\vz\rangle)|]\\
    &=\underset{k\leq t/\eta}{\sup}4\E[|\sigma(\langle\hat{\vtheta}^{k\eta}_n,\vz\rangle)-\sigma(\langle\hat{\vtheta}^{k}_{\eta,n},\vz\rangle)|\mathbbm{1}_{A_n}]+8e^{-n}\\
    &\leq\underset{k\leq t/\eta}{\sup}4\E\left[\lrv{\hat{\vtheta}^{k\eta}_n-\hat{\vtheta}^{k}_{\eta,n}}\mathbbm{1}_{A_n}\right]+8e^{-n}\\
    &\leq 4\E\left[\underset{k\leq t/\eta}{\sup}\lrv{\hat{\vtheta}^{k\eta}_n-\hat{\vtheta}^{k}_{\eta,n}}\mathbbm{1}_{A_n}\right]+8e^{-n}\\
    & = 4 \E\left[\underset{k\leq t/\eta}{\sup}\lrv{\hat{\vtheta}^{k\eta}_n-\Tilde{\vtheta}^{k\eta}_{\eta,n}}\mathbbm{1}_{A_n}\right]+8e^{-n}\\
    &\leq 4C_2(\kappa,t,\delta,\gamma_0)t\eta+8e^{-n}
\end{align*}

\end{proof}

We are now ready to prove Theorem \ref{thmA1}.
\begin{proof}[Proof of Theorem \ref{thmA1}] Pick a small $T_2(\kappa,\delta,\gamma_0)>0$ that satisfy $C_2(\kappa,T_2,\delta,\gamma_0)T_2\leq \frac{C_1(\gamma,\gamma_0)}{32}$. This is possible since $C_2(\kappa,t,\delta,\gamma_0)$ is bounded as $t\to 0$.
Let $N_1$ and $T_1$ be as in Lemma \ref{lemA3}, and let $T^*(\gamma,\gamma_0,\kappa,\delta)=\min\{T_1(\gamma,\gamma_0),T_2(\kappa,\delta,\gamma_0)\}$. For any $k<\frac{T^*}{\eta}-1$,
\begin{align*}
    \text{MSE}_{\eta,n}(k+1)-\text{MSE}_{\eta,n}(k)&=(\text{MSE}_{n}((k+1)\eta)-\text{MSE}_{n}(k\eta))+(\text{MSE}_{\eta,n}(k+1)-\text{MSE}_{n}((k+1)\eta))\\
    &+(\text{MSE}_{\eta,n}(k)-\text{MSE}_{n}((k)\eta))\\
    &\leq (\text{MSE}_{n}((k+1)\eta)-\text{MSE}_{n}(k\eta))+2\underset{m\leq T^*/\eta}{\sup}\left|\text{MSE}_n(m\eta)-\text{MSE}_{\eta,n}(m)\right|
\end{align*}
By Lemma \ref{lemA3}.(b) and the mean value theorem, for any $ n>N_1$,
\begin{align*}
    \text{MSE}_{n}((k+1)\eta)-\text{MSE}_{n}(k\eta)\leq-C_1(\gamma,\gamma_0)\eta.
\end{align*}
By Lemma \ref{lemA4}.(b), it holds that for any $ n>N_0$,
\begin{align*}
    \underset{m\leq T^*/\eta}{\sup}\left|\text{MSE}_n(m\eta)-\text{MSE}_{\eta,n}(m)\right|&\leq 4C_2(\kappa,T^*,\delta,\gamma_0)T^*\eta+8e^{-n}\\
    &\leq \frac{C_1(\gamma,\gamma_0)\eta}{8}+8e^{-n}.
\end{align*}
Let $N_2(\eta,\gamma,\gamma_0)$ satisfy $8e^{-N_2}<\frac{C_1(\gamma,\gamma_0)\eta}{8}$, and let $N^*(\eta,\gamma,\gamma_0)=\max\{N_1(\gamma,\gamma_0),N_2(\eta,\gamma,\gamma_0),\}$. We therefore have for any $n>N^*$ and $k<\frac{T^*}{\eta}-1$,
\begin{align*}
    \text{MSE}_{\eta,n}(k+1)-\text{MSE}_{\eta,n}(k)&\leq -C_1(\gamma,\gamma_0)\eta+\frac{C_1(\gamma,\gamma_0)\eta}{4}+16e^{-N^*}\\
    &=-\frac{C_1(\gamma,\gamma_0)\eta}{2}<0.
\end{align*}

\end{proof}
\newpage
\section{Early Learning and Memorization in a Linear Model}
\label{app:linear_model}
We provide a numerical example that illustrates the theoretical results of Section~\ref{sec:early_learning_model}, and demonstrates the similarities between the behavior of linear and deep-learning models in high dimensions. 

We train a logistic regression model in an overparametrized regime ($\frac{p}{n}=1$). To generate the features, we draw $\{\vx_i\}_{i=1}^{500},\hspace{0.2em}\vx_i\in\mathbb{R}^{500}$ i.i.d. Gaussian random variables, $\vx_i\sim\mathcal{N}\left(0 ,I_{500}\right)$. The ground-truth unobserved probability labels $p_i$ are generated according to \eqref{true_prob}. i.e. $p_i =\sigmoid \left(\langle{\vtheta}^\ast, \vx_i\rangle\right).$ The true parameter $\vtheta^* = (1,0,\dots, 0)$ is fixed to equal the first standard basis vector, and $\sigma(x)=(1+e^{x})^{-1}$ is the sigmoid function. 

The 0-1 labels $\left\{y_i\right\}_{i=1}^{500}$ are drawn from Bernoulli distribution with probability $p_i$,  $y_i\sim\text{Bern}(p_i)$. 

We use stochastic gradient descent (to fit a logistic regression model with parameter $\vtheta \in \mathbb{R}^{500}$ on the simulated data $\{(\vx_i, y_i)\}_{i=1}^{500}$, using a cross-entropy loss function. We compare the model trained on two datasets: $(a)$ finite $(\vx,y)$ pairs, and $(b)$ a large amount of data set, generated by repeatedly resampling new outcomes $y_i$ from the ground-truth probabilities $p_i$ at every iteration.

Figure~\ref{fig:scatter_plot_analytic} shows that the linear model trained with cross entropy begins by improving the estimate of the true probabilities (at the early-learning stage), but eventually memorizes the 0-1 labels (the overfitting stage). Figure~\ref{fig:leraning_curve_analytic} illustrates that the $\text{MSE}_p$ of cross entropy training increases when the model starts memorizing the 0-1 labels, as predicted by the results of Appendix~\ref{app:analytical_model}.
\begin{figure}
    \centering
    \includegraphics[width=0.9\textwidth]{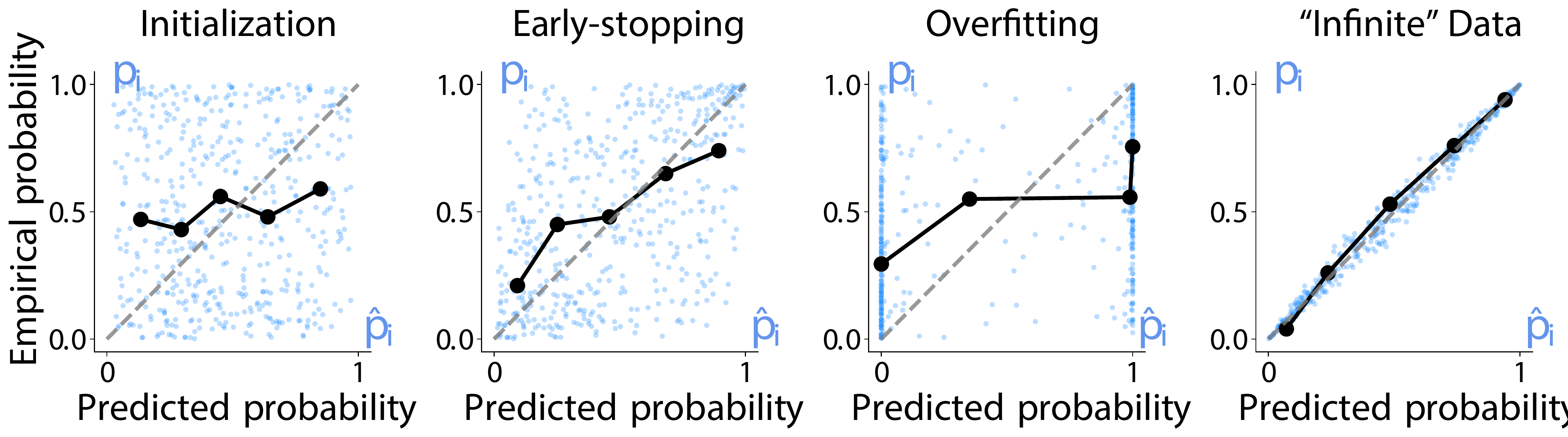}
    \caption{\textbf{Reliability curves for a linear model (Appendix~\ref{app:linear_model})}. The horizontal and vertical coordinate of each blue point represent the predicted ($\hat{p}_i$) and true ($p_i$) probabilities of a test example respectively. We also show a reliability diagram of binned mean predicted and empirical probabilities in black (see Section~\ref{sec:metrics}). The dashed diagonal line indicates perfect calibration. The model is initialized randomly (first column) and initially improves the estimation, with the reliability curve trending towards the diagonal line (second column). However, because the dataset is finite, it eventually overfits, and the predicted probabilities collapse to 0 or 1 (third column). When labels are resampled to generate large amount of data, the estimates converge to the ground-truth probability labels (right column).}
    \label{fig:scatter_plot_analytic}
\end{figure}
\begin{figure}
    \centering
    \includegraphics[width=0.6\textwidth]{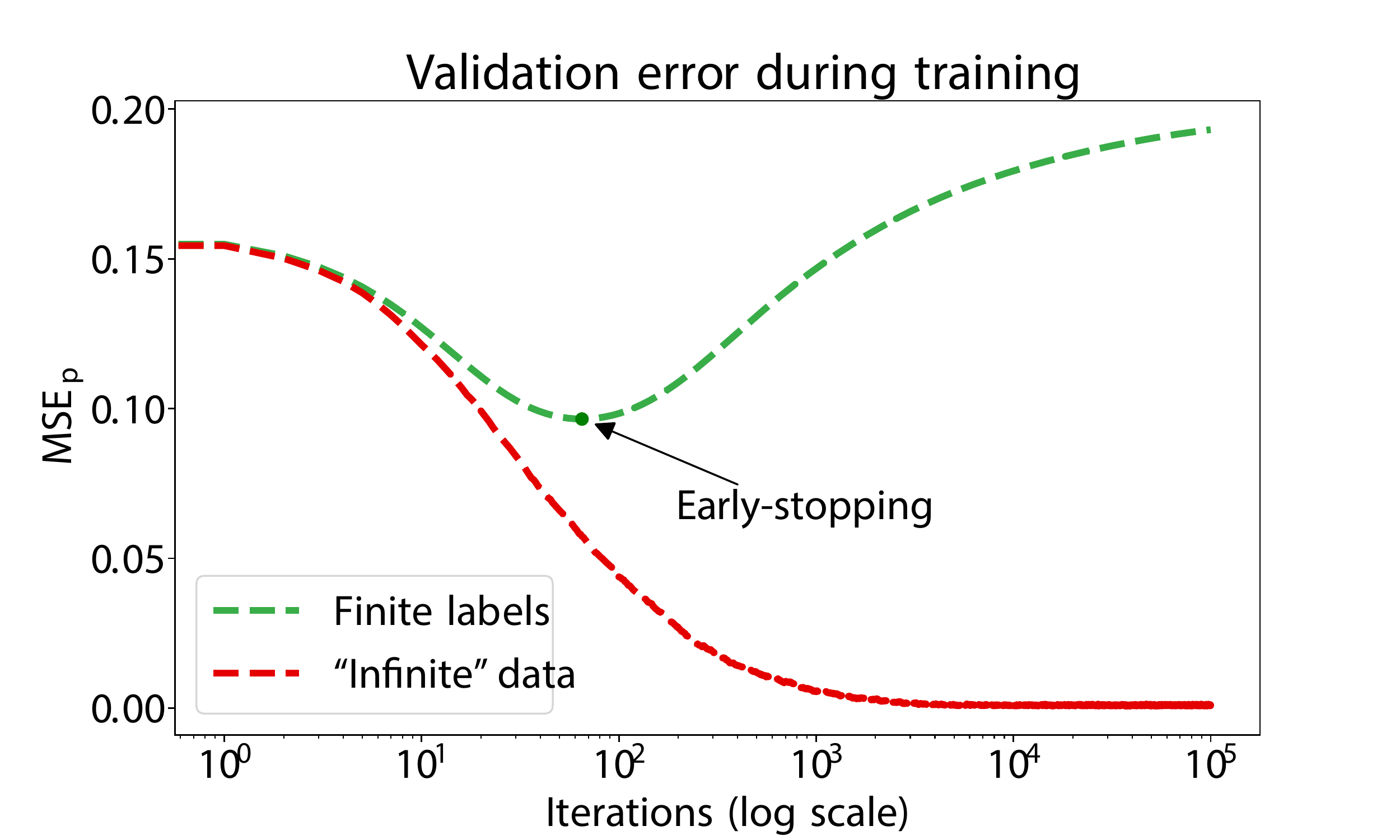}
    \caption{\textbf{Validation loss for training a linear model (Appendix~\ref{app:linear_model})}. Training with finite data (green line) decreases $\text{MSE}_p$ initially but eventually overfits to the 0-1 labels. Training with resampled labels (red line) avoids overfitting and results in accurate estimation of the true probabilities.}
    \label{fig:leraning_curve_analytic}
\end{figure}

\section{Additional Results}
\label{app:results}
We present here supplementary results to the ones presented in Section~\ref{sec:results}.
\subsection{Face-based Risk Prediction}
\label{app:face_results}

Full evaluation with confidence intervals derived using 1000 bootstraps for the five simulated scenarios are examined: \textit{Linear} (Table.\ref{tab:uniform_ci_face}); \textit{Sigmoid} (Table.\ref{tab:sigmoid_ci_face}); \textit{Centered} (Table.\ref{tab:middle_ci_face}); \textit{Skewed} (Table.\ref{tab:scaled_ci_face}); \textit{Discrete} (Table.\ref{tab:stepwise_ci_face}). Note that all numbers are upscaled by $10^{-2}$ in the tables.

\begin{table}[h]
    \centering
    \begin{tabular}{l@{\hspace{0.25cm}}c@{\hspace{0.25cm}}c@{\hspace{0.25cm}}c@{\hspace{0.25cm}}c@{\hspace{0.25cm}}c@{\hspace{0.20cm}}c}
\toprule
\textit{Linear} &                       ECE &                       MCE &                        KS &                     Brier &                 MSE$_p$ &                     KL$_p$ \\
\midrule
CE + label resampling        &   4.14$\pm$0.81 &  12.07$\pm$3.29 &   2.24$\pm$0.88 &  18.97$\pm$0.33 &  1.14$\pm$0.04 &   2.82$\pm$0.11 \\
\midrule
CE early-stop     &  12.32$\pm$0.83 &  21.79$\pm$1.97 &  12.16$\pm$0.83 &  21.82$\pm$0.51 &  4.21$\pm$0.15 &  10.94$\pm$0.36 \\
Temperature        &    5.7$\pm$0.74 &  13.82$\pm$2.71 &   2.36$\pm$0.74 &  20.47$\pm$0.37 &  2.73$\pm$0.11 &   6.75$\pm$0.25 \\
Platt Scaling     &   4.29$\pm$0.77 &  10.94$\pm$2.55 &   1.3$\pm$0.45 &  20.18$\pm$0.36 &  2.48$\pm$0.09 &   6.07$\pm$0.22 \\
Dirichlet Cal.  &   7.38$\pm$1.12 &  22.58$\pm$7.24 &   3.78$\pm$0.46 &  21.32$\pm$0.33 &  3.56$\pm$0.13 &   9.08$\pm$0.29 \\
Focal Loss &   5.34$\pm$0.68 &  13.31$\pm$2.67 &   3.56$\pm$0.85 &  21.99$\pm$0.28 &  4.13$\pm$0.11 &  10.52$\pm$0.28 \\
Mix-n-match &   5.46$\pm$0.84 &   12.9$\pm$2.51 &   1.92$\pm$0.44 &  20.43$\pm$0.35 &   2.7$\pm$0.11 &   6.72$\pm$0.24 \\
Entropy Reg.    &    5.7$\pm$0.74 &  13.52$\pm$1.67 &    4.94$\pm$0.9 &   20.42$\pm$0.3 &  2.58$\pm$0.09 &   6.65$\pm$0.21 \\
MMCE Reg.  &   4.89$\pm$0.74 &  12.57$\pm$2.27 &   1.92$\pm$0.46 &  20.04$\pm$0.38 &  2.24$\pm$0.08 &    5.68$\pm$0.2 \\
Deep Ensemble  &  4.26$\pm$0.72 &  11.33$\pm$2.38 &   1.95$\pm$0.61 &  19.88$\pm$0.32 &   1.9$\pm$0.07 &   4.55$\pm$0.18 \\
\midrule
CaPE (bin)         &   4.58$\pm$0.75 &  11.85$\pm$2.49 &   1.71$\pm$0.51 &  19.68$\pm$0.36 &  1.78$\pm$0.07 &   4.35$\pm$0.16 \\
CaPE (kernel)       &   4.62$\pm$0.62 &  12.25$\pm$2.31 &   1.65$\pm$0.38 &  19.71$\pm$0.34 & 1.74$\pm$0.07 &    4.3$\pm$0.17 \\
\bottomrule
\end{tabular}
    \caption{Performance on Face-based Risk Prediction. \textit{Linear} scenario. All numbers are downscaled by $10^{-2}$.}
    \label{tab:uniform_ci_face}
\end{table}

\begin{table}[h]
    \centering
    \begin{tabular}{lc@{\hspace{0.25cm}}c@{\hspace{0.25cm}}c@{\hspace{0.25cm}}c@{\hspace{0.25cm}}c@{\hspace{0.25cm}}c}
\toprule

\textit{Sigmoid} &                     ECE &                       MCE &                        KS &                     Brier &                 MSE$_p$ &                     KL$_p$ \\
\midrule
CE + label resampling         &   6.4$\pm$0.71 &  20.63$\pm$3.44 &  2.74$\pm$0.45 &  16.28$\pm$0.44 &   5.34$\pm$0.2 &  14.82$\pm$0.51 \\
\midrule
CE early-stop        &  6.19$\pm$0.75 &   17.0$\pm$3.68 &   5.86$\pm$0.8 &  16.68$\pm$0.42 &  6.16$\pm$0.17 &  17.16$\pm$0.48 \\
Temperature        &  5.57$\pm$0.71 &  15.32$\pm$3.09 &  5.02$\pm$0.83 &  16.58$\pm$0.34 &  6.13$\pm$0.17 &  17.09$\pm$0.43 \\
Platt Scaling      &  3.45$\pm$0.68 &  10.32$\pm$2.79 &   1.3$\pm$0.43 &  16.33$\pm$0.34 &  5.78$\pm$0.19 &  16.15$\pm$0.47 \\
Dirichlet Cal. &  14.5$\pm$1.15 &  25.68$\pm$3.02 &  4.67$\pm$0.32 &  19.21$\pm$0.43 &  8.64$\pm$0.26 &  25.18$\pm$0.58 \\
Focal Loss      &   4.65$\pm$0.7 &  11.84$\pm$2.78 &  2.66$\pm$0.77 &  16.96$\pm$0.34 &  6.86$\pm$0.21 &   19.46$\pm$0.5 \\
Mix-n-match &  5.65$\pm$0.76 &  15.32$\pm$3.41 &  5.09$\pm$0.94 &   16.6$\pm$0.36 &  6.12$\pm$0.17 &  17.08$\pm$0.46 \\
Entropy Reg.    &  9.51$\pm$0.79 &  18.77$\pm$2.38 &  7.26$\pm$0.78 &  17.17$\pm$0.31 &  7.02$\pm$0.17 &  21.16$\pm$0.42 \\
MMCE Reg.    &  4.67$\pm$0.76 &  13.63$\pm$2.59 &   2.5$\pm$0.53 &   15.9$\pm$0.51 &  5.35$\pm$0.18 &  15.06$\pm$0.49 \\
Deep Ensemble  &   5.17$\pm$0.74 &  16.12$\pm$3.11 &  2.04$\pm$0.44 &  16.39$\pm$0.45 &  5.86$\pm$0.22 &   16.46$\pm$0.6 \\
\midrule
CaPE (bin)         &   3.78$\pm$0.6 &  11.96$\pm$2.59 &   2.22$\pm$0.7 &  15.84$\pm$0.43 &   5.17$\pm$0.2 &  14.27$\pm$0.49 \\
CaPE (kernel)          &   3.9$\pm$0.75 &  11.73$\pm$2.79 &  2.05$\pm$0.54 &  15.85$\pm$0.41 &   5.16$\pm$0.2 &  14.34$\pm$0.49 \\
\bottomrule
\end{tabular}
    \caption{Performance on Face-based Risk Prediction. \textit{Sigmoid} scenario. All numbers are downscaled by $10^{-2}$.}
    \label{tab:sigmoid_ci_face}
\end{table}

\begin{table}[]
    \centering
    \begin{tabular}{lc@{\hspace{0.3cm}}c@{\hspace{0.3cm}}c@{\hspace{0.3cm}}c@{\hspace{0.3cm}}c@{\hspace{0.3cm}}c}
\toprule
\textit{Centered}  &                ECE &                       MCE &                        KS &                     Brier &                 MSE$_p$ &                     KL$_p$  \\
\midrule
CE + label resampling        &  4.29$\pm$0.74 &  12.38$\pm$2.92 &   2.68$\pm$0.8 &  24.22$\pm$0.13 &   0.2$\pm$0.01 &  0.41$\pm$0.01 \\
\midrule
CE early-stop          &  5.76$\pm$0.84 &  15.32$\pm$3.07 &  4.19$\pm$1.02 &  24.68$\pm$0.08 &  0.48$\pm$0.01 &  0.98$\pm$0.03 \\
Temperature        &  6.09$\pm$0.82 &  15.83$\pm$2.91 &  4.74$\pm$0.96 &  24.74$\pm$0.06 &  0.48$\pm$0.01 &  0.98$\pm$0.03 \\
Platt Scaling      &  4.57$\pm$0.76 &   11.85$\pm$2.5 &  2.79$\pm$0.85 &  24.62$\pm$0.08 &  0.41$\pm$0.01 &  0.83$\pm$0.03 \\
Dirichlet Cal.  &  4.84$\pm$1.15 &  13.13$\pm$7.61 &  2.16$\pm$0.86 &    24.7$\pm$0.1 &  0.46$\pm$0.01 &  0.94$\pm$0.03 \\
Mix-n-match &  6.05$\pm$0.83 &  15.71$\pm$2.92 &  4.68$\pm$0.98 &  24.74$\pm$0.06 &  0.48$\pm$0.01 &  0.98$\pm$0.02 \\
Focal Loss     &  5.09$\pm$0.83 &   13.4$\pm$2.87 &  3.44$\pm$1.02 &   24.8$\pm$0.05 &  0.48$\pm$0.01 &  0.97$\pm$0.03 \\
Entropy Reg.     &  5.02$\pm$0.86 &  12.69$\pm$3.42 &  3.27$\pm$0.96 &  24.74$\pm$0.06 &  0.45$\pm$0.01 &  0.92$\pm$0.03 \\
MMCE Reg.        &  5.56$\pm$0.86 &  13.59$\pm$2.65 &  2.71$\pm$0.93 &   24.7$\pm$0.08 &  0.44$\pm$0.01 &   0.9$\pm$0.03 \\
Deep Ensemble  &  4.84$\pm$0.78 &  12.39$\pm$2.52 &  2.64$\pm$0.71 &  24.69$\pm$0.07 &  0.44$\pm$0.01 &  0.89$\pm$0.03 \\
\midrule
CaPE (bin)         &  4.73$\pm$0.82 &  11.81$\pm$2.54 &   2.07$\pm$0.6 &  24.56$\pm$0.11 &  0.38$\pm$0.01 &  0.78$\pm$0.03 \\
CaPE (kernel)          &  5.41$\pm$0.87 &   12.71$\pm$2.5 &  2.39$\pm$0.78 &  24.59$\pm$0.11 &   0.4$\pm$0.01 &  0.81$\pm$0.03 \\

\bottomrule
\end{tabular}
    \caption{Performance on Face-based Risk Prediction. \textit{Centered} scenario. All numbers are downscaled by $10^{-2}$.}
    \label{tab:middle_ci_face}
\end{table}

\begin{table}[]
    \centering
    \begin{tabular}{lc@{\hspace{0.25cm}}c@{\hspace{0.25cm}}c@{\hspace{0.25cm}}c@{\hspace{0.25cm}}c@{\hspace{0.25cm}}c}
\toprule

\textit{Skewed} &                  ECE &                       MCE &                        KS &                     Brier &                 MSE$_p$ &                     KL$_p$  \\
\midrule
CE + label resampling        &   2.7$\pm$0.46 &   7.64$\pm$2.14 &  1.05$\pm$0.39 &   11.0$\pm$0.51 &  0.22$\pm$0.01 &   0.92$\pm$0.03 \\
\midrule
CE early-stop          &  3.07$\pm$0.57 &   7.88$\pm$1.88 &  1.28$\pm$0.41 &   11.18$\pm$0.5 &   0.4$\pm$0.01 &   1.79$\pm$0.06 \\
Temperature        &  3.14$\pm$0.49 &   7.92$\pm$1.84 &  1.12$\pm$0.33 &  11.22$\pm$0.47 &   0.4$\pm$0.02 &   1.76$\pm$0.06 \\
Platt Scaling     &  2.99$\pm$0.53 &   7.73$\pm$1.59 &  1.07$\pm$0.37 &   11.1$\pm$0.54 &  0.39$\pm$0.01 &   1.72$\pm$0.06 \\
Dirichlet Cal.  &  3.04$\pm$0.73 &   7.81$\pm$2.43 &   0.97$\pm$0.3 &  11.22$\pm$0.42 &  0.47$\pm$0.02 &   2.31$\pm$0.07 \\
Focal Loss       &  8.29$\pm$0.67 &  14.93$\pm$1.43 &  6.16$\pm$0.67 &  12.01$\pm$0.41 &  1.28$\pm$0.03 &  1.63$\pm$0.66 \\
Mix-n-match &  2.99$\pm$0.53 &   7.78$\pm$1.78 &  1.08$\pm$0.32 &  11.18$\pm$0.49 &   0.4$\pm$0.01 &   1.75$\pm$0.05 \\
Entropy Reg.    &  7.67$\pm$0.57 &   14.43$\pm$1.5 &   5.2$\pm$0.71 &  11.94$\pm$0.45 &  1.18$\pm$0.03 &  10.74$\pm$0.65 \\
MMCE Reg.        &  3.68$\pm$0.59 &  10.94$\pm$2.76 &  1.47$\pm$0.31 &  11.14$\pm$0.44 &  0.54$\pm$0.02 &   2.44$\pm$0.08 \\
Deep Ensemble  &   2.87$\pm$0.5 &   7.21$\pm$1.63 &  1.36$\pm$0.44 &   11.28$\pm$0.5 &  0.55$\pm$0.02 &   2.58$\pm$0.07 \\
\midrule
CaPE (bin)         &   3.29$\pm$0.5 &   8.18$\pm$1.51 &  1.17$\pm$0.34 &  11.07$\pm$0.47 &   0.4$\pm$0.02 &   1.73$\pm$0.06 \\
CaPE (kernel)          &   3.16$\pm$0.5 &   8.14$\pm$1.58 &  1.09$\pm$0.33 &  11.17$\pm$0.53 &  0.39$\pm$0.01 &   1.69$\pm$0.06 \\
\bottomrule
\end{tabular}
    \caption{Performance on Face-based Risk Prediction. \textit{Skewed} scenario. All numbers are downscaled by $10^{-2}$.}
    \label{tab:scaled_ci_face}
\end{table}

\begin{table}[]
    \centering
    \begin{tabular}{lc@{\hspace{0.3cm}}c@{\hspace{0.3cm}}c@{\hspace{0.3cm}}c@{\hspace{0.3cm}}c@{\hspace{0.3cm}}c}
\toprule
\textit{Discrete} &                    ECE &                       MCE &                        KS &                     Brier &                 MSE$_p$ &                     KL$_p$ \\
\midrule
CE + label resampling        &  4.23$\pm$0.74 &   11.16$\pm$2.5 &  1.45$\pm$0.49 &  20.38$\pm$0.35 &  1.52$\pm$0.05 &  3.63$\pm$0.12 \\
\midrule
CE early-stop          &   6.7$\pm$0.86 &  18.62$\pm$3.52 &  2.61$\pm$0.53 &  21.91$\pm$0.36 &  2.24$\pm$0.08 &  5.27$\pm$0.17 \\
Temperature       &  6.12$\pm$0.87 &  16.82$\pm$3.56 &  3.37$\pm$0.86 &  21.76$\pm$0.35 &  2.21$\pm$0.08 &  5.15$\pm$0.18 \\
Platt Scaling    &   4.7$\pm$0.72 &  11.69$\pm$2.44 &  1.67$\pm$0.51 &  21.44$\pm$0.32 &  2.06$\pm$0.08 &  4.83$\pm$0.17 \\
Dirichlet Cal.  &  7.13$\pm$0.86 &  22.67$\pm$5.08 &  3.18$\pm$0.68 &   22.1$\pm$0.34 &   2.74$\pm$0.1 &  6.53$\pm$0.22 \\
Focal Loss      &   5.7$\pm$0.75 &  13.68$\pm$2.32 &  4.62$\pm$0.91 &  21.77$\pm$0.28 &  2.92$\pm$0.09 &  6.77$\pm$0.21 \\
Mix-n-match &  6.27$\pm$0.76 &  16.83$\pm$2.95 &  3.47$\pm$0.93 &  21.77$\pm$0.33 &  2.21$\pm$0.08 &  5.14$\pm$0.18 \\
Entropy Reg.    &  6.69$\pm$0.87 &  15.38$\pm$2.43 &  6.03$\pm$1.13 &  21.79$\pm$0.31 &  2.84$\pm$0.08 &  6.62$\pm$0.19 \\
MMCE Reg.        &   3.96$\pm$0.7 &    10.4$\pm$2.4 &  1.51$\pm$0.47 &  21.12$\pm$0.35 &  2.09$\pm$0.08 &  4.92$\pm$0.18 \\
Deep Ensemble  &  4.76$\pm$0.74 &  11.49$\pm$2.23 &  2.04$\pm$0.61 &  21.17$\pm$0.31 &  1.97$\pm$0.08 &  4.61$\pm$0.17 \\
\midrule
CaPE (bin)         &  5.41$\pm$0.74 &  14.45$\pm$3.15 &  2.24$\pm$0.59 &  21.33$\pm$0.36 &  1.81$\pm$0.08 &  4.28$\pm$0.18 \\
CaPE (kernel)          &   4.96$\pm$0.8 &  12.97$\pm$2.63 &  2.18$\pm$0.58 &  21.21$\pm$0.42 &  1.84$\pm$0.08 &  4.35$\pm$0.17 \\
\bottomrule
\end{tabular}
    \caption{Performance on Face-based Risk Prediction. \textit{Discrete} scenario. All numbers are downscaled by $10^{-2}$.}
    \label{tab:stepwise_ci_face}
\end{table}

\newpage
\newpage

\subsection{Supplementary Metrics on Real-world Dataset}
\label{app:real_world_results}
We present here additional metrics on the real world data: Cancer Survival (Table~\ref{tab:cancer_app}); Climate Forecasting (Table~\ref{tab:weather_app}); Collision Prediction (Table~\ref{tab:traffic_app}). 
\begin{table}[h]
    \centering
    \begin{tabular}{lrrrrrc}
\toprule
Methods ~~~~~~$(\times 10^{-2})$ &    AUC &    ECE &    MCE &   NLL &  Brier &  KS \\
\midrule
CE Early-stop               &  58.88 &  12.25 &  25.35 &  67.92 &  23.96 &      6.44 \\
Temperature      &  58.88 &  12.07 &  24.65 &  67.11 &  23.73 &      6.92 \\
Platt Scaling &  58.91 &  10.28 &  27.69 &  66.11 &  23.33 &      4.91 \\
Dirichlet Cal.        &  49.89 &  13.83 &  35.52 &  67.57 &  24.08 &      6.00 \\
Mix-n-match      &  58.88 &  12.16 &  \textbf{24.52} &  66.89 &  23.67 &      7.18 \\
Focal loss         &  55.02 &  12.15 &  26.34 &  65.92 &  23.31 &      6.38 \\
Entropy Reg.     &  56.29 &  11.73 &  30.81 &  66.49 &  23.62 &      6.83 \\
MMCE Reg. & 48.45 & 11.84 & 37.36 & 66.83 & 23.73 & 3.64\\
Deep Ensemble          &  52.26 &  9.99 &  28.30 &  66.22 &  23.47 &     5.02 \\
\midrule
CaPE (bin)       & \textbf{ 61.44 }&  12.31 &  25.27 &  65.75 &  23.20 &      \textbf{2.59} \\
CaPE (kernel)        &  61.22 &  \textbf{9.48} &  32.40 &  \textbf{65.70} &  \textbf{23.18} &      3.70 \\
\bottomrule
\end{tabular}
    \caption{Baselines with full metrics for cancer survival. All numbers are downscaled by $10^{-2}$.}
    \label{tab:cancer_app}
\end{table}
\begin{table}[h!]
    \centering
    \begin{tabular}{lrrrrrc}
\toprule
Methods ~~~~~~$(\times 10^{-2})$ &    AUC &    ECE &    MCE &   NLL &  Brier &  KS \\
\midrule
CE Early-stop&77.64&10.91&25.50&59.97&20.57&11.03\\
Temperature&77.64&8.66&23.56&58.77&20.21&7.41\\
Platt Scaling&77.65&6.97&16.47&57.38&19.53&3.26\\
Dirichlet Cal.&77.51&14.29&30.09&62.83&21.89&5.21\\
Mix-n-match&77.64&8.65&23.58&58.77&20.21&7.39\\
Focal Loss&76.18&8.32&21.25&59.01&20.27&4.45\\
Entropy Reg&79.01&10.53&20.72&57.83&19.77&5.00\\
MMCE Reg&76.69&8.46&19.73&59.25&20.12&7.31\\
Deep Ensemble&\textbf{79.86}&7.41&18.24&55.28&18.82&7.57\\
\midrule
CaPE (bin)&78.99&5.16&15.09&\textbf{79.00}&\textbf{18.37}&\textbf{2.34}\\
CaPE (kernel) &79.00&\textbf{5.08}&\textbf{13.28}&54.32&18.39&\textbf{2.34}\\
\bottomrule
\end{tabular}
    \caption{Baselines with full metrics for weather prediction. All numbers are downscaled by $10^{-2}$.}
    \label{tab:weather_app}
\end{table}
\begin{table}[h!]
    \centering
    \begin{tabular}{lrrrrrc}
\toprule
Methods ~~~~~~$(\times 10^{-2})$ &    AUC &    ECE &    MCE &   NLL &  Brier &  KS \\
\midrule
CE Early-stop&85.68&4.36&19.87&31.67&8.59& 1.54\\
Temperature& 85.68&4.56&16.79&30.36&8.52& 2.9\\
Platt Scaling& 85.76&	3.04& 12.39&	\textbf{29.42}&	8.23& \textbf{1.52}\\
Dirichlet Cal.& 83.36&	5.78&	18.13&	30.90&	8.77& 1.60\\
Mix-n-match& 85.68&	4.40&	17.41&	30.25&	8.52& 2.60\\
Focal Loss& 82.21&	9.07&	19.85&	34.41&	9.82&  8.72\\
Entropy Reg&83.15&	14.54&	21.27&	38.74&	11.10& 13.44\\
MMCE Reg.&85.18&	\textbf{2.94}&	\textbf{8.95}&	30.65&	8.48& 2.44\\
Deep Ensemble& 85.27&	3.15&	16.53&	30.20&	8.54& 2.01\\
\midrule
CaPE (bin)&85.70&  	3.16&	12.21&	30.61&	8.18& 2.13\\
CaPE (kernel)& \textbf{85.95}&	3.22&	13.32&	30.44&	\textbf{8.13}& 2.10\\
\bottomrule
\end{tabular}
    \caption{Baselines with full metrics for collision prediction. All numbers are downscaled by $10^{-2}$.}
    \label{tab:traffic_app}
\end{table}


\newpage
\subsection{Additional Reliability Diagram}
\label{app:app_reliability}
Figure~\ref{fig:reliability_curves_all} shows additional reliability curves on the real world data, supplementing the ones illustrated in Figure~\ref{fig:real_world_reliablility_curve}. 
\begin{figure}[tp]
    \centering
    \includegraphics[width=0.99\textwidth]{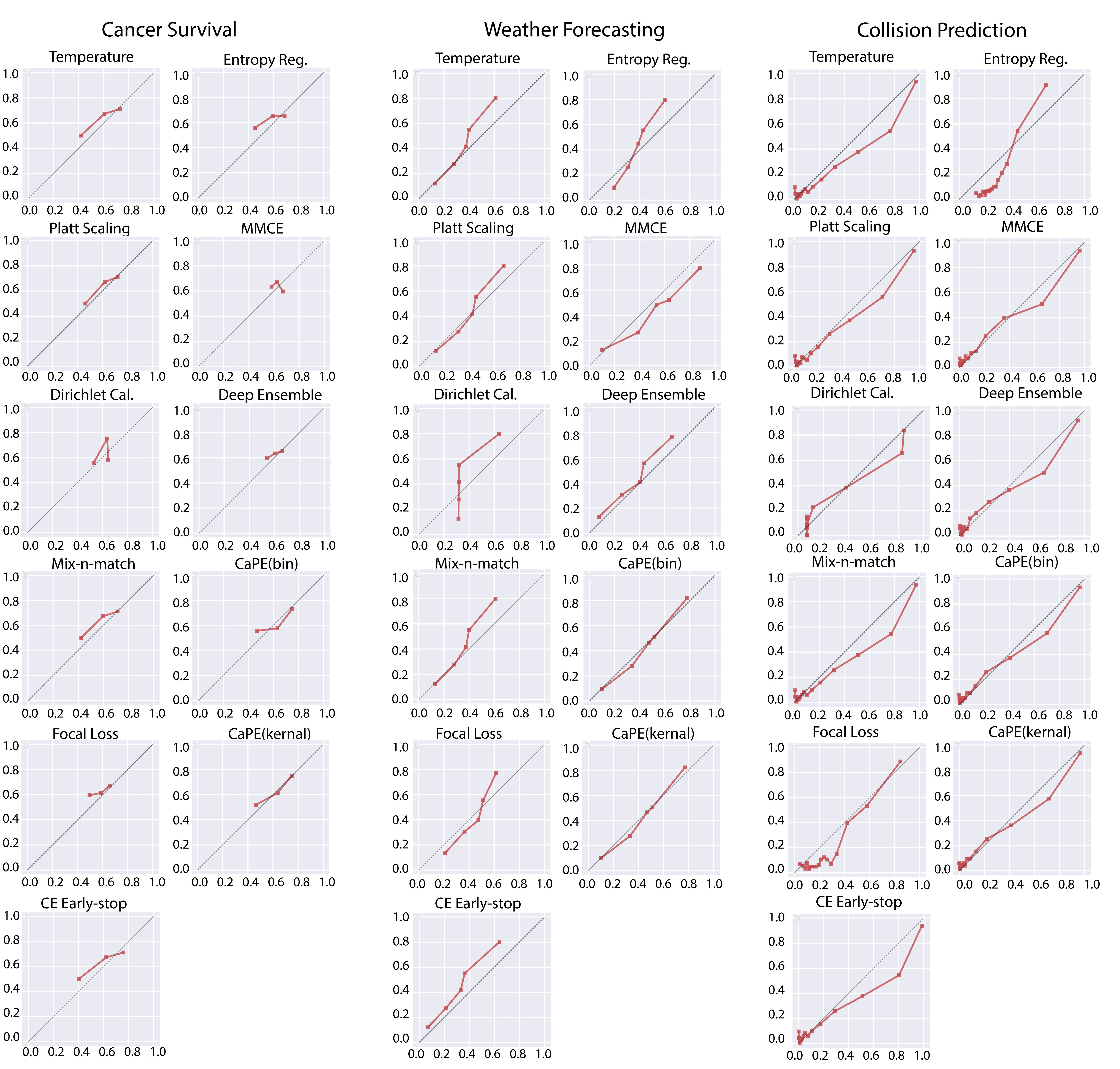}
    \caption{Reliability diagrams of all the baselines on the real-world datasets. We train all baseline methods on each of the datasets and plot the empirical probability(y-axis) against predicted probability(x-axis). The axis labels are removed to ease visualization (they are the same as in Figure~\ref{fig:real_world_reliablility_curve}). }
    \label{fig:reliability_curves_all}
\end{figure}
Figure~\ref{fig:synthetic_5_reliability_curves}  shows reliability curves for the different synthetic data scenarios. 
\begin{figure}[tp]
    \centering
    \begin{tabular}{c}
    \textit{\textbf{Linear}} \\
    \includegraphics[trim={0 0 0 0},clip, width=0.85\textwidth]{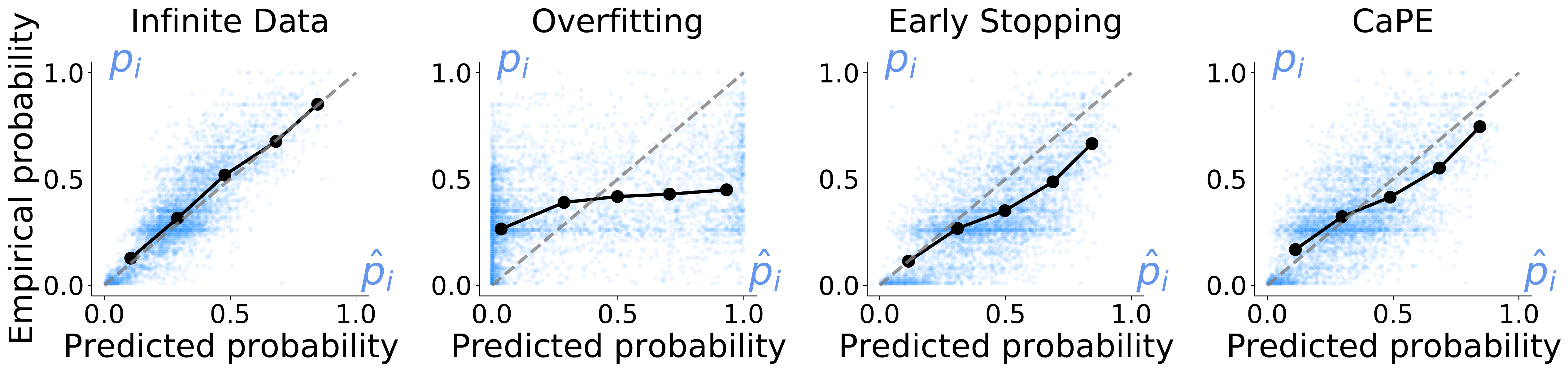} \\[2.5ex] 
    \textit{\textbf{Sigmoid}} \\
    \includegraphics[trim={0 0 0 0},clip, width=0.85\textwidth]{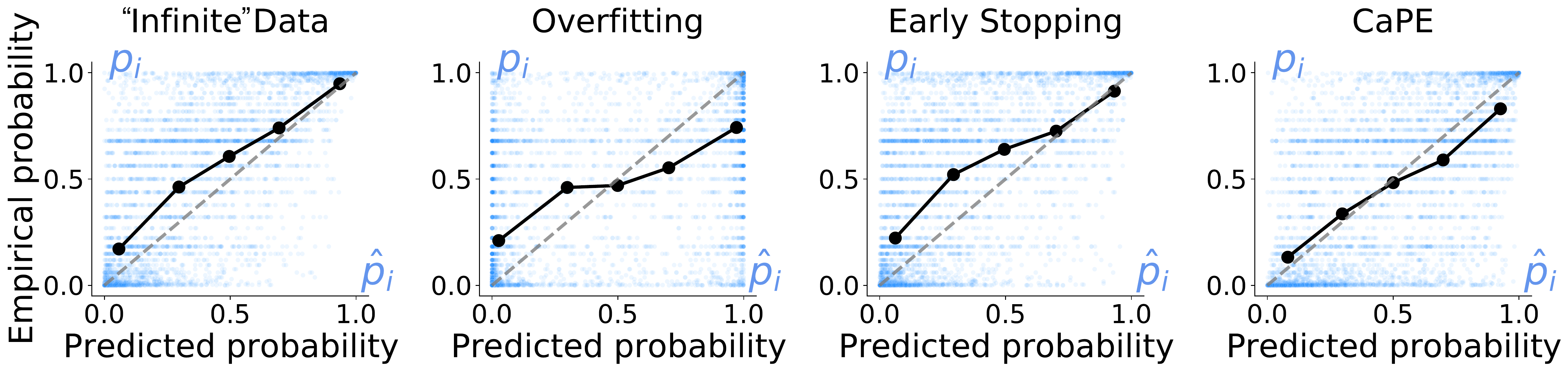} \\ [2.5ex]
    \textit{\textbf{Centered}} \\
    \includegraphics[trim={0 0 0 0},clip, width=0.85\textwidth]{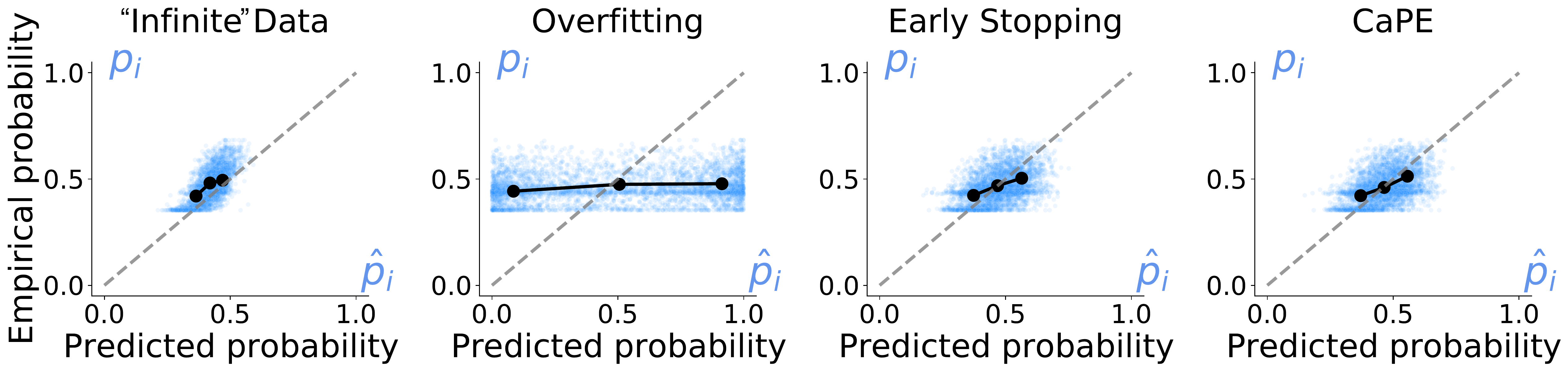} \\ [2.5ex]
    \textit{\textbf{Skewed}} \\
    \includegraphics[trim={0 0 0 0},clip, width=0.85\textwidth]{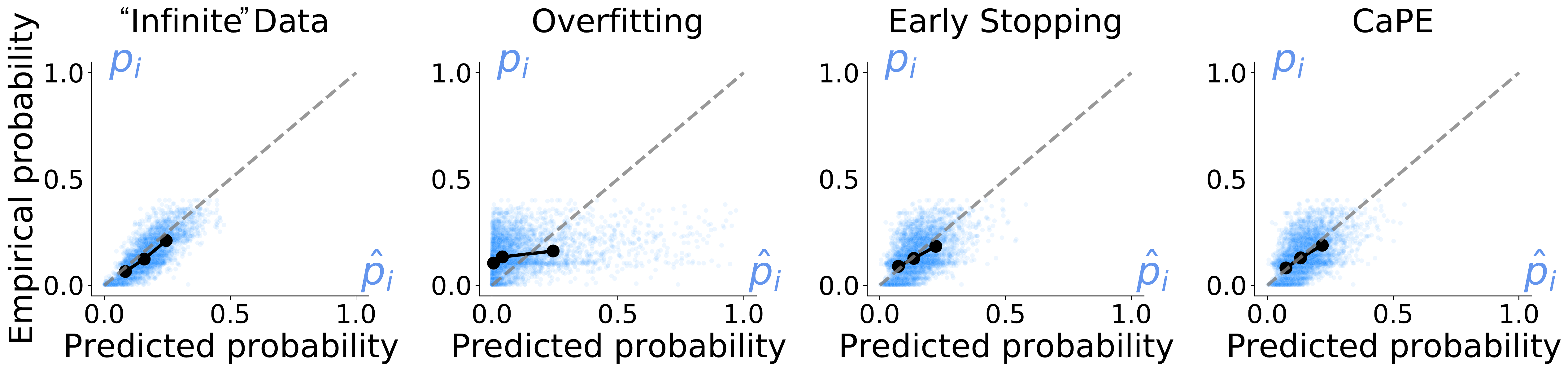} \\[2.5ex] 
    \textit{\textbf{Discrete}} \\
    \includegraphics[trim={0 0 0 0},clip, width=0.85\textwidth]{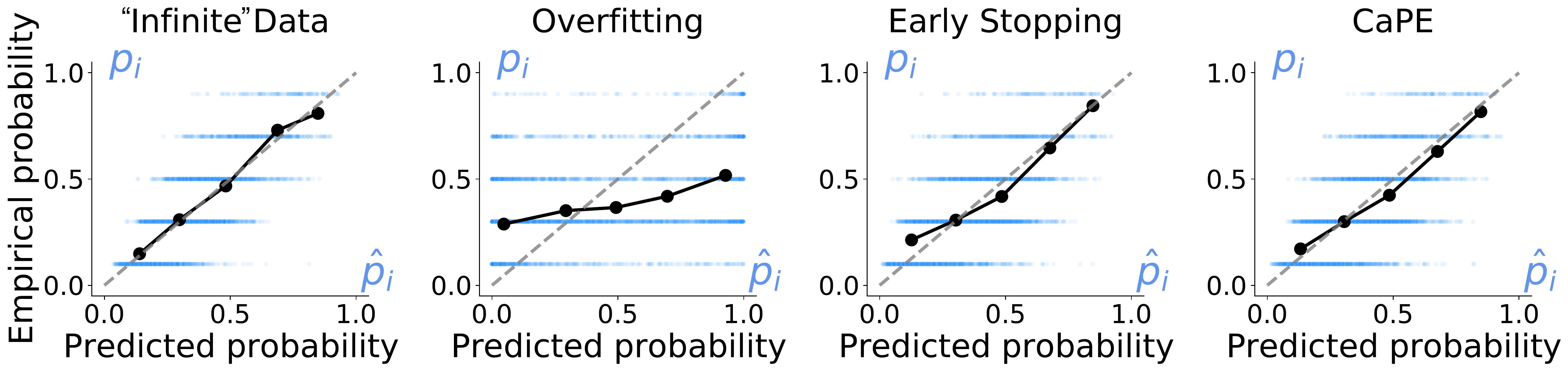} \\
    \end{tabular}
    \caption{Reliability diagrams for different synthetic data scenarios. We can see that CaPE outperforms early stopping, prevents overfitting, and achieves a performance on par with training on large amount of resampled data.  }
    \label{fig:synthetic_5_reliability_curves}
\end{figure}

\section{Kolmogorov-Smirnov Error}
\label{app:KS_error}
We derive the KS-error, mentioned in Section~\ref{sec:metrics}.

For a calibrated estimator
    $$\mathbb{P}\left(y=1|f(\vx)\in I(q)\right)=q, \hspace{0.5em}\forall 0\le q\le 1,$$ 
    for some small interval $I(q)$ around $q$. 
    
    Hence
    $$\mathbb{P}\left(y=1,f(\vx)\in I(q)\right)=\mathbb{P}\left(f(\vx)\in I(q)\right)\hspace{0.1em}q, \hspace{0.5em}\forall 0\le q\le 1.$$ 
    Similarly to the Kolmogorov-Smirnov (KS) test for distribution functions, we can recast this property in integral form
    $$\phi_1(\sigma)=\int\limits_0^\sigma \mathbb{P}\left(y=1,f(\vx)\in I(q) \right)dq,\quad \phi_2(\sigma)=\int\limits_0^\sigma \mathbb{P}\left(f(\vx)\in I(q)\right)\hspace{0.1em} q dq.$$
    We can evaluate $\phi_1,\phi_2$ from a finite sample $(\vx_i,y_i), i=1\dots n,$
    $$\phi_1(\sigma)=\frac{1}{n}\sum\limits_{i=1}^n\mathbf{1}(y_i=1,f(\vx_i)\le \sigma), \quad\phi_2(\sigma)=\frac{1}{n}\sum\limits_{i=1}^n\mathbf{1}(f(\vx_i)\le \sigma)f(\vx_i).$$
    The KS error is defined as $$\text{KS}=\max\limits_{1\le \sigma\le 1}|\phi_1(\sigma)-\phi_2(\sigma)|.$$
    $\phi_1,\phi_2$ can be efficiently computed by sorting the data points with respect to their confidence scores $f(\vx_i)$. The KS error has the advantage of being independent of binning configurations, unlike ECE and MCE.
\section{Brier Score Decomposition}
\label{app:brier_decomp}
We present here a decomposition of the Brier score into two components, discussed in Section~\ref{sec:metrics}.

The Brier score can be interpreted as a sum of two terms, calibration and refinement. Assume the network can output one of $K$ distinct possible predictions, i.e., $\hat{p}\in\{\hat q_1,\dots,\hat q_K\}$. 
    
    Denote $S_k$, the set of all inputs with output $\hat{q}_k$ and $\bar{q}_k$ the empirical probability over $S_k$, i.e.,
     $$S_k=\{\vx|f(\vx)=\hat q_k\},\hspace{1em}|S_k|=n_k,\hspace{1em} \bar{q}_k=\frac{1}{n_k}\sum\limits_{\vx_i \in S_k}y_i.$$ 
    Then we can write
    $$\text{Brier}=\frac{1}{N}\sum\limits_{i=1}^N(\hat p_i-y_i)^2=\frac{1}{N}\sum\limits_{k=1}^K n_k(\hat q_k-\bar{q}_k)^2+\frac{1}{N}\sum\limits_{k=1}^K n_k\bar{q}_k(1-\bar{q}_k),$$
    The first term on the RHS, calibration, is similar to MSE$_p$, with the empirical probabilities $\bar q_k$ substituting for the true labels. The second term, refinement, is an estimate of the confidence in determining $\bar{q}_k$. It is related to the area under curve (AUC), which measures to the achievable accuracy of the network as a classifier. The term is smaller as the prediction classes $\bar{q}_k$ tend towards $0$ or $1$. Thus, this term penalizes empirically calibrated predictors, with low discriminative power, as in Figure~\ref{fig:pred_diag}. 

\section{Metric Comparison}
\label{app:metric_comparsion_appendix}
Figure~\ref{fig:metric_comparsion_appendix} shows the correlation between different calibration and accuracy metrics, and two gold-standard metric that use ground truth probabilities: $\text{MSE}_p$
and KL-divergence. The correlations are computed using all five scenarios in our Face-based Risk Prediction synthetic dataset. 

\begin{figure}[tp]
    \centering

    \begin{tabular}{c}
    \textit{Linear} \\
    \includegraphics[trim={0 108 0 0},clip, width=0.95\textwidth]{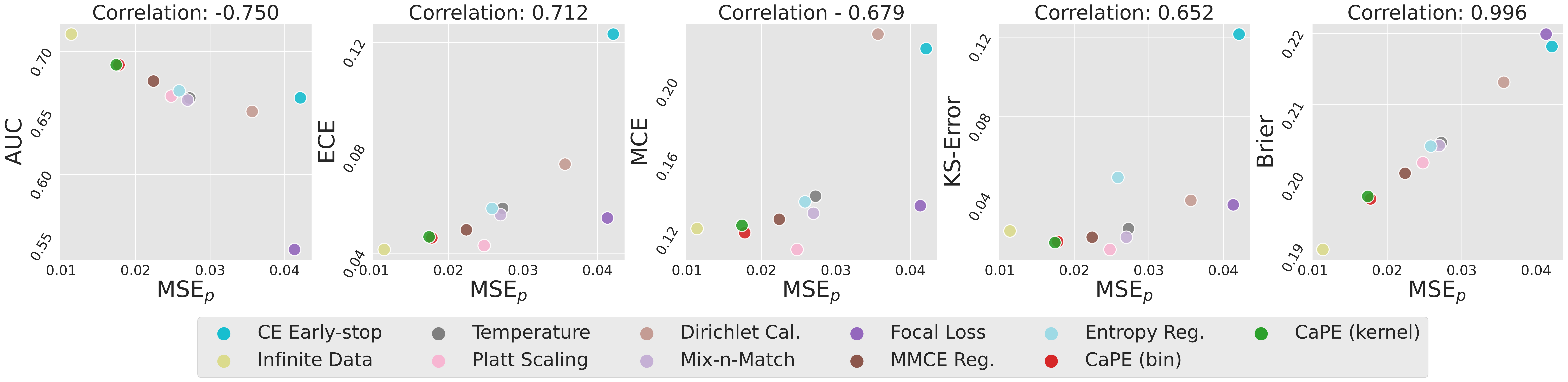} \\ 
    \textit{Sigmoid} \\
    \includegraphics[trim={0 108 0 0},clip, width=0.95\textwidth]{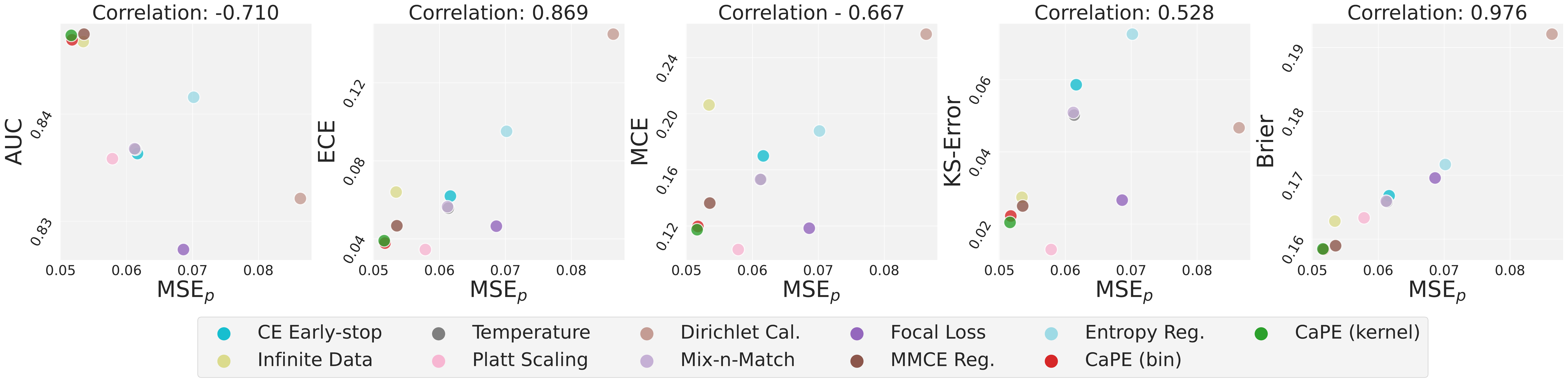} \\ 
    \textit{Centered} \\
    \includegraphics[trim={0 108 0 0},clip, width=0.95\textwidth]{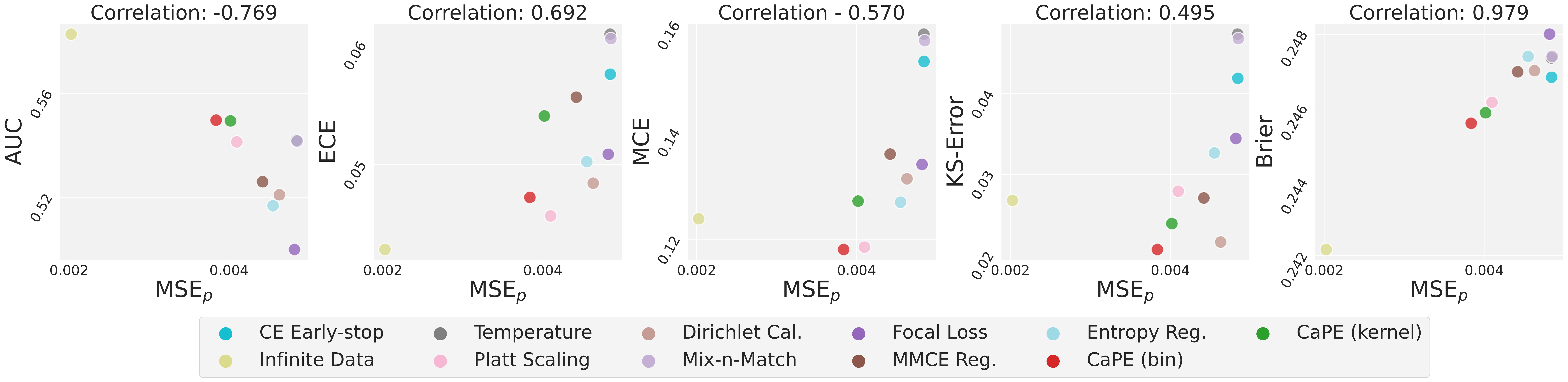} \\ 
    \textit{Skewed} \\
    \includegraphics[trim={0 108 0 0},clip, width=0.95\textwidth]{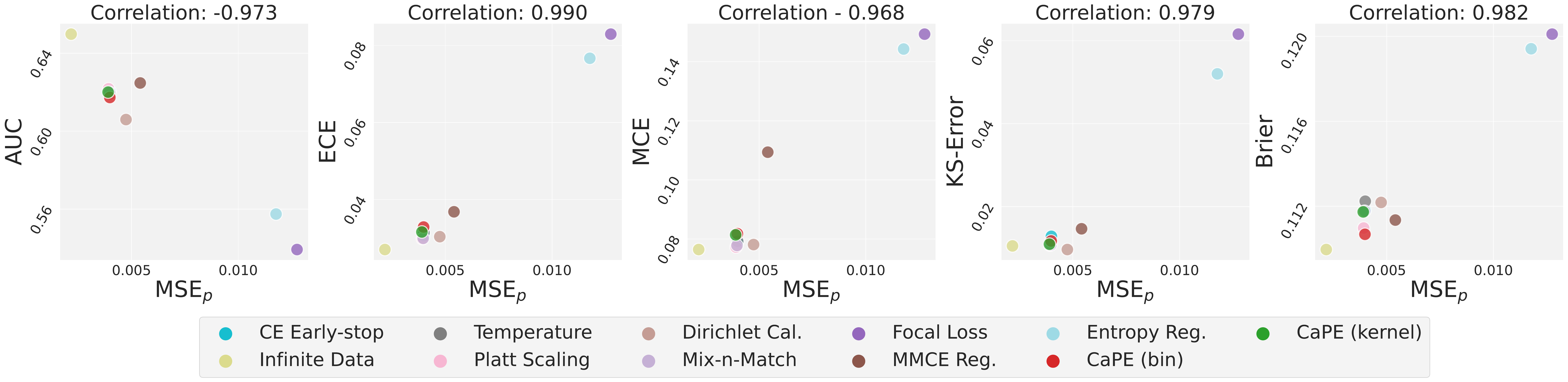} \\ 
    \textit{Discrete} \\
    \includegraphics[width=0.95\textwidth]{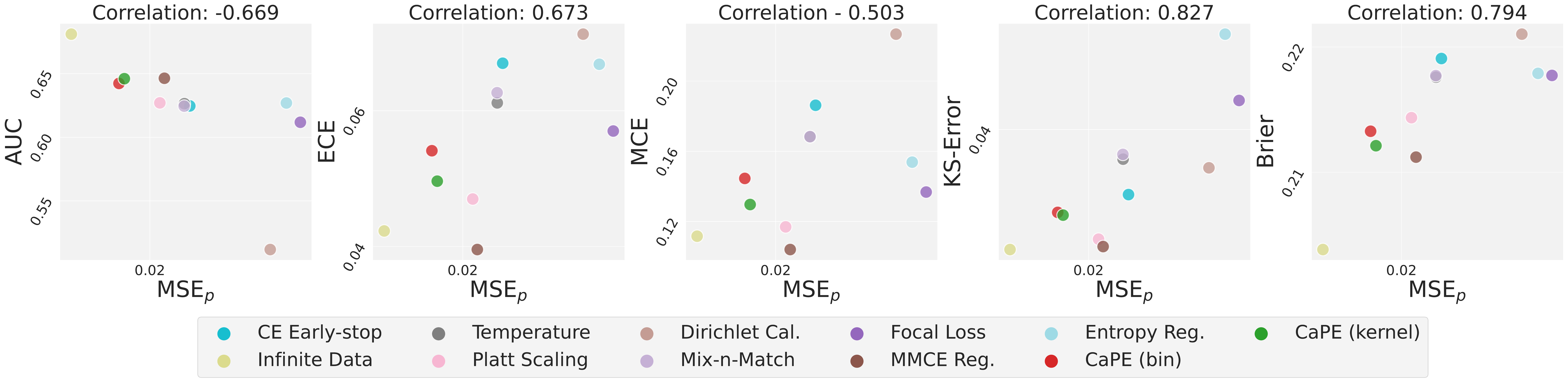} \\ 
    \end{tabular}
    \caption{Correlation between MSE$_p$ and other metrics on synthetic data. Brier score achieves the most consistent correlation with $\text{MSE}_p$.}
    \label{fig:metric_comparsion_appendix}
    
\end{figure}
\begin{figure}[tp]
    \centering
    \begin{tabular}{c}
    \textit{Linear} \\
    \includegraphics[trim={0 108 0 0},clip, width=0.95\textwidth]{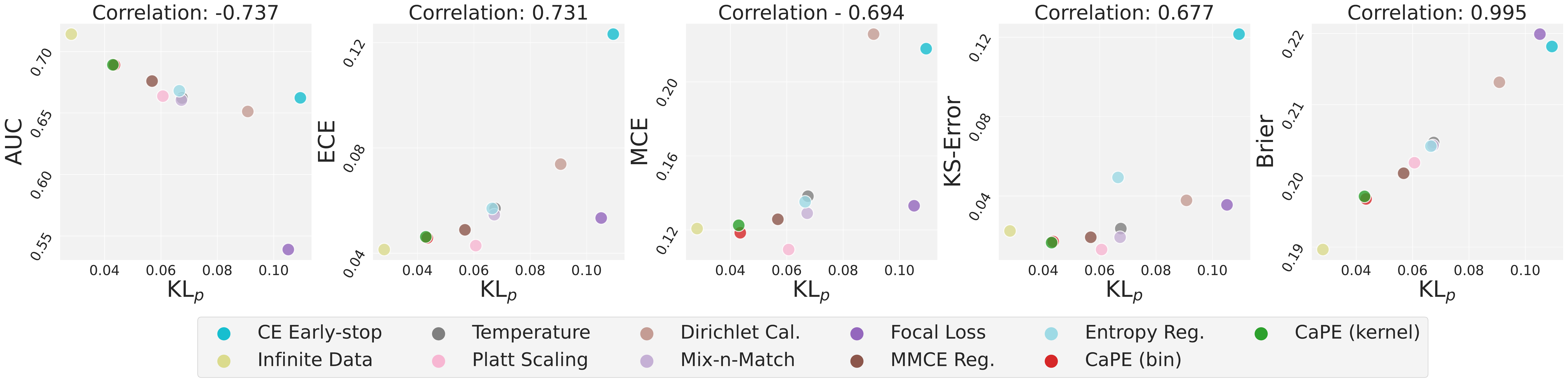} \\ 
    \textit{Sigmoid} \\
    \includegraphics[trim={0 108 0 0},clip, width=0.95\textwidth]{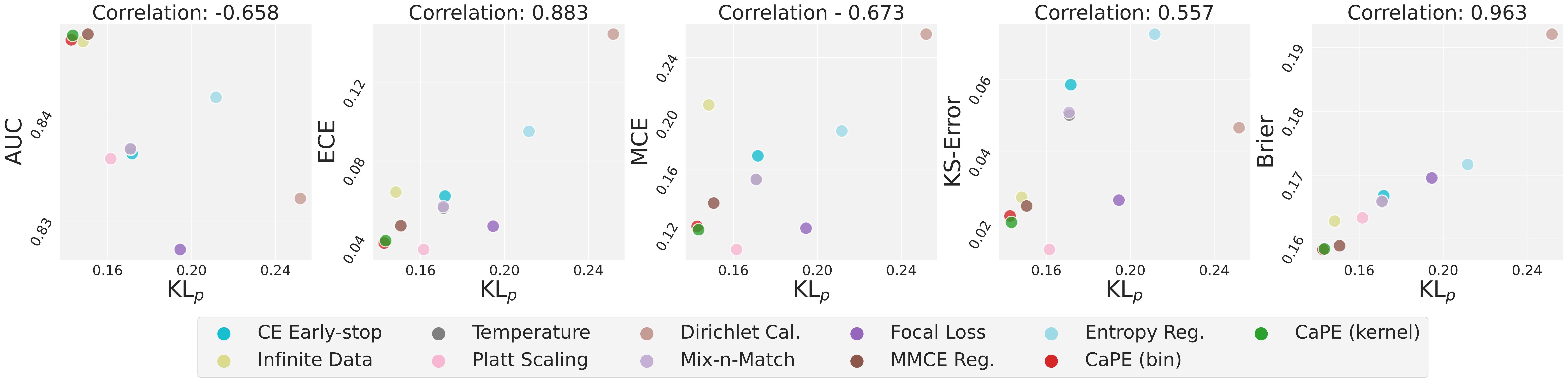} \\ 
    \textit{Centered} \\
    \includegraphics[trim={0 108 0 0},clip, width=0.95\textwidth]{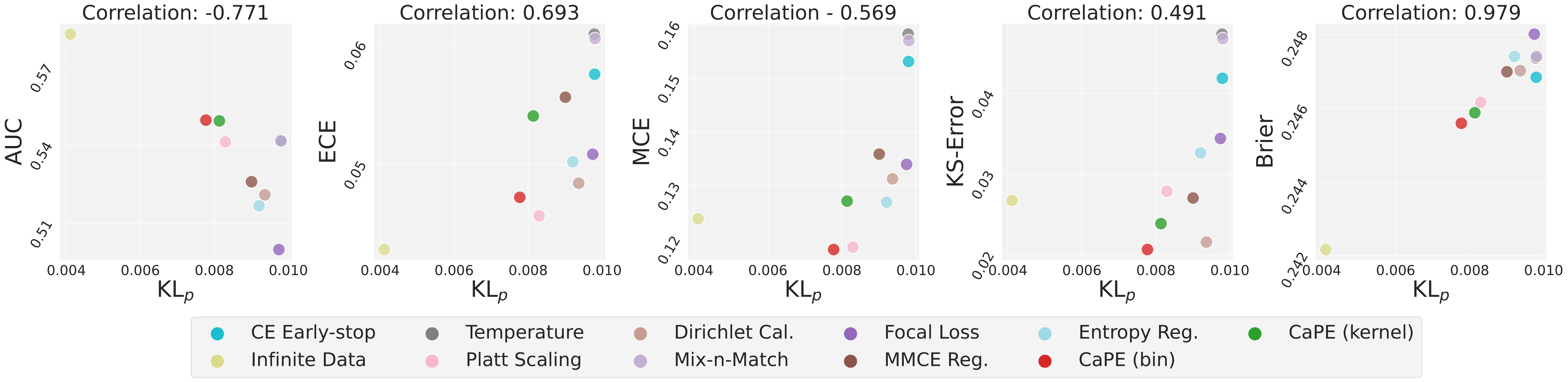} \\ 
    \textit{Skewed} \\
    \includegraphics[trim={0 108 0 0},clip, width=0.95\textwidth]{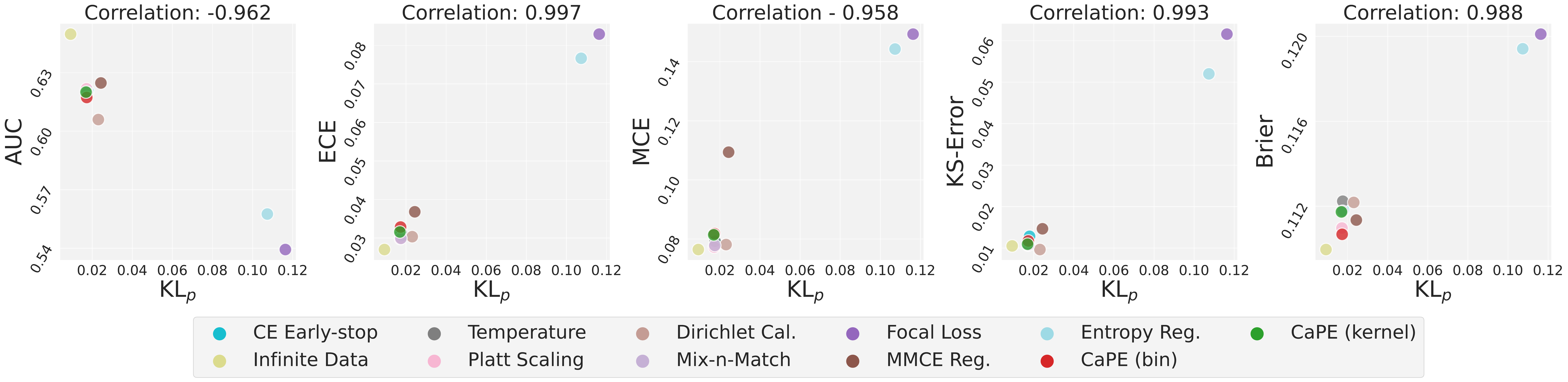} \\ 
    \textit{Discrete} \\
    \includegraphics[width=0.95\textwidth]{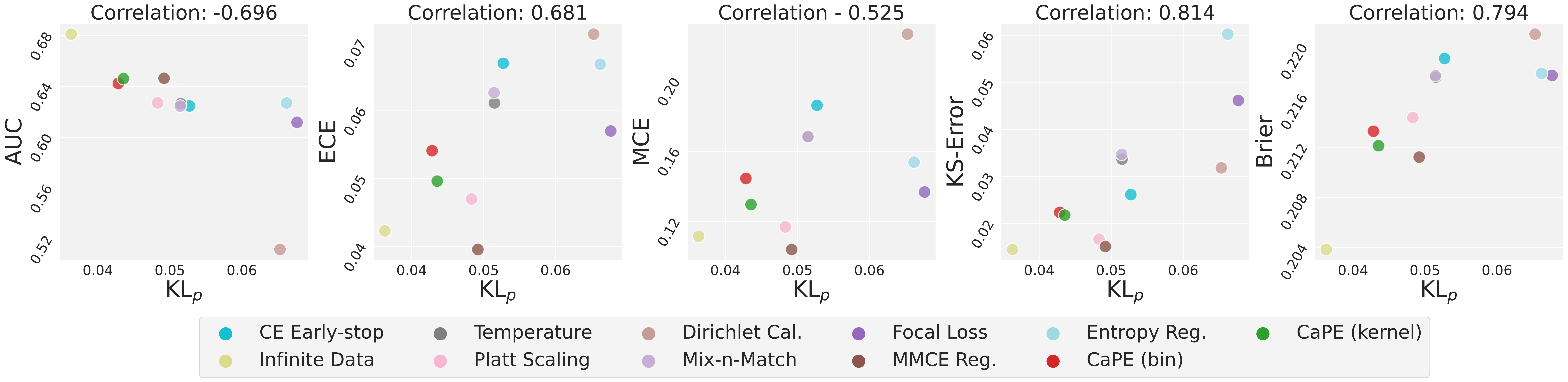} \\ 
    \end{tabular}
    
    \caption{Correlation between KL$_p$ and other metrics on synthetic data. Brier score achieves the most consistent correlation with $\text{KL}_p$.}
    \label{fig:metric_comparsion_appendix_1}
\end{figure}

\section{Estimation of Empirical Probability in CaPE}
\label{app:emp_estimation}
 We describe in further detail the two ways to estimate the conditional probability $\mathbbm{P}\left(y = 1 | f(\vx) \in I(q)\right)$, introduced in Section~\ref{sec:our_method}.

We wish to estimate the conditional probability of an output $y$ given a network prediction $f(\vx)$,  $\mathbbm{P}\left(y = 1 | f(\vx) \in I(q)\right)$
We can approximate the probability by averaging over points $\hat p\in I(q)$,
\begin{align}
    \mathbbm{P}\left(y = 1 | f(\vx) \in I(q)\right) &\approx \frac{1}{|I(q)|}\sum_{\hat{p} \in I(q)}\mathbbm{P}\left(y=1| f(\vx) = {p}\right).
\end{align}
An empirical estimate of $\mathbbm{P}\left(y = 1 | f(\vx) \in I(q)\right)$ would be
\begin{align}
    \mathbb{P}\left(y=1| f(\vx) \in I(q)\right) \approx \frac{1}{|\text{Index}(I(q))|}\sum_{f(\vx_i) \in I(q)}y_i,
\end{align}
where $\text{Index}(I_q)=\{i|f(\vx_i)\in I(q)\}$. 

Alternatively, we can use kernel  estimation:
\begin{align}
    \mathbbm{P}\left(y = 1 | f(\vx) \in I(q)\right) &\approx \frac{1}{Z} \sum_{\hat{p} \in I(q)}\mathbbm{P}\left(y=1| f(\vx) = \hat{p}\right)\cdot \exp\left(-\frac{\left(p - q\right)^2}{\sigma^2} \right),
\end{align}
where $Z=\sum_{{p}\in I(q)}\exp\left(-\frac{\left({p} - q\right)^2}{\sigma^2} \right)$ is the normalization factor.
An empirical estimate of the conditional probability would then be
\begin{align}
    \mathbbm{P}\left(y = 1 | f(\vx) \in I(q)\right)
    & \approx \frac{1}{Z}\sum_{f(\vx_i) \in I(q)}y_i\exp\left(-\frac{\left(f(\vx_i) - q\right)^2}{\sigma^2} \right).
\end{align}
Based on these two approximation methods, we can design an algorithm to estimate $p_\text{emp}^i$.

\textbf{Bin}  We divide our data into $B$ bins of equal size. $Q_1,\dots,Q_B$ are the data B-quantiles. We wish to estimate $\mathbb{P}\left(y=1 | f(\vx)\in [Q_{b-1},Q_b]\right),\hspace{0.2em}b=1,\dots,B,Q_0=0$. Denote $I_b:=[Q_{b-1},Q_b] \cap \{f(\vx_i)\}_{i=1}^N$, set of all predictions in $[Q_{b-1},Q_b]$, and $\text{Index}(I_b)=\{i|f(\vx_i)\in I_b\}$. 
We have,
$$\mathbb{P}\left(y=1 | f(\vx)\in [Q_{b-1},Q_b\right)]\approx p_\text{emp}^{(b)}=\frac{1}{|I_b|}\sum\limits_{i\in \text{Index}(I_b)}y_i$$
We assign $p_\text{emp}^{(b)}$ to all data points $i$ in the $b$-th quantile
$$p_\text{emp}^{i}=p_\text{emp}^{(b)}\hspace{1em}\forall i\in \text{Index}(I_b)$$

    
\textbf{Kernel}   
In this case we use  kernel estimation:
\begin{equation}
        p_{\text{emp}}^i  = \frac{\sum_{k \in \text{NN}(i,r)} K\left(i,k\right) y_{k}}{\sum_{k\in \text{NN}(i,r)}K\left(i,k\right)}.
        \end{equation}
    NN$(i,r)$ defines $r$ data points whose predictions are nearest to $\hat{p_i}=f(\vx_i)$.  
    $K(i,j)$ is the Gaussian kernel$$
    K(i,j)=\exp\left(- \frac{\left(\hat{p}_{i} - \hat{p}_{j}\right)^2}{\sigma^2}\right),$$ with hyperparameter $\sigma$.

\section{Calibration Baselines}
This section provides a review of the baseline methods, discussed in Section~\ref{sec:experiments}.
\label{app:baselines}
\paragraph{Post-processing}
Postprocessing consists of finding a function $f :[0,1] \rightarrow [0, 1]$, that transforms the model outputs $\hat{p}_i\rightarrow f(\hat{p}_i)$ to improve their calibration. 
\begin{itemize}
    \item In Platt scaling \citep{Platt1999} the model predictions are used as inputs to a logistic regression model optimized using a validation set,
    \begin{equation}
    f_1(\hat{p}_i) = \sigmoid \left( W^T \hat{p}_i + b \right)
    \end{equation}
    where $W \in \mathbb{R}^2, b \in \mathbb{R}$ and $\sigma$ is the sigmoid function. 
    \item Temperature scaling \citep{guo2017calibration} is a single-parameter variant of Platt Scaling where we change a temperature parameter in the logistic function.
    \item Beta/Dirichlet calibration (Dir-ODIR) \citep{betacalibration, dirchlet} assumes that the probabilities can be parametrized by a Beta/Dirichlet distribution i.e. 
    \begin{equation}
        f_j \sim \text{Beta}(\alpha^{(j)}, \beta^{(j)})
    \end{equation}
    Assuming the prior to be $p(y = j) = \pi_j, \pi_j \in [0,1]$, we have $P(y | f_j) \propto \pi_j f_j$, and then $\alpha^{(j)}, \beta^{(j)}$ are estimated by maximizing the posterior.
\end{itemize}

\paragraph{Ensembling} These calibration methods simultaneously train several neural networks, varying parameters in the training process. The final output is a function of all the different outputs.

\begin{itemize}
    \item Mix-n-Match \citep{Zhang2020MixnMatchEA} improves calibration by ensembling parametric and non-parametric calibrators. Denote the temperature scaling function with $g(\hat{y}_i, T)$. Then Mix-n-Match ensembles different temperatures
    \begin{equation}
        f_j(\hat{p}_i) = w_1 g_j(\hat{p}_i, T) + w_2 g_j(\hat{p}_i, 0) + w_3 g_j(\hat{p}_i, \infty).
    \end{equation}
    After ensembling the parametric temperature scaling, Mix-n-Match applies non-parametric isotonic regression.
    \item Deep ensemble \citep{deepensemble} trains $M$ copies of the neural network with different initializations. The final estimate is the average of all single model outputs
    \begin{equation}
        p(y_i \mid x_i) = \frac{1}{M} \sum_j^M p_{\theta_j}(y_i \mid x_i).
    \end{equation}
\end{itemize}

\paragraph{Modified training} These calibration methods train the neural networks from end to end, modifying the training process to improve calibration.
\begin{itemize}
    \item Confidence penalty \citep{entropy} penalizes low entropy output distributions (confidence penalty), 
    \begin{equation}
        \mathcal{L}(\theta) = -\sum_{i} \log p_\theta(y_i|x_i) - \beta H(p_\theta(y_i|x_i))
    \end{equation}
    
    \item Focal loss \citep{Mukhoti2020CalibratingDN} maximizes entropy while minimizing the KL divergence between the predicted and the target distributions. It also regularizes the weights of the model to avoid overfitting,
    \begin{equation}
        \mathcal{L}(\theta) = -\sum_i(1-p_\theta(y_i|x_i))^\beta \log p_\theta(y_i|x_i), \quad \beta \in \mathbb{R}.
    \end{equation}
    \item Kernel MMCE \citep{MMCE} is a reproducing kernel Hilbert space (RKHS) kernel based measure of calibration that is efficiently trainable, alongside the negative likelihood loss. Given data samples $\mathcal{D} = \left\{(c_i, r_i)\right\}_{i=0}^m$, where $c_i = \chi_{\{\hat{y}_i = y_i\}}$ and $r_i = \mathbb{P}(c_i = 1| \hat{y}_i)$, MMCE is computed on samples $\mathcal{D}$ as follows,
    \begin{equation}
        \text{MMCE}^2(\mathcal{D}) = \sum_{i,j} \frac{(c_i-r_i)(c_j-r_j)k(r_i, r_j)}{m^2}
    \end{equation}
    where $k(r_i, r_j)$ is a kernel function. MMCE is optimized together with the cross entropy loss as a regularization term weighted by a scaling parameter $\lambda \in \mathbb{R}$
    \begin{equation}
        \mathcal{L}(\theta) = -\sum_i \log p_\theta(y_i|x_i) + \beta \left(\text{MMCE}^2(D)\right)^{\frac{1}{2}}.
    \end{equation}
    


\end{itemize}

\paragraph{Hyperparameter tuning for baseline methods}
Most postprocessing methods do not involve hyperparameters, and are optimized based on a validation set. The modified training methods all use a single hyperparameter to control the strength of regularization. We tune the hyperparameter, using a validation set. Figure~\ref{fig:hyperparameter} shows that MMCE is robust to the choice of hyperparameter. Focal loss and entropy regularization result in inferior performance to that of MMCE for all hyperparameter choices.

\begin{figure}[ht]
    \centering
    \includegraphics[width=0.8\textwidth]{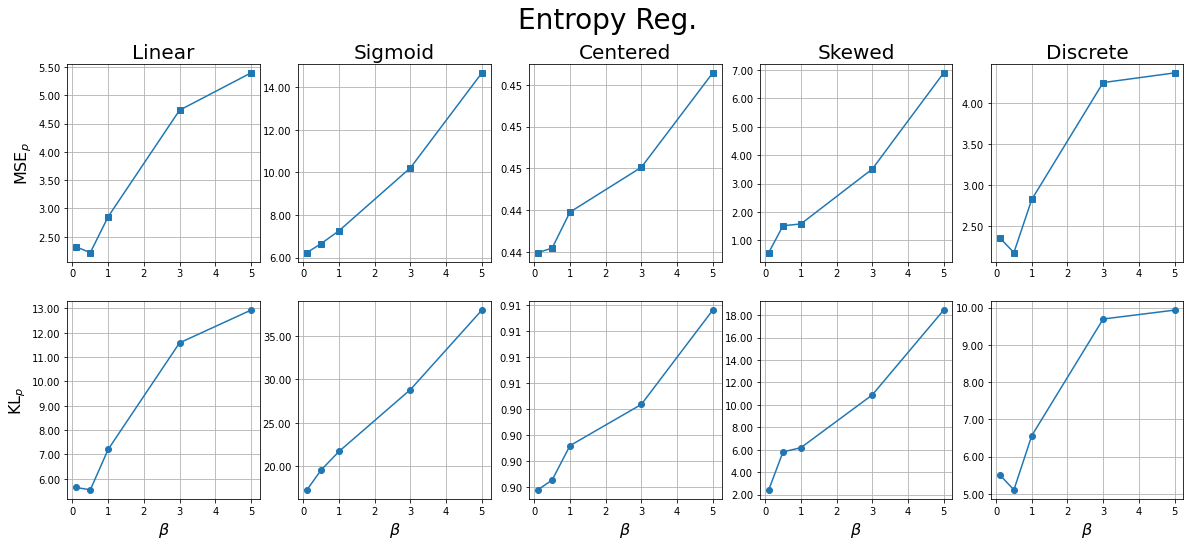}
    \includegraphics[width=0.8\textwidth]{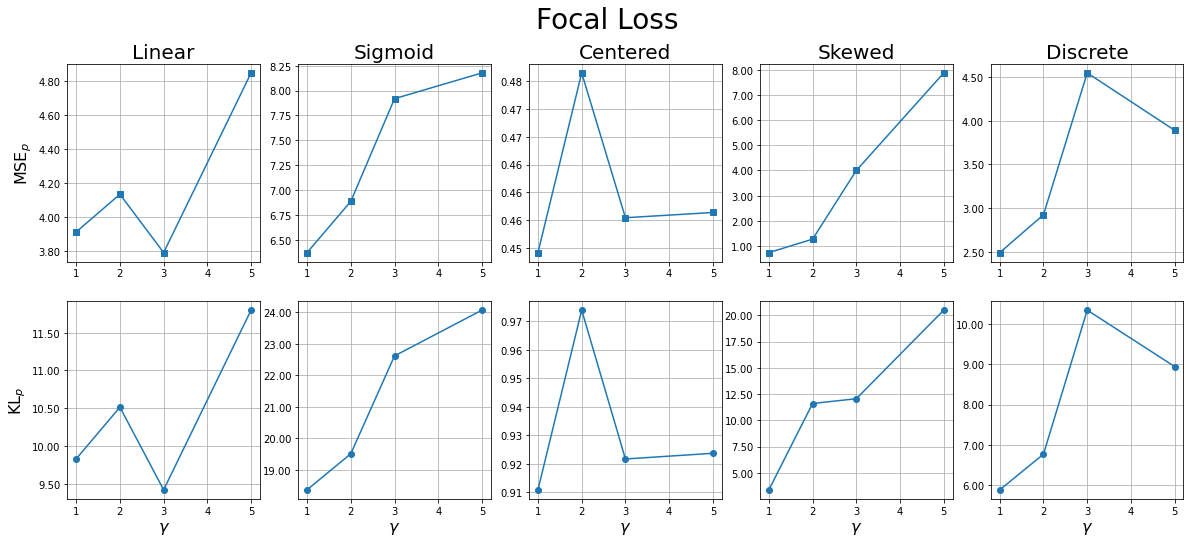}
    \includegraphics[width=0.8\textwidth]{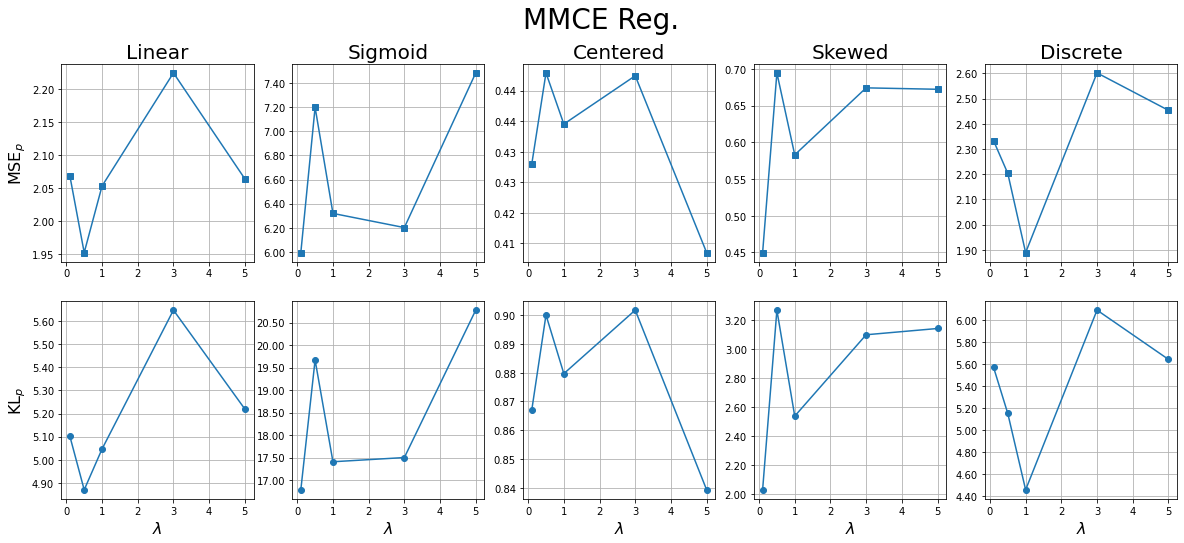}
    \caption{The hyperparameter tuning of baseline methods.}
    \label{fig:hyperparameter}
\end{figure}

\section{Synthetic data experiments}
\label{app:syntheic}
\textcolor{black}{We use a ResNet-18 backbone architecture for all our experiments with synthetic data.}


The synthetic data is split into training, validation, and test sets with 16641, 4738, and 2329 samples, respectively. The training and validation sets contain only images $x_i$ and 0-1 labels $y_i$ for training and tuning the model. In order to evaluate the performance of the model for probability estimation, the held-out test set contains the ground truth probabilities $p_i$, in addition to $x_i$ and $y_i$. Note that we  do not use the ground-truth probability labels $p_i$ values during training or inference - we only use them to compare the performance of different models. 

\paragraph{Ground Truth Probability Generation}
The ground truth probability associated with example $i$ is simulated by $p_i = \psi(z_i)$ where $z_i$ is age of the person. 
\label{app:prob_generation}
\begin{figure}[h]
    \centering
    \includegraphics[width=0.3\textwidth]{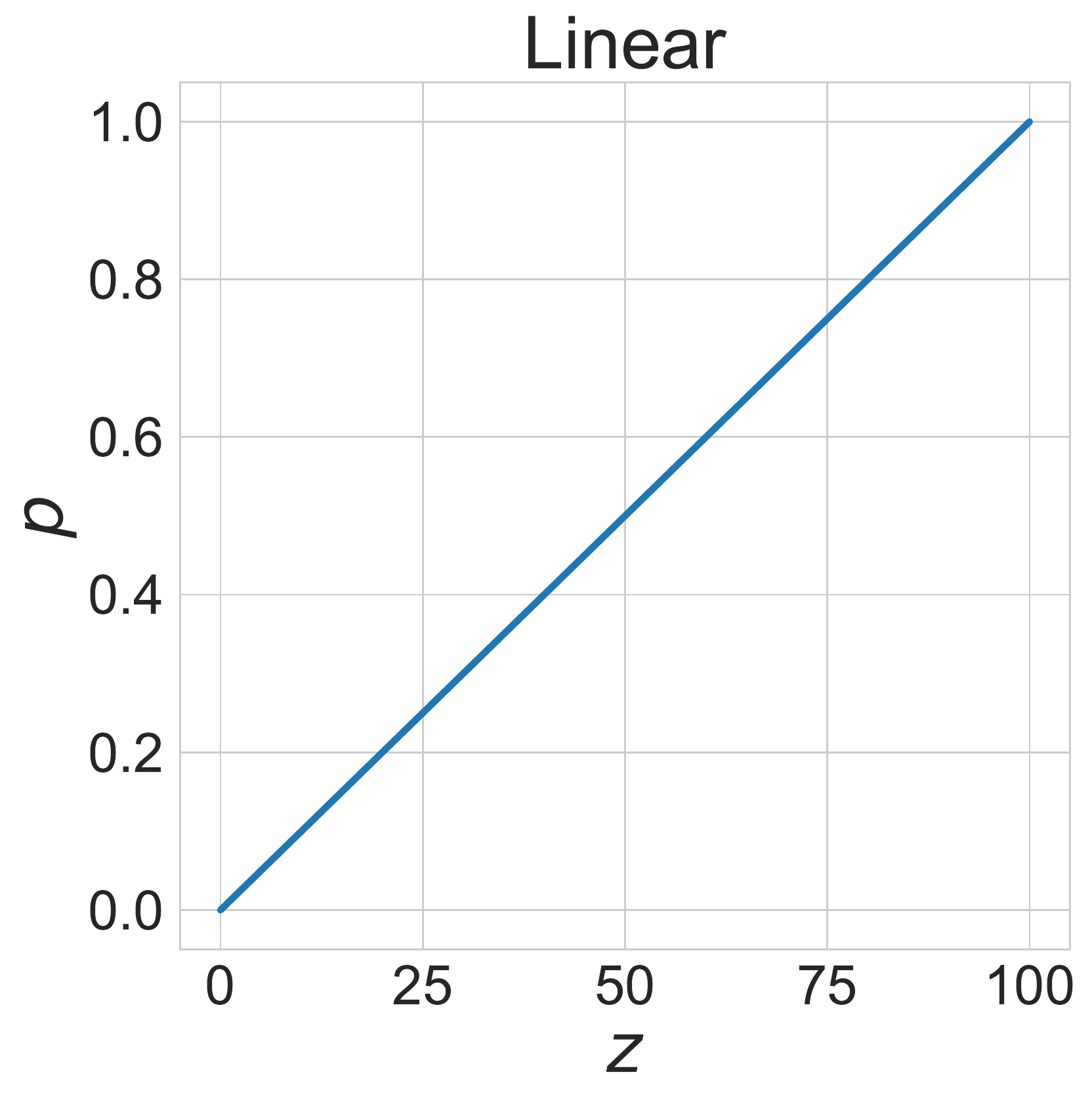}
    \includegraphics[width=0.3\textwidth]{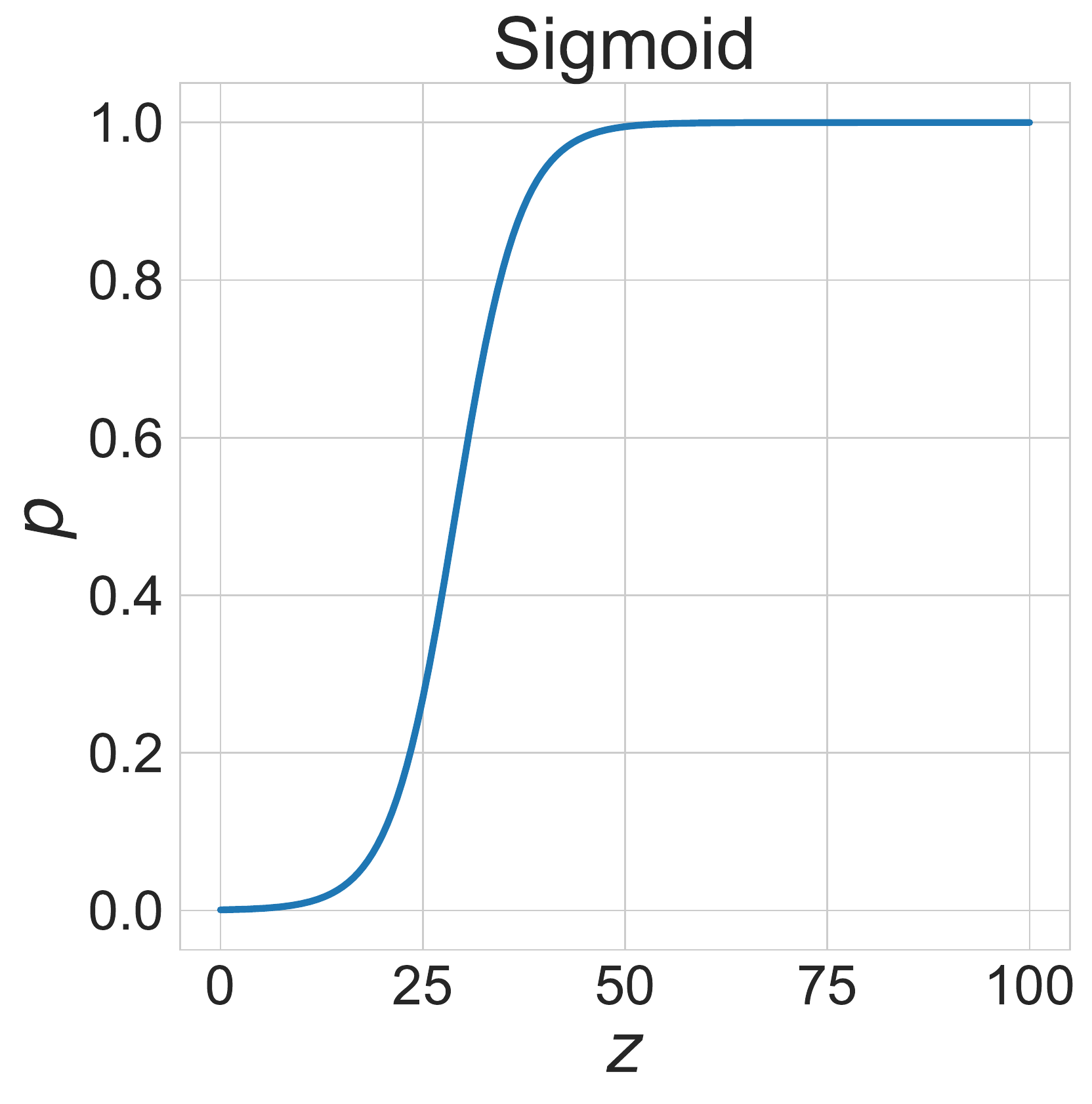}
    \includegraphics[width=0.3\textwidth]{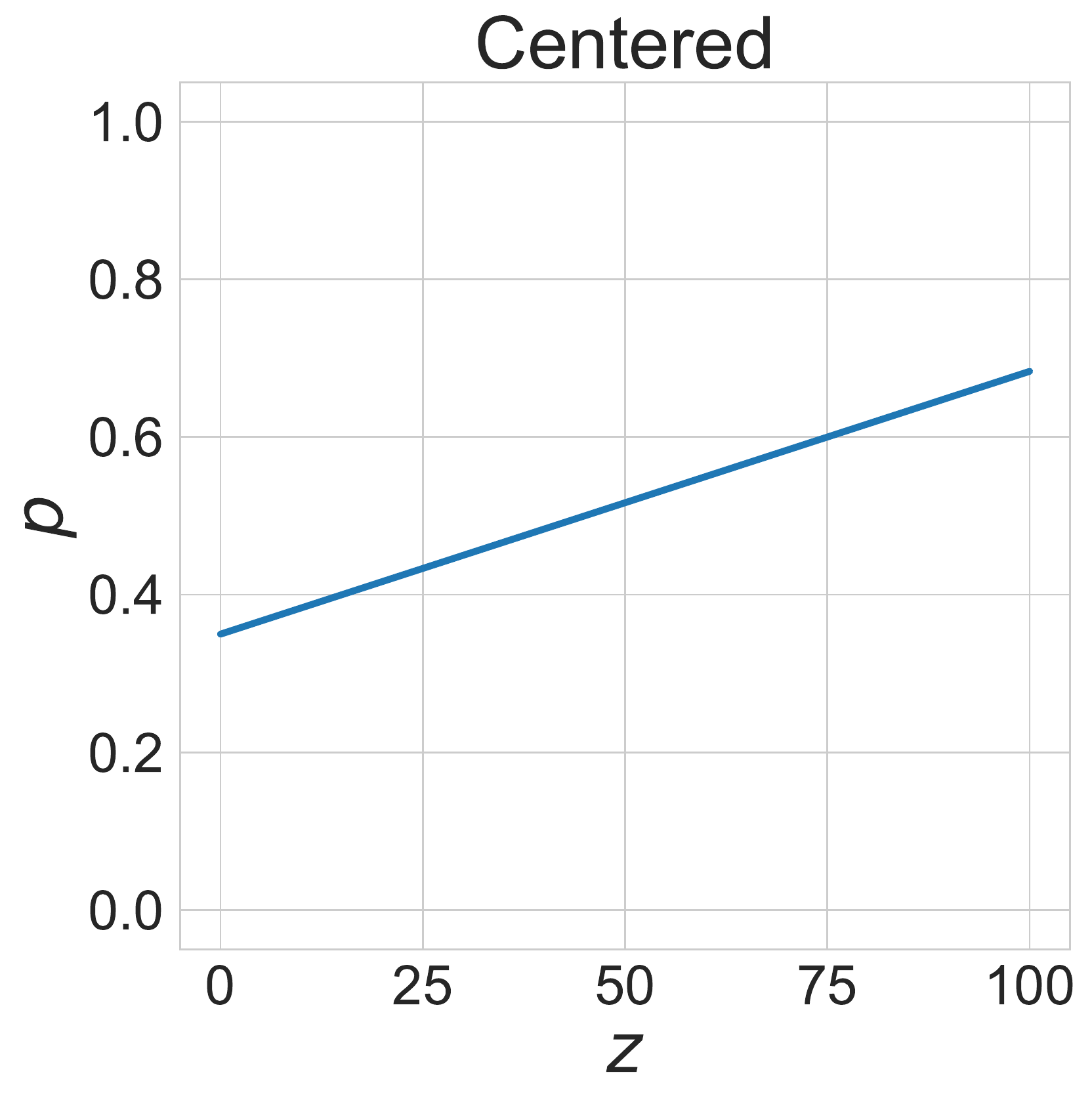} \\
    \includegraphics[width=0.3\textwidth]{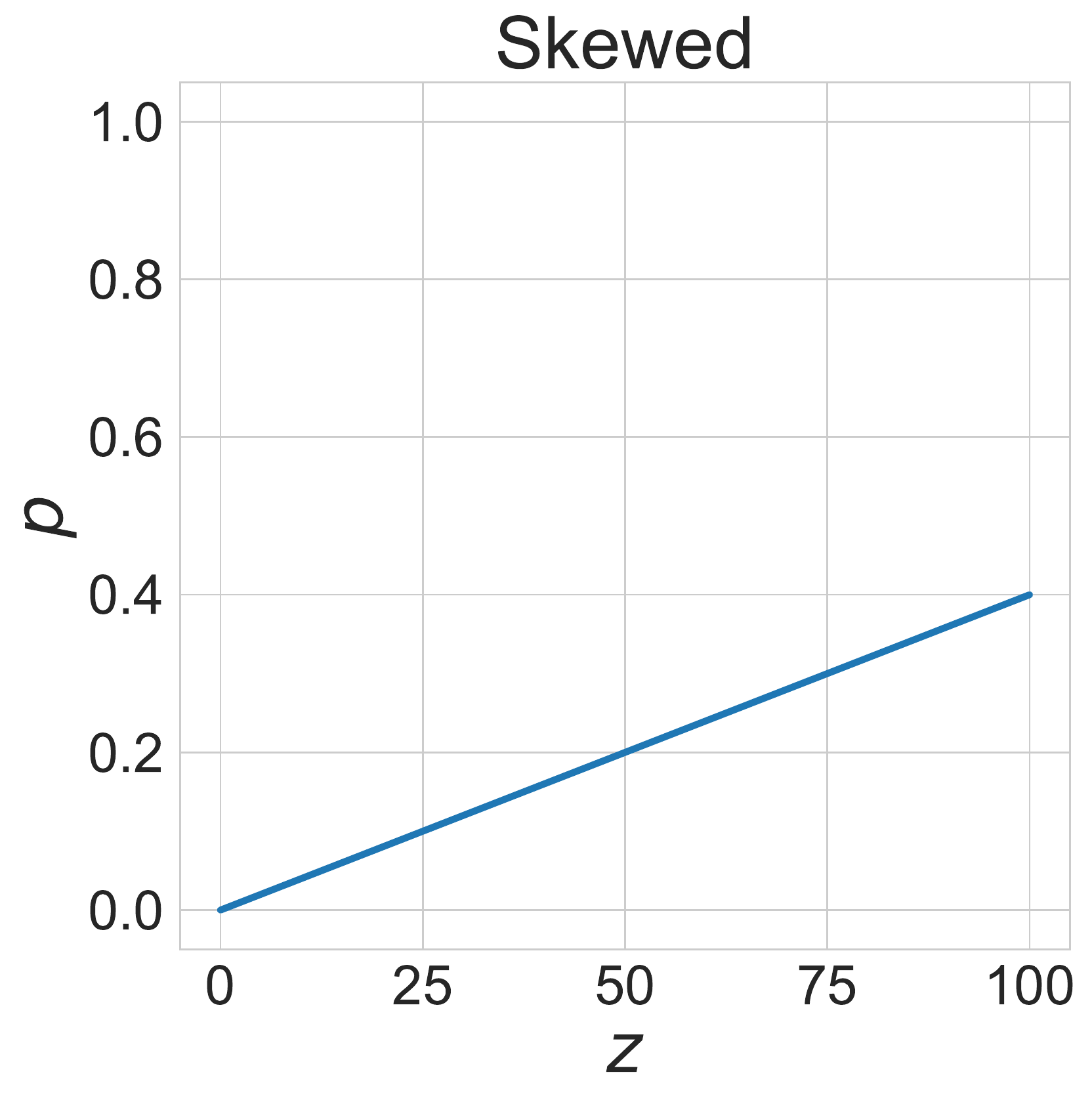}
    \includegraphics[width=0.3\textwidth]{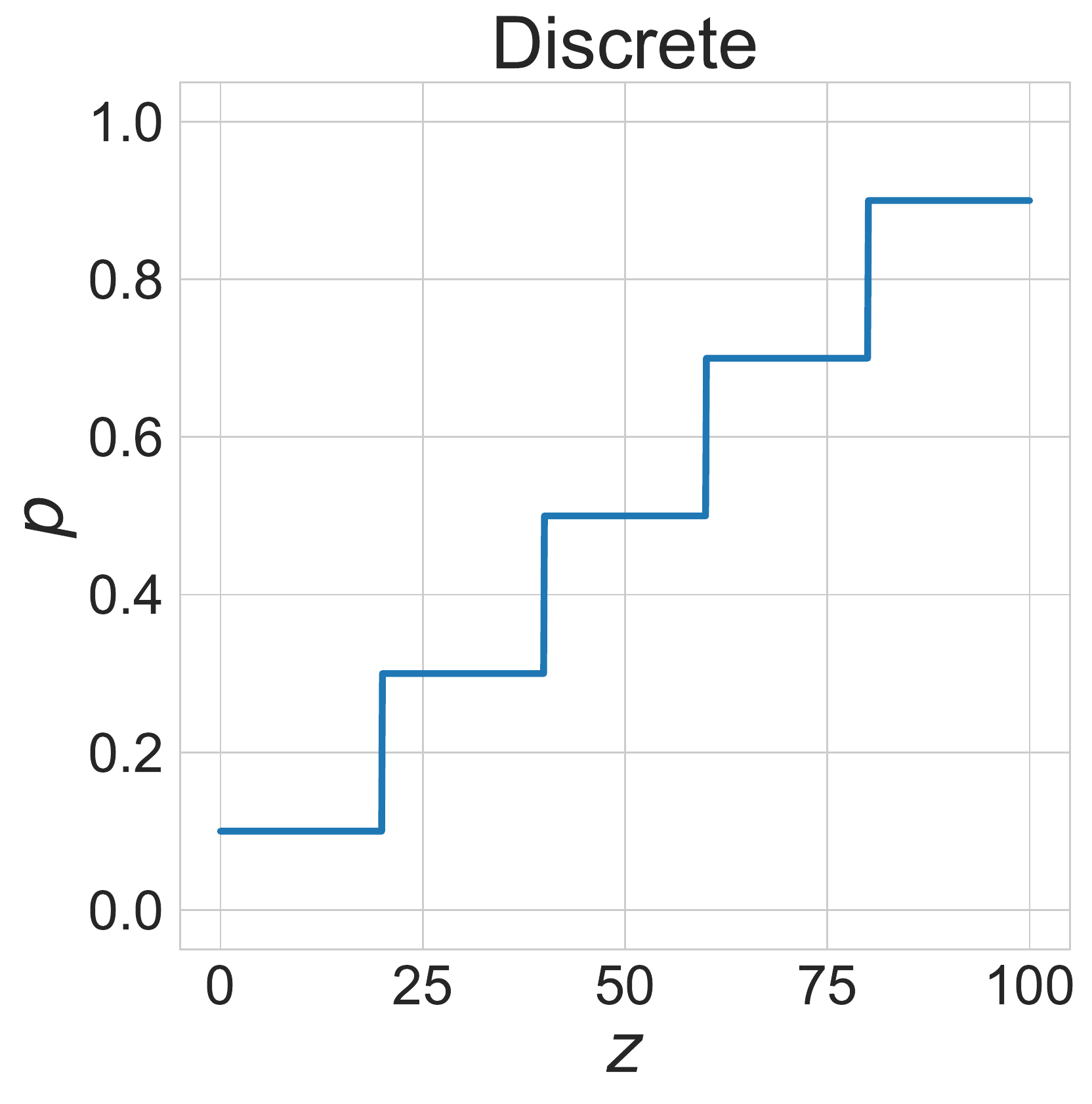}
    \caption{Illustration of the function $\psi(z)$ used to generate the different synthetic-data scenarios.}
    \label{fig:psi}
\end{figure}

\paragraph{Label distribution} After determining the probability $p_i$ using $\psi(z)$, the label $y_i$ is sampled from a Bernoulli distribution parametrized by $p_i$, so that it takes the value 1 with probability $p_i$.
The distributions of $y_i$ under five different scenarios are illustrated in Fig.\ref{fig:psi}.
\begin{figure}[h]
    \centering
     \includegraphics[width=0.32\textwidth]{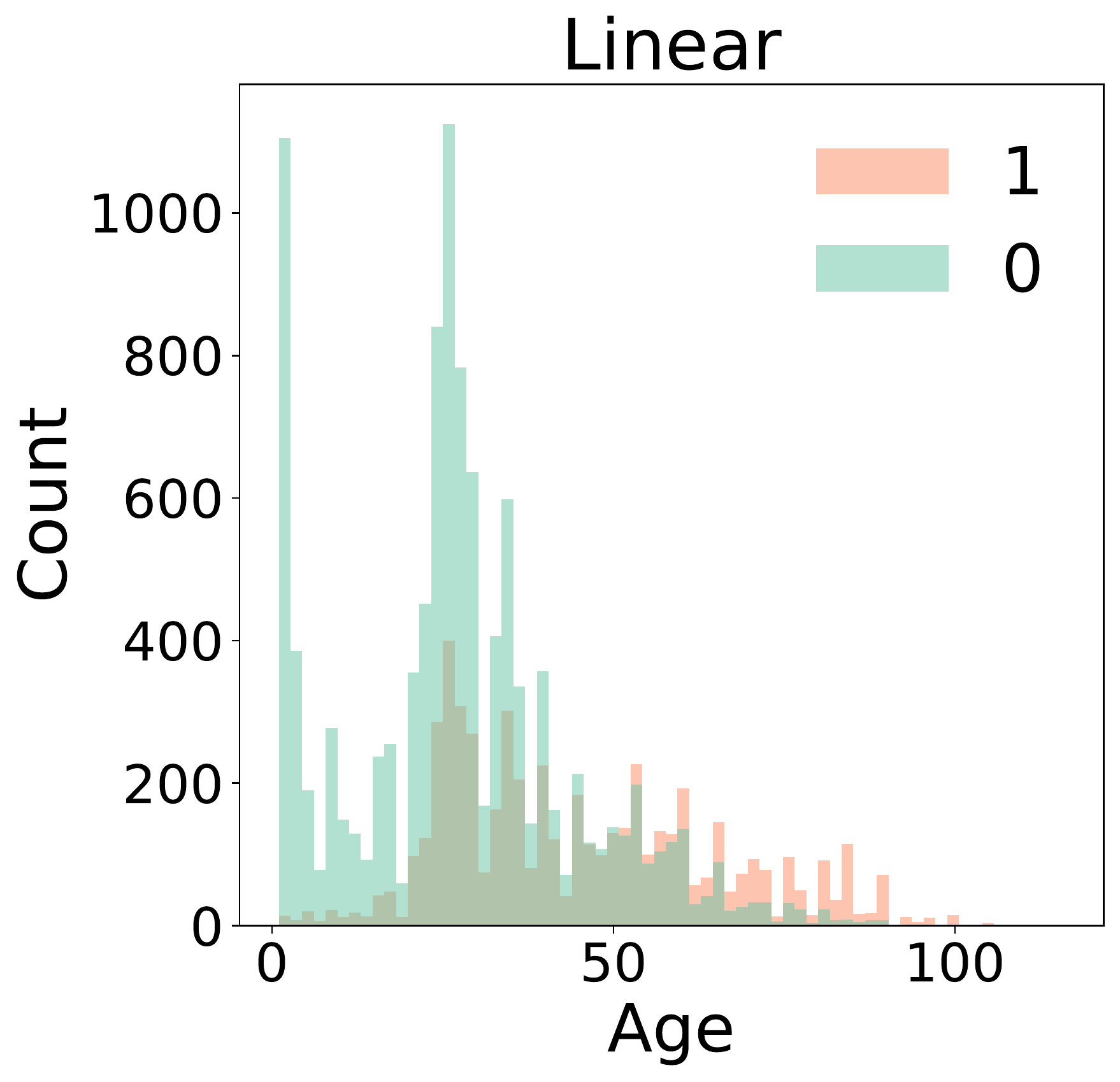}
    \includegraphics[width=0.32\textwidth]{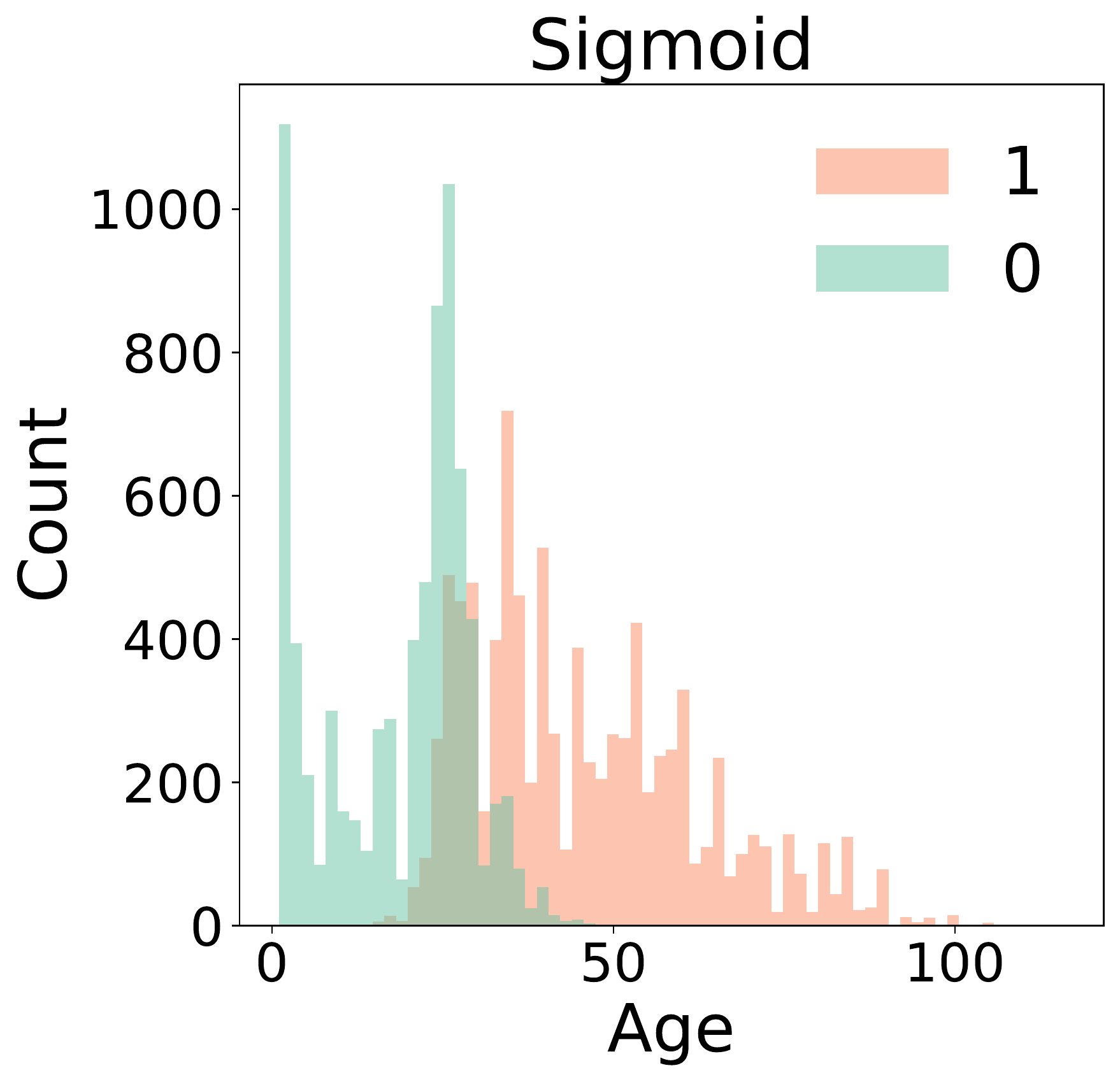}
    \includegraphics[width=0.32\textwidth]{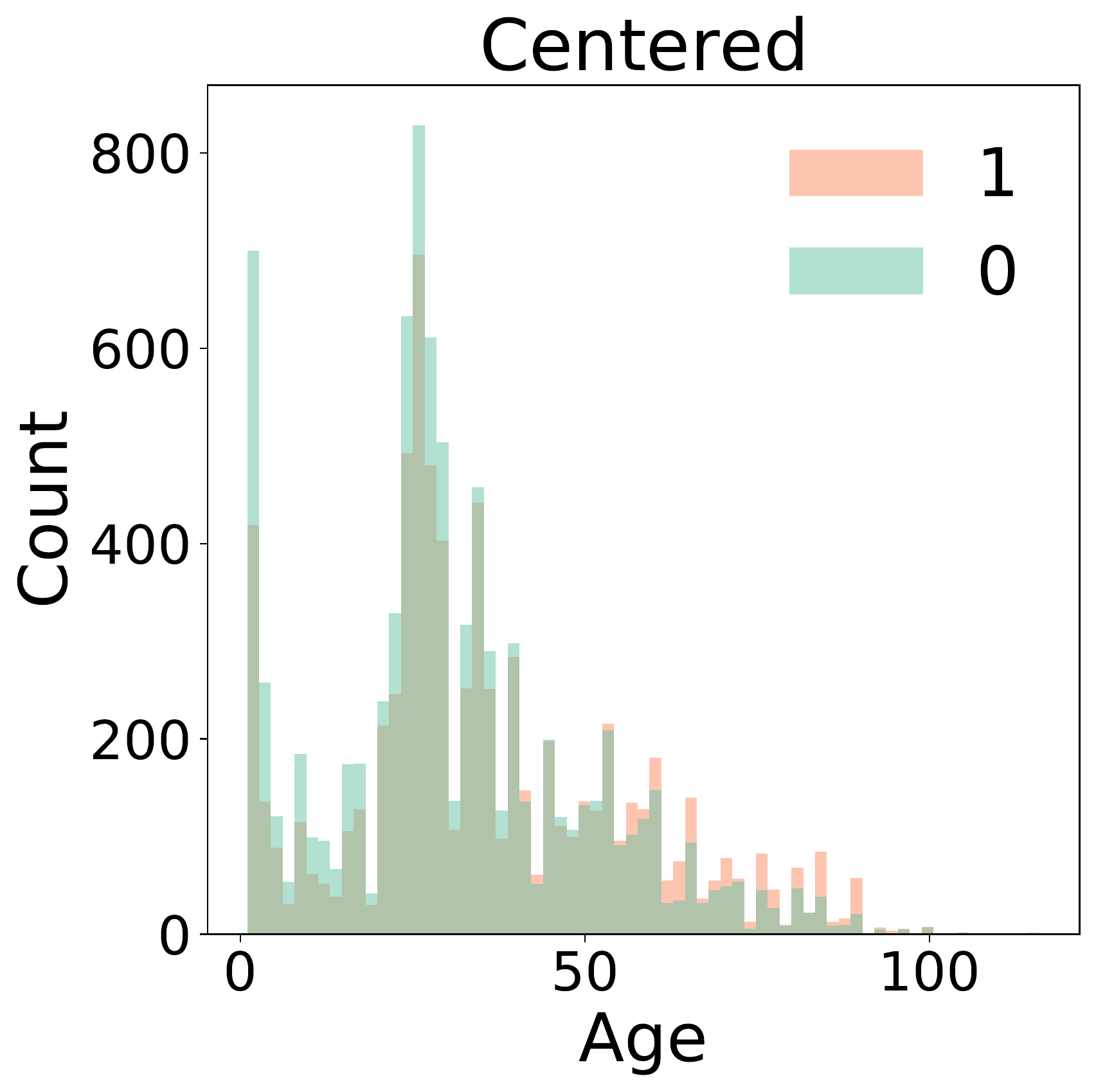} \\
    \includegraphics[width=0.32\textwidth]{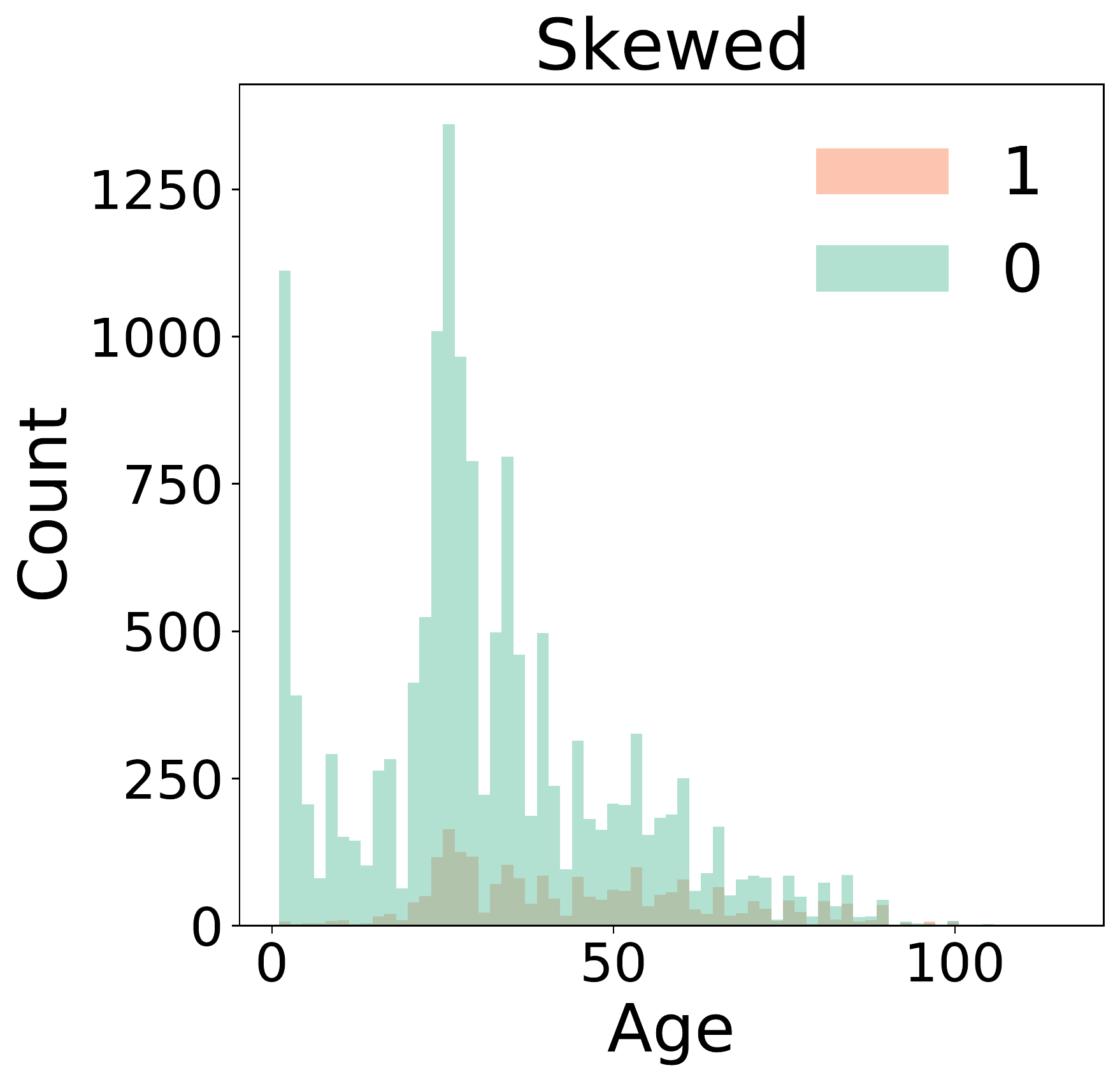}
    \includegraphics[width=0.32\textwidth]{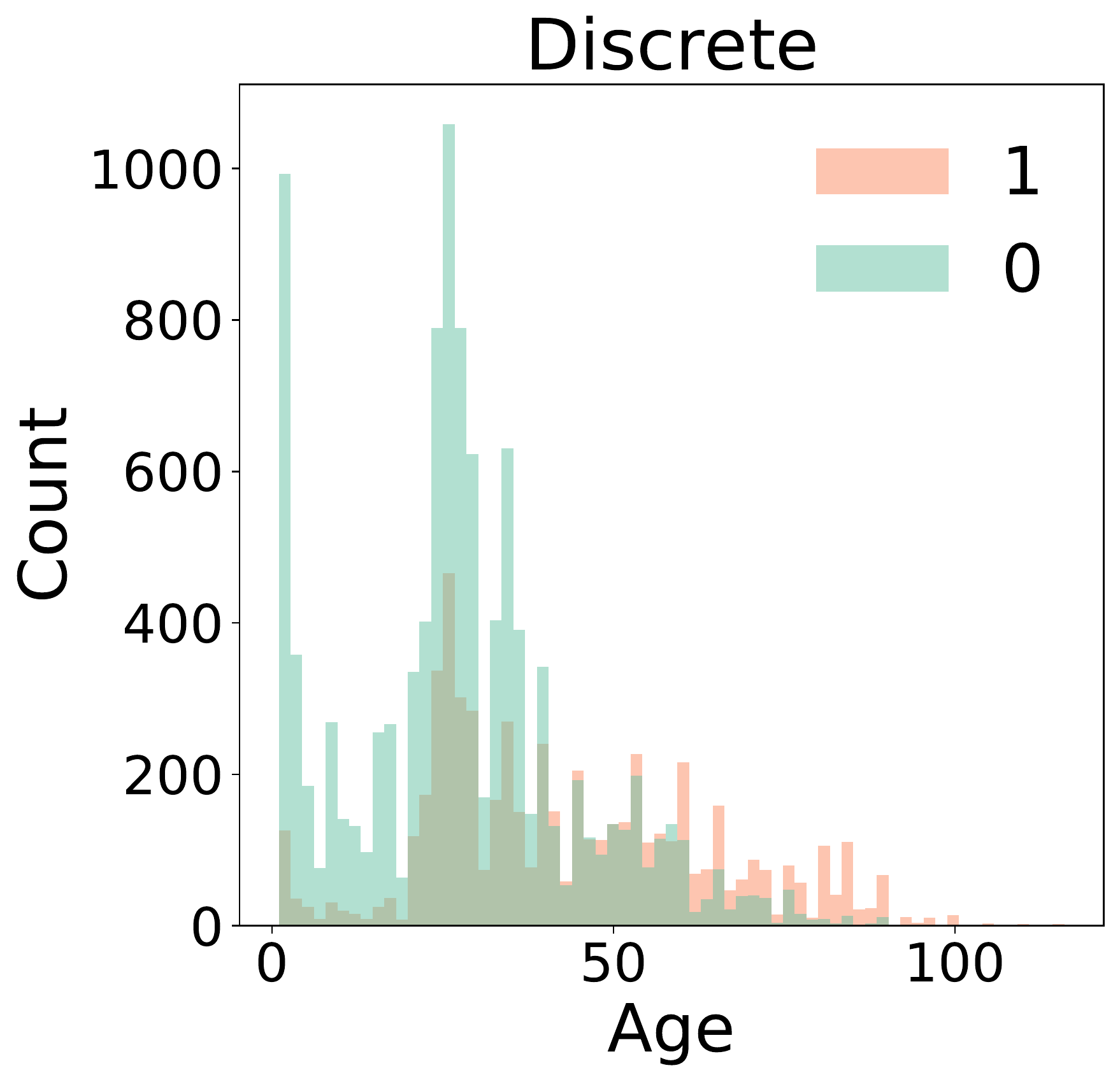}
    \caption{Histograms of the outcomes ($y_i$) for the different synthetic-data scenarios.}
    \label{fig:psi}
\end{figure}

\section{Additional Synthetic-Data Experiments on CIFAR-10} \label{app:cifar10}

We simulated probabilistic labels for CIFAR-10 to perform additional experiments. 
Each of the ten classes of CIFAR-10 was assigned a different ground-truth probability. To this end we assigned an integer $c$ between 0 and 9 to each class and set the corresponding probability equal to $c / 10$. The training and validation sets were built by assigning each image $x_i$ to a binary label $y_i$ sampled from the corresponding class probability $p_i$. Note that we do not use the ground-truth probability labels during training or inference - we only use them to evaluate the performance on the test set.


Table~\ref{tab:cifar10results} shows the results on the test set. We again observe that CaPE outperforms the cross-entropy baseline based on early stopping.

\begin{table}[t]
\begin{center}
{
\begin{tabular}{l|l|l|l}
\toprule
Name    & CE early stop & CaPE (bin) & CaPE (kd) \\ \midrule
$\text{MSE}_p$ & 0.0297        & 0.0252     & \textbf{0.0247}    \\ 
$\text{KL}_p$  & 0.4051        & \textbf{0.2468}  & 0.2598    \\ \bottomrule
\end{tabular}
}
\caption{Results of on CIFAR 10 with simulated probabilistic labels. CaPE outperforms a CE early stopped model.}
\label{tab:cifar10results}
\end{center}
\end{table}

\section{Additional Details on Real-World Data and Experiments}
\label{app:real-world}
We present here supplementary information for the real-world datasets used in our experiments. 
\paragraph{Cancer Survival}

Histopathological features are useful in identification of tumor cells, cancer subtypes, and the stage and level of differentiation of the cancer. Hematoxylin and Eosin (H\&E)-stained slides are the most common type of histopathology data and the basis for decision making in the clinics. With these properties, H\&E are used for mortality prediction of cancer \citep{Wulczyn2020DeepLS}. In this experiment, we use the H\&E slides of non-small cell lung cancers from The Cancer Genome Atlas Program (TCGA)\footnote{\url{https://www.cancer.gov/tcga}} to predict the 5-year survival. The dataset has 1512 whole slide images from 1009 patients, and 352 of them died in 5-years. We split the samples by patients and source institutions into training, validation, and test set, which has 1203, 151, and 158 samples respectively.

The whole slide images contain numerous pixels, so we cropped the slides into tiles at 10x magnification with 1/4 overlapping, resized them to $299\times 299$ with bicubic interpolation, and filtered out the tiles with more than 85\% area covered by the background. \textcolor{black}{The representations of each tile are trained with self-supervised momentum contrastive learning (MoCo) \citep{moco}, and the slide-level prediction is obtained from a multiple-instance learning network \citep{mil_attention} trained with the binary label of survival in 5 years.}

\paragraph{Weather Forecasting}



We use the German Weather service dataset\footnote{\url{https://opendata.dwd.de/weather/radar/}}, which contains quality-controlled rainfall-depth composites from 17 operational Doppler radars. Three precipitation maps from the past 30 minutes serve as an input. The training labels are the 0/1 events indicating whether the mean precipitation increases (1) or not (0). 


The German Weather service (DWD - Deutshce Wetter Dienst) dataset \url{https://opendata.dwd.de/weather/radar/} contains quality-controlled rainfall-depth composites from 17 operational DWD Doppler radars. It has a spatial extent of 900x900 km, and covers the entirety of Germany. Data exists since 2006, with a spatial and temporal resolution of 1x1 km and 5 minutes, respectively. The dataset has been used to train RainNet, a pricipitation nowcasting model \citep{Ayzel2020RainNet}. 

\textcolor{black}{The network architecture is ResNet18, with 3 input channels and 2 output channels.}
The input to the network are 3 precipitation maps which cover a fixed area of 300km$\times$300 km in the center of the grid (300$\times$ 300 pixels), set 10 minutes apart. The training, validation and test datasets consist of 20000, 6000 and 3000 samples, respectively, all separated temporally, over the span of 15 years.


\paragraph{Collision Prediction} Vehicle collision is one of the leading causes of death in the world. Reliable collision prediction systems which can warn drivers about potential collisions can save a significant number of lives. A standard way to design such a system is to train a convolutional model for identifying if a particular vehicle in the dash-cam video feed might collide with the car in next few seconds. More formally, at time $t=T$ the system tries to predict if any car in the video might collide with our given car in time $t \in [T, T+T_{\text{look-ahead}}]$. Each labelled training sample consists of features $X = (X^{T-\delta}, X^{T-2\delta}, \dots, X^{T-d\delta})$ and a binary label $Y\in \{0, 1\}$ denoting if an accident will occur in $t \in [T, T+T_{\text{look-ahead}}]$. Each $X^t$ is a tensor with 4 channels where the first 3 channels corresponds to an RGB image of the dashcam view at time $t=t$, and the fourth channel consists of a mask with a bounding box on a particular vehicle of interest. In this work, we use \textit{YouTubeCrash} dataset~\citep{kim2019crash} to train and test our model, which uses $\delta = 0.1s$,$T_{\text{look-head}}=18\delta=1.8s$, and $d=3$. \textcolor{black}{Following ~\cite{kim2019crash} we used a VGG-16 network architecture.}

The dataset contains $122$ accident scenes, and $2096$ non-accident scenes, which after feature extraction gives us $2096$ positive samples, and $11486$ negative samples (the dataset is severely imbalanced, and similar to the Skewed situation in Section~\ref{sec:synthetic_data}). We further split the dataset into train (6453 samples for label 0, and 1023 samples for label 1), validation (2348 samples for label 0, and 545 samples for label 1), and test (2685 samples for label 0, and 528 samples for label 1) sets. The samples in train, validation and test sets are generated from disjoint scenes/dashcam videos. 

\section{Analysis of Cancer Survival Results}
\label{app:case_study}

\begin{figure}[h]
    \centering
    \includegraphics[width=0.4\textwidth]{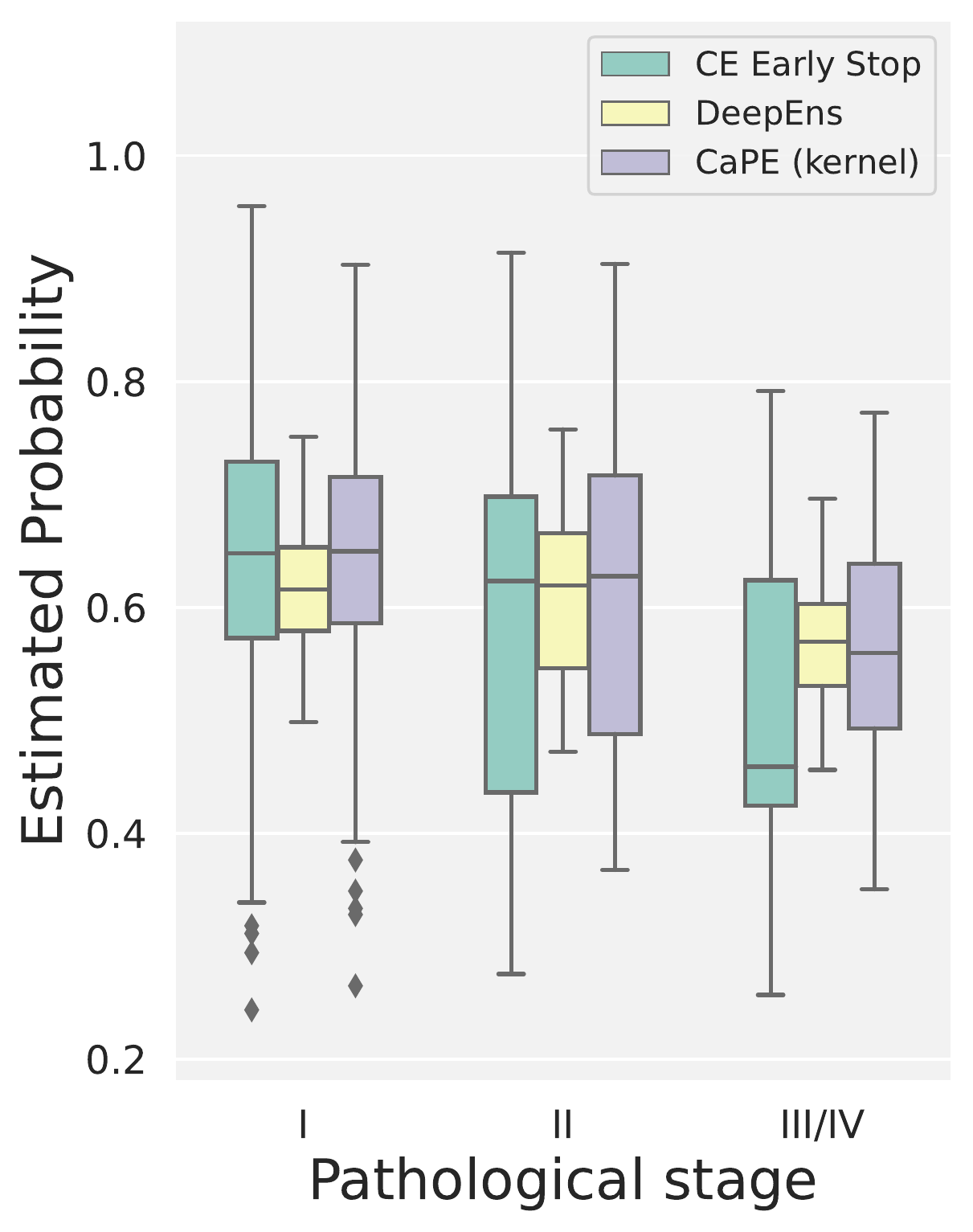}
    \caption{Estimated probability of survival grouped by pathological stages. \textcolor{black}{The plot shows median, samples between 25th to 75th percentile in the box, samples between 0th and 100th percentile on the line, and the outliers as dots.} Deep ensemble produces similar probability estimates for patients across all the stages; CE is more discriminative but has a very large variance; CaPE achieves a trade-off between the two baselines.}
    \label{fig:survival_boxplot}
\end{figure}
For cancer survival prediction, we visualize the
estimated probabilities on the test set in different pathological stages in Figure~\ref{fig:survival_boxplot}. In general, patients in earlier stages should have higher probabilities of survival. Deep ensemble produces similar probability estimates for all stages (i.e the model is less discriminative). Cross-entropy minimization (CE) is more discriminative, but has very wide confidence intervals. CaPE is more discriminative than deep ensemble, while having narrower confidence intervals than CE. 

\section{Calibrating from the Beginning}
\label{sec:beginning}
\textcolor{black}{CaPE exploits a calibration-based cost function to improve its probability estimates without overfitting. The empirical probabilities in this loss are computed from the model itself. Consequently, applying this strategy from the beginning of training can be counterproductive, because the model predictions are essentially random. This is demonstrated in the following table, which compares CaPE with a model trained using the calibration loss from the beginning (in the same way as CaPE, alternating with cross-entropy minimization).}
\begin{table}[h]
    \centering
    \footnotesize{
   \begin{tabular}{lc@{\hspace{0.3cm}}c|c@{\hspace{0.3cm}}c|c@{\hspace{0.3cm}}c|c@{\hspace{0.3cm}}c|c@{\hspace{0.3cm}}c}
\toprule
   Methods  &      \multicolumn{2}{c}{\textit{Linear}} &    \multicolumn{2}{c}{\textit{Sigmoid}}         &       \multicolumn{2}{c}{\textit{Centered}}        &         \multicolumn{2}{c}{\textit{Skewed}}  &\multicolumn{2}{c}{\textit{Discrete}}           \\
 \multicolumn{1}{r}{$\small{(\times 10^{-2})}$} &  $\small{\text{MSE}_p}$ &  $\small{\text{KL}_p}$ &  $\small{\text{MSE}_p}$ &  $\small{\text{KL}_p}$  &  $\small{\text{MSE}_p}$ &  $\small{\text{KL}_p}$ &  $\small{\text{MSE}_p}$ &  $\small{\text{KL}_p}$ &  $\small{\text{MSE}_p}$ &  $\small{\text{KL}_p}$ \\
\midrule
Bin (start) & 2.59 & 6.81&8.07 &22.10 &0.48 & 0.98
& 0.51 & 2.37 & 2.74 & 6.36 \\
Kernel (start) & 2.23 & 5.68 & 7.60 & 21.15 & 0.54 & 1.10 & 0.68 & 2.84 & 2.40 & 5.63 \\ 
\midrule
Bin (CaPE)  &        1.83 &        4.46 &        5.29 &       14.59 &        \textbf{0.38} &        \textbf{0.78} &        0.40 &        1.72 &        \textbf{1.83} &        \textbf{4.31} \\
Kernel (CaPE)        &        \textbf{1.81} &        \textbf{4.41} &        \textbf{5.22} &       \textbf{14.47} &        0.40 &        0.81  &        \textbf{0.39} &        \textbf{1.70} &        1.85 &        4.36 \\

\bottomrule

\end{tabular}
}
    \caption{\textcolor{black}{Comparison between CaPE and a model that uses the calibration loss from the beginning (in the same way as CaPE, alternating with cross-entropy minimization) on synthetic data. All numbers are downscaled by $10^{-2}$.}
    }
    \label{tab:scratch_baseline}
\end{table}

\section{Correspondence of Real-World Datasets and the Scenarios of the Simulated Dataset}
\textcolor{black}{
Figure~\ref{fig:comp_real_curves} illustrates the similarity between the empirical probability curves of different real-world datasets and the different scenarios of our synthetic dataset. For the cancer survival dataset, the empirical probabilities are clustered in the center (0.4-0.6) similar to the \textit{Centered} scenario. For the weather forecasting dataset, the probabilities are uniformly distributed across 0.1-0.8 similar to the \textit{Linear} scenario. For the collision prediction dataset, the majority of the data points are clustered in the lower probability region which makes it similar to the \textit{Skewed} scenario.}
\begin{figure}[t]
    \centering
    \includegraphics[width=0.8\textwidth]{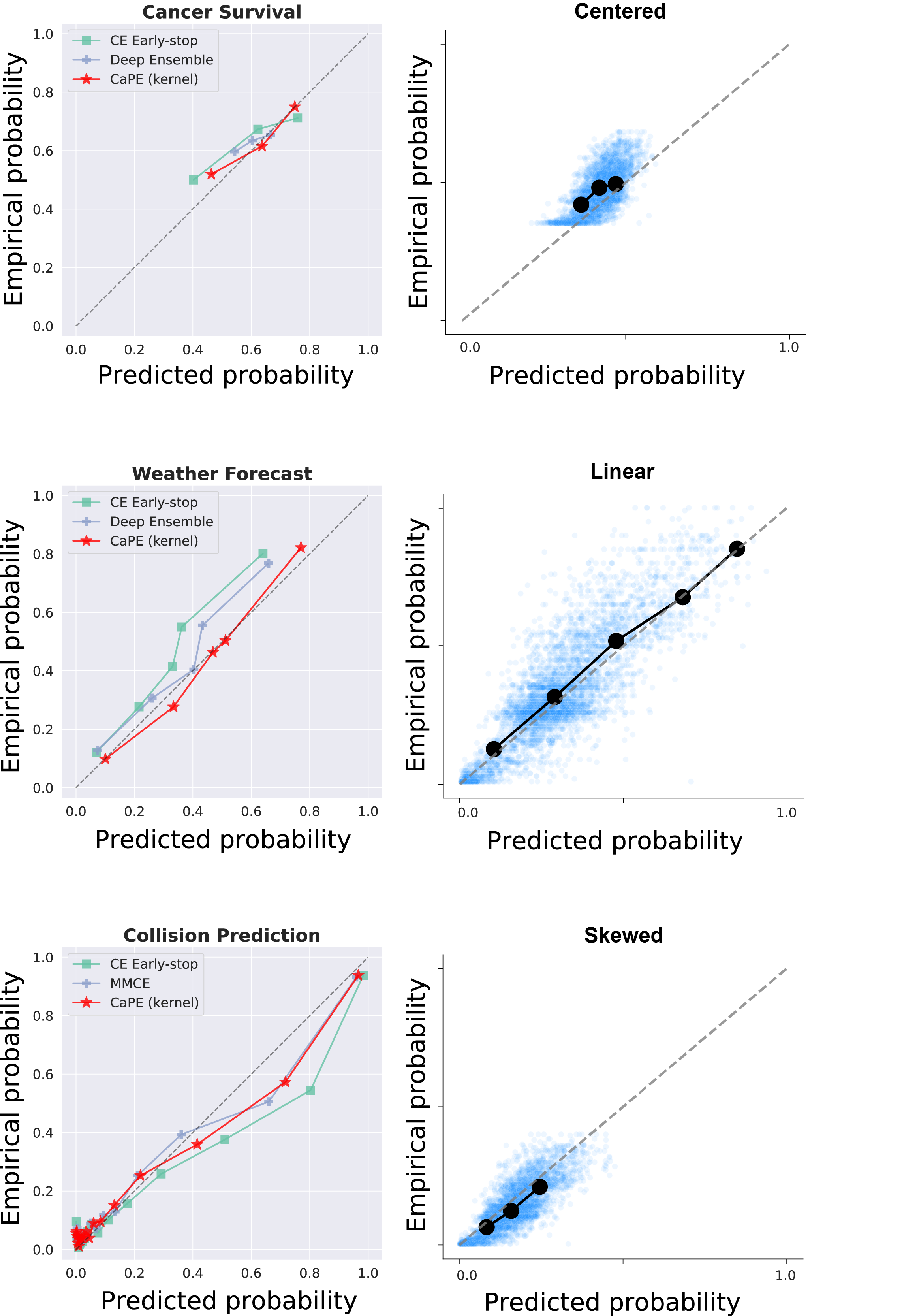}
    \caption{\textcolor{black}{\textbf{Comparison of reliability diagrams for real-world data with different scenarios of simulated data}.  For the cancer survival dataset, the empirical probabilities are clustered in around (0.4-0.6), similar to the \textit{Centered} scenario. For the weather forecasting dataset, the probabilities are uniformly distributed across 0.1-0.8, similar to the \textit{Linear} scenario. For the collision prediction dataset, the majority of the output probabilities are clustered in the lower probability region, similar to the \textit{Skewed} scenario.}}
    \label{fig:comp_real_curves}
\end{figure}

\end{document}

%% file: math_commands.tex
\usepackage{amsmath,amsfonts,bm}

\def\eqref#1{equation~\ref{#1}}

\def\floor#1{\lfloor #1 \rfloor}
\def \lrv#1{\left\lVert #1 \right\rVert_2}
\def\1{\bm{1}}

\def\vzero{{\bm{0}}}

\def\vtheta{{\bm{\theta}}}

\def\vw{{\bm{w}}}
\def\vx{{\bm{x}}}
\def\vy{{\bm{y}}}
\def\vz{{\bm{z}}}

\DeclareMathAlphabet{\mathsfit}{\encodingdefault}{\sfdefault}{m}{sl}
\SetMathAlphabet{\mathsfit}{bold}{\encodingdefault}{\sfdefault}{bx}{n}

\newcommand{\E}{\mathbb{E}}

\newcommand{\R}{\mathbb{R}}

\newcommand{\sigmoid}{\sigma}



%% file: algorithm.tex
\begin{algorithm*}[t]
\small
\caption{Pseudocode for CaPE}\label{alg:algo_cali}
\vspace{-3mm}
\begin{multicols}{2}
\begin{algorithmic}
\REQUIRE{$f$} \COMMENT{early stopped model}
\REQUIRE $m$\COMMENT{freq. of training with $\mathcal{L}_{C}$}
\REQUIRE $\{\vx_i, y_i\}_{i=1}^N$ \COMMENT{training set}
\REQUIRE $K(p,q)\coloneqq \exp\left[-\left(p - q\right)^2/\sigma^2\right]$\COMMENT{Gaussian kernel}
\FOR{$t = 1$ to \texttt{num\_epochs}}
    \IF{$t\mod m = 0$}
        \STATE $\hat{p}_i \gets f(\vx_i), \forall i$ 
        \STATE Update $p^{i}_{\text{emp}}, \forall i$, with \textsc{Bin} or \textsc{Kernel}
        \STATE $\mathcal{L} \gets \mathcal{L}_{C}$ \COMMENT{compute discrimination loss}
    \ELSE
        \STATE $\mathcal{L} \gets \mathcal{L}_{D}$ \COMMENT{compute calibration loss}
    \ENDIF
\ENDFOR
\end{algorithmic}
\columnbreak
\begin{algorithmic}
\FUNCTION{\textsc{Bin}$(B)$ \hfill\(\triangleright\) $B$-number of bins}
    \STATE $I_1, \cdots I_B \gets$ partitions by quantile of ${\{\hat{p}_j\}_{j=1}^N}$
    \STATE Find $b$ such that $\hat{p}_i \in I_b$
    \STATE $\text{Index}(I_b) \gets \{j | \hat p_j\in I_b\}$ \COMMENT{get indices in bin $b$}
    \STATE $p^i_\text{emp} \gets \frac{1}{|I_b|}\sum\limits_{i \in \text{Index}(I_b)} y_i$ \COMMENT{empirical mean of bin $b$}
\ENDFUNCTION
\vspace{3mm}

\FUNCTION{\textsc{Kernel}($r$, $K$) \hfill\(\triangleright\) $r$-window size; $K$-kernel}
\STATE $N_r(i) \gets r$-nearest neighbor of $\hat{p}_{i}$ \textcolor{black}{(output probability space)}
\STATE $Z \gets \sum\limits_{\hat{p}_j \in N_r(i)}K\left(\hat{p}_{i},\hat{p}_j\right)$ \COMMENT{normalization factor}
\STATE $p_{\text{emp}}^{i}  \gets \sum\limits_{\hat{p}_j \in N_r(i)} K\left(\hat{p}_{i},\hat{p}_j\right) y_{j}/Z$ \COMMENT{kernel smooth}
\ENDFUNCTION
\end{algorithmic}
\end{multicols}
\vspace{-3mm}
\end{algorithm*}